\documentclass{article}
\title{%\LARGE{Fast Overparameterized Convergence:\\Gradient Descent Takes the Shortest Path}\\
Overparameterized Nonlinear Learning:\\
Gradient Descent Takes the Shortest Path?
}%/Overdetermined
\author{}
\usepackage{hyperref,graphicx,amsmath,amsfonts,amssymb,bm,url,breakurl,epsfig,epsf,color,fullpage,MnSymbol,mathbbol,fmtcount,semtrans,cite,caption,subcaption,multirow,comment}
\usepackage[noend]{algpseudocode}

\usepackage{amssymb}

\usepackage[utf8]{inputenc} % allow utf-8 input
\usepackage[T1]{fontenc}    % use 8-bit T1 fonts
\usepackage{url}            % simple URL typesetting
\usepackage{booktabs}       % professional-quality tables
\usepackage{amsfonts}       % blackboard math symbols
\usepackage{nicefrac}       % compact symbols for 1/2, etc.
\usepackage{microtype}      % microtypography

\makeatletter
\providecommand*{\boxast}{%
  \mathbin{% as \boxplus and \boxtimes
    \mathpalette\@boxit{*}%
  }%
}
\newcommand*{\@boxit}[2]{%
  \sbox0{$\m@th#1\Box$}%
  \ifx#1\displaystyle \ht0=\dimexpr\ht0+.05ex\relax \fi
  \ifx#1\textstyle \ht0=\dimexpr\ht0+.05ex\relax \fi
  \ifx#1\scriptstyle \ht0=\dimexpr\ht0+.04ex\relax \fi
  \ifx#1\scriptscriptstyle \ht0=\dimexpr\ht0+.065ex\relax \fi
  \sbox2{$#1\vcenter{}$}% \ht2 is positionn of math axis
  \rlap{%
    \hbox to \wd0{%
      \hfill
      \raisebox{%
        \dimexpr.5\dimexpr\ht0+\dp0\relax-\ht2\relax
      }{$\m@th#1#2$}%
      \hfill
    }%
  }%
  \Box
}
\makeatother

  \makeatletter
\def\BState{\State\hskip-\ALG@thistlm}
\makeatother

  \usepackage{mathtools}

\usepackage{titlesec}

\usepackage{tikz}
\usepackage{pgfplots}
\usetikzlibrary{pgfplots.groupplots}

\titleformat{\paragraph}
{\normalfont\normalsize\bfseries}{\theparagraph}{1em}{}
\titlespacing*{\paragraph}
{0pt}{3.25ex plus 1ex minus .2ex}{1.5ex plus .2ex}

\usepackage{movie15}

\usepackage{cite}
\usepackage{caption}
\usepackage[bottom,hang,flushmargin]{footmisc} 

\newcommand{\tsn}[1]{{\left\vert\kern-0.25ex\left\vert\kern-0.25ex\left\vert #1 
    \right\vert\kern-0.25ex\right\vert\kern-0.25ex\right\vert}}

\definecolor{darkred}{RGB}{150,0,0}
\definecolor{darkgreen}{RGB}{0,150,0}
\definecolor{darkblue}{RGB}{0,0,200}
\hypersetup{colorlinks=true, linkcolor=darkred, citecolor=darkgreen, urlcolor=darkblue}

\newtheorem{theorem}{Theorem}[section]

\newtheorem{assumption}{Assumption}

\newtheorem{lemma}[theorem]{Lemma}
\newtheorem{corollary}[theorem]{Corollary}

\newtheorem{definition}[theorem]{Definition}

\newcommand{\eps}{\varepsilon}

\newcommand{\el}{L}

\newcommand{\vc}[1]{{\text{vect}(#1)}}
\newcommand{\alpb}{\zeta}
\newcommand{\alpbb}{(\la -\eta\bp^2/2)\bn}
\newcommand{\gam}{{\gamma}}

\newcommand{\bp}{\beta}
\newcommand{\bn}{\alpha}

\newcommand{\beq}{\begin{equation}}

\newcommand{\vs}{{\vspace*{-4pt}}}
\newcommand{\eeq}{\end{equation}}

\newcommand{\nn}{\nonumber}
\newcommand{\la}{\lambda}
\newcommand{\A}{{\mtx{A}}}

\newcommand{\Ub}{{\mtx{U}}}

\newcommand{\Gb}{{\mtx{G}}}

\newcommand{\Lc}{{\cal{L}}}
\newcommand{\Jc}{{\cal{J}}}
\newcommand{\Dc}{{\cal{D}}}

\newcommand{\Vc}{{\cal{V}}}

\newcommand{\rng}{\gamma}
\newcommand{\Cb}{{\mtx{C}}}

\newcommand{\Gc}{{\cal{G}}}

\newcommand{\bSi}{{\boldsymbol{{\Sigma}}}}

\newcommand{\Db}{{\mtx{D}}}

\newcommand{\Iden}{{\mtx{I}}}

\newcommand{\smn}[1]{{\sigma_{\min}(#1)}}
\newcommand{\smx}[1]{{\sigma_{\max}(#1)}}

\newcommand{\tn}[1]{\|{#1}\|_{\ell_2}}
\newcommand{\tti}[1]{\|{#1}\|_{2,\infty}}

\newcommand{\tf}[1]{\|{#1}\|_{F}}

\newcommand{\Rc}{\mathcal{R}}
\newcommand{\Nc}{\mathcal{N}}

\newcommand{\bteta}{\boldsymbol{\theta}}
\newcommand{\bTeta}{\boldsymbol{\Theta}}

\newcommand{\Bc}{\mathcal{B}}
\newcommand{\Sc}{\mathbb{S}^{n-1}}
\newcommand{\Mc}{\mathbb{S}^{dr-1}}

\newcommand{\Nn}{\mathcal{N}}

\newcommand{\vb}{\vct{v}}

\newcommand{\w}{\vct{w}}

\newcommand{\li}{\left<}
\newcommand{\ri}{\right>}

\newcommand{\h}{\vct{h}}

\newcommand{\Fc}{\mathcal{F}}

\newcommand{\qqq}[1]{{\textcolor{black}{{#1}}}}

\newcommand{\mat}[1]{{\text{mat}\left(#1\right)}}

\newcommand{\opnorm}[1]{\left\|#1\right\|}
\newcommand{\fronorm}[1]{\left\|#1\right\|_{F}}

\newcommand{\twonorm}[1]{\left\|#1\right\|_{\ell_2}}

\newcommand{\nucnorm}[1]{\left\|#1\right\|_*}
\newcommand{\abs}[1]{\left|#1\right|}

\newcommand{\x}{\vct{x}}
\newcommand{\rb}{\vct{r}}
\newcommand{\rbt}{\vct{r}(\vct{\theta})}
\newcommand{\y}{\vct{y}}

\newcommand{\bgl}{{~\big |~}}

\definecolor{emmanuel}{RGB}{255,127,0}

\newcommand{\R}{\mathbb{R}}
\newcommand{\Pro}{\mathbb{P}}

\newcommand{\E}{\operatorname{\mathbb{E}}}

\newcommand{\grad}[1]{{\nabla\Lc(#1)}}

\newcommand{\vct}[1]{\bm{#1}}
\newcommand{\mtx}[1]{\bm{#1}}

\newcommand{\bz}{B}
\newcommand{\cp}{R_{p}}
\newcommand{\cs}{\nu}

\newcommand{\Pc}{{\cal{P}}}
\newcommand{\X}{{\mtx{X}}}
\newcommand{\Y}{{\mtx{Y}}}
\newcommand{\Vb}{{\mtx{V}}}

\numberwithin{equation}{section} 

\def \endprf{\hfill {\vrule height6pt width6pt depth0pt}\medskip}
\newenvironment{proof}{\noindent {\bf Proof} }{\endprf\par}

\usepackage{multirow}

\usepackage{algorithm}
\usepackage{algpseudocode}
\author{Samet Oymak\thanks{{Department of Electrical and Computer Engineering, University of California, Riverside, CA}}\quad and\quad Mahdi Soltanolkotabi\thanks{Ming Hsieh Department of Electrical Engineering, University of Southern California, Los Angeles, CA}}

\date{}
\begin{document}
\maketitle
%logarithm of the %if the optimization landscape has nice properties around the initial point%some of which may perform better than others in the test set. Hence
%In this work, we explore the convergence properties% in particular a ball of radius $R$% (a ball centered at initial point)%\kappa{R_0}
% !TEX root = shortest.tex
\begin{abstract} Many modern learning tasks involve fitting nonlinear models to data which are trained in an overparameterized regime where the parameters of the model exceed the size of the training dataset. Due to this overparameterization, the training loss may have infinitely many global minima and it is critical to understand the properties of the solutions found by first-order optimization schemes such as (stochastic) gradient descent starting from different initializations. In this paper we demonstrate that when the loss has certain properties over a minimally small neighborhood of the initial point, first order methods such as (stochastic) gradient descent have a few intriguing properties: (1) the iterates converge at a geometric rate to a global optima even when the loss is nonconvex, (2) among all global optima of the loss the iterates converge to one with a near minimal distance to the initial point, (3) the iterates take a near direct route from the initial point to this global optima. As part of our proof technique, we introduce a new potential function which captures the precise tradeoff between the loss function and the distance to the initial point as the iterations progress. For Stochastic Gradient Descent (SGD), we develop novel martingale techniques that guarantee SGD never leaves a small neighborhood of the initialization, even with rather large learning rates. We demonstrate the utility of our general theory for a variety of problem domains spanning low-rank matrix recovery to neural network training. Underlying our analysis are novel insights that may have implications for training and generalization of more sophisticated learning problems including those involving deep neural network architectures.
\end{abstract}

\section{Introduction}
\subsection{Motivation}
In a typical statistical estimation or supervised learning problem, we are interested in fitting a function $f(\cdot;\vct{\theta}):\R^d\mapsto \R$ parameterized by $\vct{\theta}\in\R^p$ to a training data set of $n$ input-output pairs $\vct{x}_i\in\R^d$ and $\vct{y}_i\in\R$ for $i=1,2,\ldots,n$. The training problem then consists of finding a parameter $\vct{\theta}$ that minimizes the empirical risk $\frac{1}{n}\sum_{i=1}^n\ell(f(\vct{x}_i;\vct{\theta}),\vct{y}_i)$. The loss $\ell(\tilde{y},y)$ measures the discrepancy between the output(or label) $y$ and the model prediction $\tilde{y}=f(\vct{x}_i;\vct{\theta})$. For regression tasks one typically uses a least-squares loss $\ell(\tilde{y},y)=\frac{1}{2}(\tilde{y}-y)^2$ so that the training problem reduces to a nonlinear least-squares problem of the form
\begin{align}
\label{NLS}
\underset{\vct{\theta}\in\R^p}{\min}\text{ }\mathcal{L}(\vct{\theta}):=\frac{1}{2}\sum_{i=1}^n\left(f(\vct{x}_i;\vct{\theta})-\vct{y}_i\right)^2.
\end{align}
In this paper we mostly focus on nonlinear least-squares problems. In Section \ref{sec polyak} we discuss results that apply to a broader class of loss functions $\mathcal{L}(\vct{\theta})$.

%\textcolor{red}{We need one more here and we should add citations}
%covering numbers, and 
Classical statistical estimation/learning theory postulates that to find a reliable model that avoids overfitting, the size of the training data must exceed the intrinsic dimension\footnote{Some common notions of intrinsic dimension include Vapnik–Chervonenkis (VC) Dimension \cite{vapnik2015uniform}, Rademacher/Gaussian complexity \cite{bartlett2002rademacher,talagrand2006generic,mohri2018foundations}, as well as naive parameter counting.} of the model class $f(\cdot;\vct{\theta})$ used for empirical risk minimization \eqref{NLS}. For many models such notions of intrinsic dimension are at least as large as the number of parameters in the model $p$, so that this literature requires the size of the training data to exceed the number of parameters in the model i.e.~$n>p$. Contrary to this classical literature, modern machine learning models such as deep neural networks are often trained via first-order methods in an over-parameterized regime where the number of parameters in the model exceed the size of the training data (i.e.~$n<p$). Statistical learning in this over-parameterized regime poses new challenges: Given the nonconvex nature of the training loss \eqref{NLS} can first-order methods converge to a globally optimal model that perfectly interpolate the training data? If so, which of the global optima do they converge to? What are the statistical properties of this model and how does this model vary as a function of the initial parameter used to start the iterative updates? What is the trajectory that iterative methods such as (stochastic) gradient descent take to reach this point? Why does a model trained using this approach \emph{generalize} to new data and avoid overfitting to the training data?

In this paper we take a step towards addressing such challenges. We demonstrate that in many cases first-order methods do indeed converge to a globally optimal model that perfectly fits the training data. Furthermore, we show that among all globally optimal parameters of the training loss these algorithms tend to converge to one which has a near minimal distance to the parameter used for initialization. Additionally, the path that these algorithms take to reach such a global optima is rather short, with these algorithms following a near direct trajectory from initialization to the global optima. We believe these key features of first-order methods may help demystify why models trained using these simple algorithms can achieve reliable learning in modern over-parametrized regimes without over-fitting to the training data. 

%Our main technical contributions can be summarized as follows:
%%\begin{itemize}
%%\item We provide a general convergence result for overparameterized learning via gradient descent, that comes with matching upper and lower bounds, showing that under appropriate assumptions if the problem is well-behaved in a minimally small neighborhood, gradient descent finds a global minimizer with nearly shortest distance.
%%\item The results are extended to SGD to show that SGD exhibits the same behavior and converges linearly without ever leaving this small neighborhood even with large learning rates.
%%\item As applications, we provide new results for overparameterized learning with generalized linear models, neural networks, and low-rank matrix regression. 
%%\end{itemize}}
%%%\SO{Should we remark that, theory can be generalized to convergence in any norm under natural assumptions (but we stick with $\ell_2$)?}
%%\vspace{5pt}
\subsection{Insights from Linear Regression}
As a prelude to understanding the key properties of (stochastic) gradient descent in over-parameterized nonlinear learning we begin by focusing on the simple case of linear regression. In this case the mapping in \eqref{NLS} takes the form $f(\vct{x}_i;\vct{\theta})=\vct{x}_i^T\vct{\theta}$. Gathering the input data $\vct{x}_i$ and labels $y_i$ as rows of a matrix $\mtx{X}\in\R^{n\times d}$ and a vector $\vct{y}\in\R^n$, the fitting problem amounts to minimizing the loss $\mathcal{L}(\vct{\theta})=\frac{1}{2}\twonorm{\mtx{X}\vct{\theta}-\vct{y}}^2$. Therefore, starting from an initialization $\vct{\theta}_0$, gradient descent iterations with a step size $\eta$ take the form\vs
\begin{align*}
\vct{\theta}_{\tau+1}=\vct{\theta}_\tau-\eta\nabla \mathcal{L}(\vct{\theta}_\tau)=\vct{\theta}_\tau-\eta\mtx{X}^T\left(\mtx{X}\vct{\theta}_\tau-\y\right).
\end{align*}
As long as the matrix $\mtx{X}$ has full row rank the set $\mathcal{G}:=\{\vct{\theta}\in\R^p: \mtx{X}\vct{\theta}=\y\}$ is nonempty and the global minimum of the loss is $0$. Using simple algebraic manipulations the residual vector $\vct{r}_\tau=\mtx{X}\vct{\theta}_{\tau+1}-\y$ obeys\vs
\begin{align*}
\vct{r}_{\tau+1}=\left(\mtx{I}-\eta\mtx{X}\mtx{X}^T\right)\vct{r}_\tau\quad\Rightarrow\quad\twonorm{\vct{r}_{\tau+1}}\le \opnorm{\mtx{I}-\eta\mtx{X}\mtx{X}^T}\twonorm{\vct{r}_\tau}.
\end{align*}
Therefore, using a step size of $\eta\le \frac{1}{\opnorm{\mtx{X}}^2}$ the residual iterates converge at a geometric rate to zero. This yields the first key property of gradient methods for over-parametrized learning: 

\vspace{4pt}
\noindent \quad\quad\quad\emph{{\bf{\emph{Key property I:}}} Gradient descent iterates converge at a geometric rate to a global optima.}
\vspace{4pt}

Let $\bteta^*$ denote the global minima we converge to and $\Pi_{\mathcal{R}}$ and $\Pi_{\mathcal{N}}$ denote the projections onto the row space and null space of $\mtx{X}$, respectively. Since the gradients lie on the row space of $\X$ and $\X$ is full row rank, denoting the unique pseudo-inverse solution by $\bteta^{\dagger}$, we have \vs
\[
\Pi_{\cal{N}}(\bteta^*)=\Pi_{\cal{N}}(\bteta_0)\quad\quad\text{and}\quad\quad\Pi_{\Rc}(\bteta^*)=\bteta^{\dagger}.
\]
The equalities above imply that $\bteta^*$ is the closest global minima to $\bteta_0$; which highlights the second property:

\vspace{4pt}
\noindent \quad\quad\quad\emph{{\bf{\emph{Key property II:}}} Gradient descent converges to the closest global optima to initialization.}
\vspace{4pt}

Finally, it can also be shown that the total path length $\sum_{\tau=0}^\infty \twonorm{\vct{\theta}_{\tau+1}-\vct{\theta}_\tau}$ can be upper bounded by the distance $\tn{\bteta^*-\bteta_0}$ (up to multiplicative factors depending on condition number of $\X$). This leads us to:

\vspace{4pt}
\noindent \quad\quad\quad\emph{{\bf{\emph{Key property III:}}} Gradient descent takes a near direct trajectory to reach the closest global optima.}
\vspace{4pt}

In this paper we show that similar properties continue to hold for a broad class of \emph{nonlinear} over-parameterized learning problems.

\subsection{Contributions}
Our main technical contributions can be summarized as follows:
\begin{itemize}
\item We provide a general convergence result for overparameterized learning via gradient descent, that comes with matching upper and lower bounds, showing that under appropriate assumptions over a small neighborhood of the initialization, gradient descent (1) finds a globally optimal model, (2) among all possible globally optimal parameters it finds one which is approximately the closest to initialization and (3) it follows a nearly direct trajectory to find this global optima.
\item We show that SGD exhibits the same behavior as gradient descent and converges linearly without ever leaving a small neighborhood of the initialization even with rather large learning rates.
\item We demonstrate the utility of our general results in the context of three overparameterized learning problems: generalized linear models, low-rank matrix regression, and shallow neural network training. 
\end{itemize}
%\SO{Should we remark that, theory can be generalized to convergence in any norm under natural assumptions (but we stick with $\ell_2$)?}

\section{Convergence Analysis for Gradient Descent}
The nonlinear least-squares problem in \eqref{NLS} can be written in the more compact form
\begin{align}
\label{NLScompact}
\underset{\vct{\theta}\in\R^p}{\min}\text{ }\mathcal{L}(\vct{\theta}):=\frac{1}{2}\twonorm{f(\vct{\theta})-\vct{y}}^2,
\end{align}
where
\begin{align*}
\vct{y}:=\begin{bmatrix}\vct{y}_1 \\ \vct{y}_2 \\ \vdots \\\vct{y}_n\end{bmatrix}\in\R^n\quad\text{and}\quad f(\vct{\theta}):=\begin{bmatrix} f(\vct{x}_1;\vct{\theta})\\f(\vct{x}_2;\vct{\theta})\\\vdots\\f(\vct{x}_n;\vct{\theta})\end{bmatrix}\in\R^n.
\end{align*}
A natural approach to optimizing \eqref{NLScompact} is to use gradient descent updates of the form
\[
\bteta_{\tau+1}=\bteta_\tau-\eta_\tau\grad{\bteta_\tau},
\]
starting from some initial parameter $\vct{\theta}_0$. For the nonlinear least-squares formulation \eqref{NLScompact} above the gradient takes the form
\begin{align}
\grad{\bteta}=\Jc(\bteta)^T(f(\bteta)-\y).\label{gident}
\end{align}
Here, $\Jc(\bteta)\in\R^{n\times p}$ is the Jacobian matrix associated with the mapping $f(\vct{\theta})$ with entries given by $\mathcal{J}_{ij}=\frac{\partial f(\x_i,\bteta)}{\partial \bteta_j}$. We note that in the  over-parameterized regime ($n<p$), the Jacobian has more columns than rows.

The particular form of the gradient in \eqref{gident} suggests that the eigenvalues of the Jacobian matrix may significantly impact the convergence of gradient descent. Our main technical assumption in this paper is that the spectrum of the Jacobian matrix is bounded from below and above in a local neighborhood of the initialization. 
\begin{assumption}[Jacobian Spectrum] \label{wcond} Consider a set $\mathcal{D}\subset\R^p$ containing the initial point $\vct{\theta}_0$ (i.e.~$\vct{\theta}_0\in\mathcal{D}$). We assume that for all $\bteta\in\Dc$ the following inequality holds
\[
\bn\le \sigma_{\min}\left(\mathcal{J}(\vct{\theta})\right)\le \|\mathcal{J}(\vct{\theta})\|\le \bp.
\] 
Here, $\sigma_{\min}(\cdot)$ and $\opnorm{\cdot}$ denote the minimum singular value and the spectral norm respectively.
%holds for for all $\bteta\in\Dc$.
\end{assumption}
%Note that, this assumption implies Local PL condition over $\Dc$ with $\mu= \bn^2$. 
Our second technical assumption ensures that the Jacobian matrix is not too sensitive to changes in the parameters of the nonlinear mapping. Specifically we require the Jacobian to have either bounded or smooth variations as detailed next.%We require one of two such assumptions: (a) bounded variation in the Jacobian or (b) smmoo
%We will also need some related conditions to control the optimization landscape. Two such conditions are (a) small deviation and (b) Lipschitzness of the Jacobian as described below.%of These conditions is the following perturbation bound.
\begin{assumption}[Jacobian Deviations] \label{spert}Consider a set $\mathcal{D}\subset\R^p$ containing the initial point $\vct{\theta}_0$ (i.e.~$\vct{\theta}_0\in\mathcal{D}$). We assume one of the following two conditions holds:\\
{\bf{(a) Bounded deviation:}} For all $\vct{\theta}_1, \vct{\theta}_2\in\mathcal{D}$ \vs
\begin{align*}
\opnorm{\mathcal{J}(\vct{\theta}_2)-\mathcal{J}(\vct{\theta}_1)}\le \frac{(1-\lambda)\bn^2}{\bp},
\end{align*}
holds for some $0\le \lambda \le 1$. Here, $\alpha$ and $\beta$ are the bounds on the Jacobian spectrum over $\mathcal{D}$ per Assumption \ref{wcond}.\\
{\bf{(b) Smooth deviation:}} For all $\vct{\theta}_1, \vct{\theta}_2\in\mathcal{D}$ \vs
\begin{equation*}
\opnorm{\mathcal{J}(\vct{\theta}_2)-\mathcal{J}(\vct{\theta}_1)}\le \el\twonorm{\vct{\theta}_2-\vct{\theta}_1}.\footnote{Note that, if $\frac{\partial \Jc(\bteta)}{\partial \bteta}$ is continuous, Lipschitzness condition holds over any compact domain (for possibly large $\el$).}
\end{equation*}
%\footnote{}.
\end{assumption}
With these assumptions in place we are now ready to state our main result.
\begin{theorem}\label{GDthm} Consider a nonlinear least-squares optimization problem of the form 
\begin{align*}
\underset{\vct{\theta}\in\R^p}{\min}\text{ }\mathcal{L}(\vct{\theta}):=\frac{1}{2}\twonorm{f(\vct{\theta})-\vct{y}}^2,
\end{align*}
 with $f:\R^p\mapsto \R^n$ and $\vct{y}\in\R^n$. Suppose the Jacobian mapping associated with $f$ obeys Assumption \ref{wcond} over a ball $\mathcal{D}$ of radius $R:=\frac{4\twonorm{f(\vct{\theta}_0)-\vct{y}}}{\bn}$ around a point $\vct{\theta}_0\in\R^p$.\footnote{That is, $\mathcal{D}=\mathcal{B}\left(\vct{\theta}_0,\frac{4\twonorm{f(\vct{\theta}_0)-\vct{y}}}{\bn}\right)$ with $\mathcal{B}(\vct{c},r)=\big\{\vct{\theta}\in\R^p: \twonorm{\vct{\theta}-\vct{c}}\le r\big\}$} Furthermore, suppose one of the following statements is valid.
\begin{itemize}
\item Assumption \ref{spert} (a) holds over $\mathcal{D}$ with $\la=1/2$ and set $\eta\leq \frac{1}{2 \bp^2}$.
\item Assumption \ref{spert} (b) holds over $\mathcal{D}$ and set $\eta\leq \frac{1}{2 \bp^2}\cdot\min\left(1,\frac{ \bn^2}{\el\twonorm{f(\vct{\theta}_0)-\vct{y}}}\right)$.
\end{itemize}
Then, running gradient descent updates of the form $\vct{\theta}_{\tau+1}=\vct{\theta}_\tau-\eta\nabla\mathcal{L}(\vct{\theta}_\tau)$ starting from $\vct{\theta}_0$, all iterates obey.
\begin{align}
\twonorm{f(\vct{\theta}_\tau)-\vct{y}}^2\le&\left(1-\frac{\eta\bn^2}{2}\right)^\tau\twonorm{f(\vct{\theta}_0)-\vct{y}}^2,\label{err}\\
\frac{1}{4}\bn\twonorm{\vct{\theta}_\tau-\vct{\theta}_0}+\twonorm{f(\vct{\theta}_\tau)-\vct{y}}\le&\twonorm{f(\vct{\theta}_0)-\vct{y}}.\label{close}
\end{align}
Furthermore, the total gradient path is bounded. That is,
\begin{align}
\label{GDpath_main}
\sum_{\tau=0}^\infty\twonorm{\vct{\theta}_{\tau+1}-\vct{\theta}_\tau}\le \frac{4\twonorm{f(\vct{\theta}_0)-\vct{y}}}{\bn}.
\end{align}
\end{theorem}
A trivial consequence of the above theorem is the following corollary.
\begin{corollary}\label{mycor} Consider the setting and assumptions of Theorem \ref{GDthm} above. Let $\vct{\theta}^*$ denote the global optima of the loss $\mathcal{L}(\vct{\theta})$ with smallest Euclidean distance to the initial parameter $\vct{\theta}_0$. Then, the gradient descent iterates $\vct{\theta}_\tau$ obey
\begin{align}
\twonorm{\vct{\theta}_\tau-\vct{\theta}_0}\le 4\frac{\bp}{\bn}\twonorm{\vct{\theta}^*-\vct{\theta}_0},\label{approxclosest}\\
\sum_{\tau=0}^\infty\twonorm{\vct{\theta}_{\tau+1}-\vct{\theta}_\tau}\le 4\frac{\bp}{\bn}\twonorm{\vct{\theta}^*-\vct{\theta}_0}.\label{approxshortest}
\end{align}
\end{corollary}
The theorem and corollary above show that if the Jacobian of the nonlinear mapping is well-conditioned (Assumption \ref{wcond}) and has bounded/smooth deviations (Assumptions \ref{spert}) in a ball of radius $R$ around the initial point, then gradient descent enjoys three intriguing properties.

\noindent\textbf{Zero traning error:} The first property property demonstrated by Theorem \ref{GDthm} above is that the iterates converge to a global optima $\vct{\theta}_{GD}$. This hold despite the fact that the fitting problem may be highly nonconvex in general. Indeed, based on \eqref{err} the fitting/training error $\twonorm{f(\vct{\theta}_\tau)-\vct{y}}$ achieved by Gradient Descent (GD) iterates converges to zero. Therefore, GD can perfectly interpolate the data and achieve zero training error. Furthermore, this convergence is rather fast and the algorithm enjoys a geometric (a.k.a.~linear) rate of convergence to this global optima.

\noindent\textbf{Gradient descent iterates remain close to the initialization:}  The second interesting aspect of these results is that they guarantee the GD iterates never leave a neighborhood of radius $\frac{4}{\bn}\twonorm{f(\vct{\theta}_0)-\vct{y}}$ around the initial point. That is the GD iterates remain rather close to the initialization. In fact, based on \eqref{approxclosest} we can conclude that
\begin{align*}
\twonorm{\vct{\theta}_{GD}-\vct{\theta}_0}=\twonorm{\underset{\tau\rightarrow\infty}{\text{lim}}\vct{\theta}_\tau-\vct{\theta}_0}=\underset{\tau\rightarrow\infty}{\text{lim}}\twonorm{\vct{\theta}_\tau-\vct{\theta}_0}\le 4\frac{\bp}{\bn}\twonorm{\vct{\theta}^*-\vct{\theta}_0}.
\end{align*}
Thus the distance between the global optima GD converges to and the initial parameter $\vct{\theta}_0$ is within a factor $4\frac{\bp}{\bn}$ of the distance between the closest global optima to $\vct{\theta}_0$ and the initialization. This shows that among all global optima of the loss, the GD iterates converge to one with a near minimal distance to the initialization. In particular, \eqref{close} shows that for all iterates the weighted sum of the distance to the initialization and the misfit error remains bounded so that as the loss decreases the distance to the initialization only moderately increases.

\noindent\textbf{Gradient descent follows a short path:} Another interesting aspect of the above results is that the total length of the path taken by gradient descent remains bounded. Indeed, based on \eqref{approxshortest} the length of the path taken by GD is within a factor of the distance between the closest global optima and the initialization. This implies that GD follows a near direct route from the initialization to a global optima!

We would like to note that Theorem \ref{GDthm} and Corollary \ref{mycor} are special instances of a more general result stated in the proofs (Theorem \ref{mainthm} stated in Section \ref{GDproof}).\footnote{Theorem \ref{GDthm} and Corollary \ref{mycor} above are a special case of this theorem with $\lambda=1/2$ and $\rho=1$.} This more general result requires Assumptions \ref{wcond} and \ref{spert} to hold in a smaller neighborhood and improves the approximation ratios. Specifically, this more general result allows the radius $R$ to be chosen as small as 
\begin{align}
\label{radius}
\frac{\twonorm{f(\vct{\theta}_0)-\vct{y}}}{\bn},
\end{align}
and \eqref{close} to be improved to
\begin{align}
\label{closeimp}
\bn\twonorm{\vct{\theta}_\tau-\vct{\theta}_0}+\twonorm{f(\vct{\theta}_\tau)-\vct{y}}\le&\twonorm{f(\vct{\theta}_0)-\vct{y}}
\end{align}
Also the approximation ratios in Corollary \ref{mycor} can be improved to
\begin{align}
\twonorm{\vct{\theta}_\tau-\vct{\theta}_0}\le \frac{\bp}{\bn}\twonorm{\vct{\theta}^*-\vct{\theta}_0},\label{approxclosestimp}\\
\sum_{\tau=0}^\infty\twonorm{\vct{\theta}_{\tau+1}-\vct{\theta}_\tau}\le \frac{\bp}{\bn}\twonorm{\vct{\theta}^*-\vct{\theta}_0}.\label{approxshortestimp}
\end{align}
However, this requires a smaller learning rate and hence leads to a slower converge guarantee. 

\noindent\textbf{The role of the sample size:} Theorem \ref{GDthm} provides a good intuition towards the role of sample size in the overparameterized optimization landscape. First, observe that adding more samples can only increase the condition number of the Jacobian matrix (larger $\beta$ and smaller $\alpha$). Secondly, assuming samples are i.i.d,~the initial misfit $\tn{\y-f(\bteta_0)}$ is proportional to $\sqrt{n}$. Together these imply that more samples lead to a more challenging optimization problem as follows.
\begin{itemize}
\item More samples leads to a slower convergence rate by \qqq{degrading} the condition number of the Jacobian,
\item \qqq{The required convergence radius} $R$ increases proportional to $\sqrt{n}$ and we need Jacobian to be well-behaved over a larger neighborhood for fast convergence.
\end{itemize}
% and how more overparameterization leads to faster convergence. 

A natural question about the results discussed so far is whether the size of the local neighborhood for which we require our assumptions to hold is optimal. In particular, one may hope to be able to show that a significantly smaller neighborhood is sufficient. We now state a lower bound showing that this is not possible. 
%In this lower-bound we relate the model misfit error $\twonorm{f(\vct{\theta})-\vct{y}}$ to the model progress
%
% is what is the size of the local neighborhood for which we require Assumption \ref{wcond} to hold. 
%
%
%This section focuses on the fundamental lower bounds  that relates prediction residual $\rb_i$ to the model's progress i.e.~the vector ${\bteta-\bteta_0}$.
%\SO{Show that, there does not exist minima closer than this radius, in particular residual is at least $\geq \dots$. i.e.~parameter takes the shortest path.}
 \begin{figure} 
\centering
\includegraphics[scale=0.9]{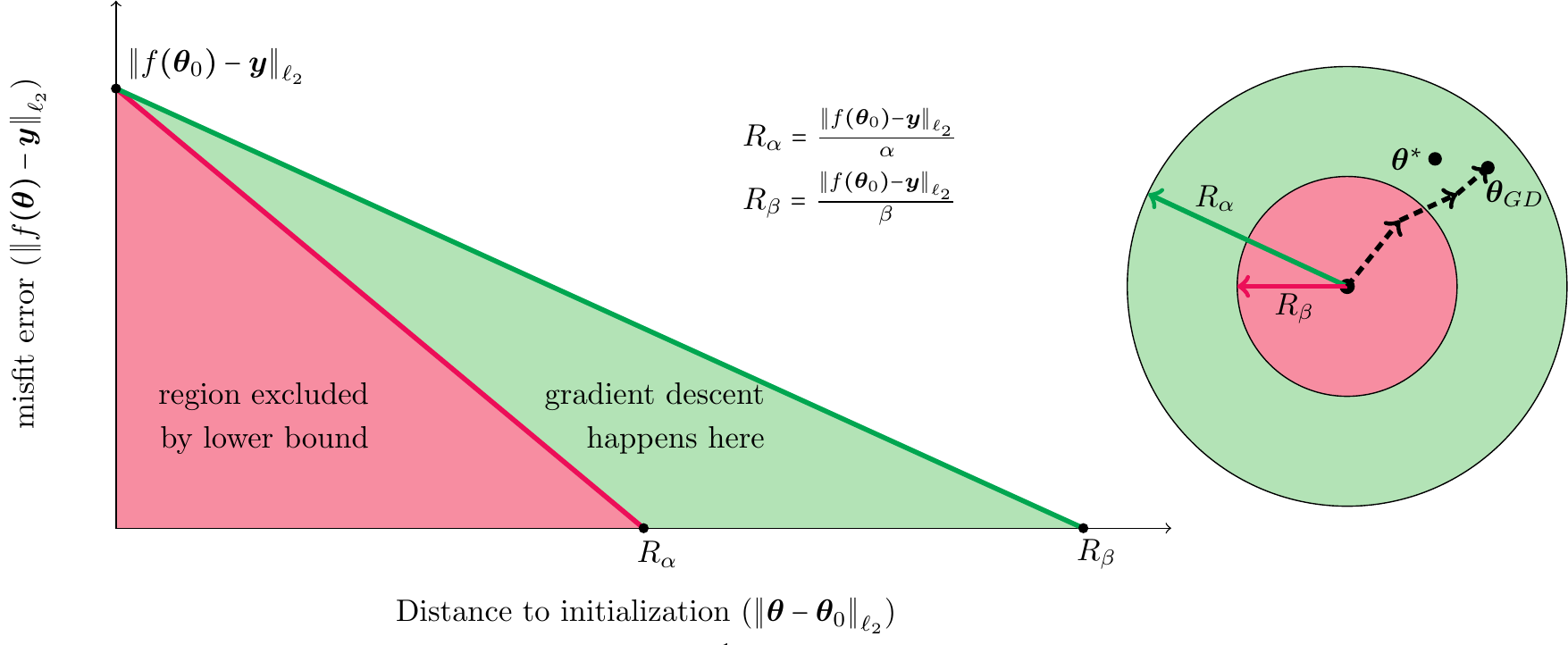}
\caption{\small{In the left figure we show that the gradient descent iterates in over-parameterized learning exhibit a sharp tradeoff between distance to the initial point ($\twonorm{\vct{\theta}-\bteta_0}$) and the misfit error ($\twonorm{f(\vct{\theta})-\vct{y}}$). Our upper (equation \eqref{closeimp}) and lower bounds (Theorem \ref{low bound thm}) guarantee that the gradient descent iterates must lie in the green region. Additionally this is the tightest region as we provide examples in Theorem \ref{low bound thm} where gradient descent occurs only on the upper bound (green) line or on the lower bound (red line). Right figure shows the same behavior in the parameter space. Our theorems predict that the gradient descent trajectory ends at a globally optimal point $\vct{\theta}_{GD}$ in the green region and this point will have approximately the same distance to the initialization parameter as the closest global optima to the initialization ($\vct{\theta}^*$). Furthermore, the GD iterates follow a near direct route from the initialization to this global optima.}}
 \label{GDpath_fig}
\end{figure}
\begin{theorem}\label{low bound thm} Consider a nonlinear least-squares optimization problem of the form 
\begin{align*}
\underset{\vct{\theta}\in\R^p}{\min}\text{ }\mathcal{L}(\vct{\theta}):=\frac{1}{2}\twonorm{f(\vct{\theta})-\vct{y}}^2,
\end{align*}
 with $f:\R^p\mapsto \R^n$ and $\vct{y}\in\R^n$. Suppose the Jacobian mapping associated with $f$ obeys Assumption \ref{wcond} over a set $\mathcal{D}$ around a point $\vct{\theta}_0\in\R^p$. Then,
\begin{align}
\tn{\y-f(\bteta)}+ \bp\tn{\bteta-\bteta_0}\geq \tn{\y-f(\bteta_0)},\label{total resi}
\end{align}
holds for all $\vct{\theta}\in\mathcal{D}$. Hence, any $\bteta$ that sets the loss to zero satisfies $\tn{\bteta-\bteta_0}\geq \tn{\y-f(\bteta_0)}/ \bp$. Furthermore, for any $\alpha$ and $\beta$ obeying $\alpha,\beta\ge 0$ and $\beta\ge \alpha$, there exists a linear regression problem such that
\begin{align}
\tn{\y-f(\bteta)}+ \bn\tn{\bteta-\bteta_0}\geq \tn{\y-f(\bteta_0)},\label{total resi2}
\end{align}
holds for all $\bteta$. Also, for any $\alpha$ and $\beta$ obeying $\alpha,\beta\ge 0$ and $\beta\ge \alpha$, there also exists a linear regression problem where running gradient descent updates of the form $\vct{\theta}_{\tau+1}=\vct{\theta}_\tau-\eta\nabla\mathcal{L}(\vct{\theta}_\tau)$ starting from $\bteta_0=0$ with a sufficiently small learning rate $\eta$, all iterates $\vct{\theta}_\tau$ obey
\begin{align}
\tn{\y-f(\bteta_\tau)}+ \bp\tn{\bteta_\tau-\bteta_0}= \tn{\y-f(\bteta_0)}.\label{total resi3}
\end{align}
\end{theorem}
The result above shows that any global optima is at least a distance $\tn{\bteta-\bteta_0}\geq \frac{\tn{\y-f(\bteta_0)}}{\bp}$ away from the initialization so that the minimum ball around the initial point needs to have radius at least $R\ge \frac{\tn{\y-f(\bteta_0)}}{\bp}$ for convergence to a global optima to occur. Comparing this lower-bound with that of Theorem \ref{GDthm} and in particular the improvement discussed in \eqref{radius} suggests that the size of the local neighborhood is optimal up to a factor $\beta/\alpha$ which is the condition number of the Jacobian in the local neighborhood. More generally, this result shows that the weighted sum of the residual/misfit to the model ($\twonorm{f(\vct{\theta})-\vct{y}}$) and distance to initialization ($\twonorm{\vct{\theta}-\vct{\theta}_0}$) has nearly matching lower/upper bounds (compare \eqref{closeimp} and \eqref{total resi}). Theorem \ref{low bound thm} also provides two specific examples in the context of linear regression which shows that both of these upper and lower bounds are possible under our assumptions. 

Collectively our theorems (Theorem \ref{GDthm}, Corollary \ref{mycor}, improvements in equations \eqref{radius} and \eqref{closeimp}, and Theorem \ref{low bound thm})  demonstrate that the path taken by gradient descent is by no means arbitrary. Indeed as depicted in the left picture of Figure \ref{GDpath_fig}, gradient descent iterates in over-parameterized learning exhibit a sharp tradeoff between distance to the initial point ($\twonorm{\vct{\theta}-\bteta_0}$) and the misfit error ($\twonorm{f(\vct{\theta})-\vct{y}}$). Our upper (equation \eqref{closeimp}) and lower bounds (Theorem \ref{low bound thm}) guarantee that the gradient descent iterates must lie in the green region in this figure. Additionally this is the tightest region as we provide examples in Theorem \ref{low bound thm} where gradient descent occurs only on the upper bound (green) line or on the lower bound (red line). In the right picture of Figure \ref{GDpath_fig} we also depict the gradient descent trajectory in the parameter space. As shown, the GD iterates end at a globally optimal point $\vct{\theta}_{GD}$ in the green region and this point will have approximately the same distance to the initialization parameter as the closest global optima to the initialization ($\vct{\theta}^*$). Furthermore, the GD iterates follow a near direct route from the initialization to this global optima.
\section{Convergence Analysis for Stochastic Gradient Descent}
%The next section demonstrates that, this tradeoff is not unique to GD. We prove that SGD exhibits the same near-optimal tradeoffs despite the randomness introduced at every single iteration. In particular, with a good Jacobian, SGD converges linearly fast to a global minima without ever leaving the domain of interest. 
%For the following discussion, we operate under the Assumptions \ref{wcond} and \ref{spert}(b) restricted to the domain $\tn{\bteta_0-\bteta}\leq \cs R$ where $R=\frac{ \bp\tn{\rb_0}}{ \bn^2}$ and $\cs>0$ is a constant to be determined. 
Arguably the most widely used algorithm in modern learning is Stochastic Gradient Descent (SGD). For learning nonlinear least-squares problems of the form \eqref{NLScompact} a natural implementation of SGD is to sample a data point at random and use that data point for the gradient updates. Specifically, let $\{\gamma_\tau\}_{\tau=0}^\infty$ be an i.i.d.~sequence of integers chosen uniformly from $\{1,2,\ldots,n\}$, the SGD iterates take the form
\begin{align}
\vct{\theta}_{\tau+1}=\vct{\theta}_\tau-\eta G(\vct{\theta}_\tau;\gamma_\tau)\quad\text{with}\quad G(\vct{\theta}_\tau;\gamma_\tau):=(f(\x_{\rng_\tau};\bteta_\tau)-y_{\rng_\tau})\nabla f(\x_{\rng_\tau};\bteta_\tau).\label{sgd eq}
\end{align}
Here, $G(\vct{\theta}_\tau;\gamma_\tau)$ is the gradient on the $\gamma_\tau$th training sample. We are interested in understanding the trajectory of SGD for over-parameterized learning. In particular, whether the three intriguing properties discussed in the previous section for GD continues to hold for SGD. Our next theorem addresses this challenge.
\begin{theorem}\label{SGDthm} Consider a nonlinear least-squares optimization problem of the form $\underset{\vct{\theta}\in\R^p}{\min}\text{ }\mathcal{L}(\vct{\theta}):=\frac{1}{2}\twonorm{f(\vct{\theta})-\vct{y}}^2$, with $f:\R^p\mapsto \R^n$ and $\vct{y}\in\R^n$. Suppose the Jacobian mapping associated with $f$ obeys Assumption \ref{wcond} over a ball $\mathcal{D}$ of radius $R:=\nu\frac{\twonorm{f(\vct{\theta}_0)-\vct{y}}}{\bn}$ around a point $\vct{\theta}_0\in\R^p$ with $\nu$ a scalar obeying $\nu\ge 3$. Also assume the rows of the Jacobian have bounded Euclidean norm over this ball, that is
\begin{align*}
\underset{i}{\max}\text{ }\twonorm{\mathcal{J}_i(\vct{\theta})}\le B\quad\text{for all}\quad\vct{\theta}\in\mathcal{D}.% \text{ such that } \twonorm{\vct{\theta}-\vct{\theta}_0}\le R.
\end{align*} 
Furthermore, suppose one of the following statements is valid.
\begin{itemize}
\item Assumption \ref{spert} (a) holds over $\mathcal{D}$ and set $\eta\leq \frac{\bn^2}{\nu\bp^2B^2}$.
\item Assumption \ref{spert} (b) holds over $\mathcal{D}$ and set $\eta\leq \frac{\bn^2}{\nu\bp^2B^2+\nu\bp BL\twonorm{f(\vct{\theta}_0)-\vct{y}}}$.
\end{itemize}
Then, there exists an event $E$ which holds with probability at least $\mathbb{P}(E)\ge1-\frac{4}{\nu}\left(\frac{\bp}{\bn}\right)^{\frac{1}{p}}$ and running stochastic gradient descent updates of the form \eqref{sgd eq} starting from $\vct{\theta}_0$, all iterates obey
\begin{align}
\E\Big[\twonorm{f(\vct{\theta}_\tau)-\vct{y}}^2\mathbb{1}_{E}\Big]\le&\left(1-\frac{\eta\bn^2}{2n}\right)^\tau\twonorm{f(\vct{\theta}_0)-\vct{y}}^2,\label{sgerr}
%\\
%\frac{1}{4}\bn\twonorm{\vct{\theta}_\tau-\vct{\theta}_0}+\twonorm{f(\vct{\theta}_\tau)-\vct{y}}\le&\twonorm{f(\vct{\theta}_0)-\vct{y}}.\label{sgdclose}
\end{align}
Furthermore,  on this event the SGD iterates never leave the local neighborhood $\Dc$.
\end{theorem}
This result shows that SGD converges to a global optima that is close to the initialization. Furthermore, SGD always remains in close proximity to the initialization with high probability. Specifically, the neighborhood is on the order of $\frac{\twonorm{f(\vct{\theta}_0)-\vct{y}}}{\bn}$ which is consistent with the results on gradient descent and the lower bounds. However, unlike for gradient descent our approach to proving such a result is not based on showing that the weighted sum of the misfit and distance to initialization remains bounded per \eqref{close}. Rather we show a more intricate function (discussed in detail in Lemma \ref{tot res} and illustrated in Figure \ref{fig sgd potent} in the proofs) remains bounded. This function keeps track of the average distances to multiple points around the initialization $\bteta_0$.% In fact, it can be shown that

One interesting aspect of the result above is that the learning rate used is rather large. Indeed, ignoring an $\bp/\bn$ ratio our convergence rate is on the order of $1-c/n$ so that $n$ iterations of SGD correspond to a constant decrease in the misfit error on par with a full gradient iteration. This is made possible by a novel martingale-based technique that keeps track of the average distances to a set of points close to the initialization and ensures that SGD iterations never exit the local neighborhood. We note that it is possible to also used Azuma's inequality applied to the sequence $\log \twonorm{f(\vct{\theta}_\tau)-\vct{y}}$ to show that the SGD iterates stay in a local neighborhood with very high probability. However, such an argument requires a very small learning rate to ensure that one can take many steps without leaving the neighborhood at which point the concentration effect of Azuma becomes applicable. In contrast, our proof guarantees that SGD can use aggressive learning rates (on par with gradient descent) without ever leaving the local neighborhood. 

%In contrast, more classical SGD literature use Azuma's inequality and small learning rates (e.g.~$\eta=\order{1/n}$ rather than $\order{1}$), one can use Azuma's inequality on $\E[\log \tn{\rb_i}]$  to provide guarantees with exponential probability and smaller radius. The idea is that, if the learning rate is too small, one can take many steps without leaving the neighborhood at which point Azuma becomes applicable. In contrast, our proof guarantees that SGD can use aggressive learning rates (on par with gradient descent) without ever leaving neighborhood at the expense of weaker tail bound

\section{Case studies}\label{glm overp}
%\SO{here}
In this section we specialize and further develop our general convergence analysis in the context of three fundamental problems: fitting a generalized linear model, low-rank regression, and neural network training.

\subsection{Learning generalized linear models}
Nonlinear data-fitting problems are fundamental to many supervised learning tasks in machine learning. Given training data consisting of $n$ pairs of input features $\vct{x}_i\in\R^p$ and desired outputs $\vct{y}_i\in\R$ we wish to infer a function that best explains the training data. In this section we focus on learning Generalized Linear Models (GLM) from data which involves fitting functions of the form $f(\cdot;\vct{\theta}):\R^d\rightarrow \R$
\begin{align*}
f(\vct{x};\vct{\theta})=\phi(\langle \vct{x},\vct{\theta}\rangle).
\end{align*} 
A natural approach for fitting such GLMs is via minimizing the nonlinear least-squares misfit of the form
\begin{align}
\label{GLMopt}
\underset{\vct{\theta}\in\R^p}{\min}\text{ }\mathcal{L}(\vct{\theta}):=\frac{1}{2}\sum_{i=1}^n \left(\phi(\langle \vct{x}_i,\vct{\theta}\rangle)-\vct{y}_i\right)^2.
\end{align}
Define the data matrix $\mtx{X}\in\R^{n\times p}$ with rows given by $\vct{x}_i$ for $i=1,2,\ldots,n$. We thus recognize the above fitting problem as a special instance of \eqref{NLScompact} with $f(\vct{\theta})=\phi\left(\mtx{X}\vct{\theta}\right)$. Here, $\phi$ when applied to a vector means applying the nonlinearity entry by entry. We wish to understand the behavior of GD in the over-parameterized regime where $n\le p$. This is the subject of the next two theorems.
\begin{theorem} [Overparameterized GLM]\label{simpGLM} Consider a data set of input/label pairs $\vct{x}_i\in\R^p$ and $y_i$ for $i=1,2,\ldots,n$ aggregated as rows/entries of a matrix $\mtx{X}\in\R^{n\times p}$ and a vector $\vct{y}\in\R^n$ with $n\le p$. Also consider a Generalized Linear Model (GLM) of the form $\vct{x}\mapsto \phi\left(\langle\vct{x},\vct{\theta}\rangle\right)$ with $\phi:\R\rightarrow\R$ a strictly increasing nonlinearity with continuous derivatives (i.e.~obeying $0<\gamma\le \phi'(z)\le \Gamma$ for all $z$). Starting from an initial parameter $\vct{\theta}_0$ we run gradient descent updates of the form $\vct{\theta}_{\tau+1}=\vct{\theta}_\tau-\eta\nabla \mathcal{L}(\vct{\theta}_\tau)$ on the loss \eqref{GLMopt} with $\eta\le \frac{1}{\opnorm{\mtx{X}}^2\Gamma^2}$. Furthermore, let $\vct{\theta}^*$ denote the closest global optima to $\vct{\theta}_0$. Then, all GD iterates obey
%Then, setting learning rate $\eta=1$ and starting at $\bteta_0=0$ gradient descent converges to a unique global minima $\bteta_\star$ which satisfies $\tn{\bteta_\star}\leq 8\sqrt{n}/\gamma\alpha$ and all GD iterations obey
\begin{align}%{}\min(1/2,\frac{L\sqrt{n}/4)
\tn{\bteta_\tau-\bteta^\star}\leq \left(1-\eta\gamma^2\lambda_{\min}\left(\mtx{X}\mtx{X}^T\right)\right)^\tau\tn{\vct{\theta}_0-\bteta^\star}.
\label{GLM}
%&\sqrt{2\Lc(\bteta_i)}\leq \left(1-\frac{\gamma^2\alpha^2}{4}\min(1,\frac{\gamma^2\alpha^2}{L\sqrt{n}})\right)^i\tn{\y}.\\
%&\frac{\gamma \alpha}{4}\tn{\bteta_i}+\sqrt{2\Lc(\bteta_i)}\leq 2\sqrt{n}.
\end{align}
\end{theorem}
The above theorem demonstrates that when fitting GLMs in the over-parameterized regime, gradient descent converges at a linear to a globally optimal model. Furthermore, this convergence is to the closest global optima to the initialization parameter. Also, we can deduce from \eqref{GLM} that the total gradient path length when using a step size on the order of $\frac{1}{\opnorm{\mtx{X}}^2\Gamma^2}$ is bounded by
\begin{align}
\label{GLMpath}
\sum_{\tau=0}^\infty \tn{\bteta_{\tau+1}-\bteta_\tau}\le \frac{\Gamma^2}{\gamma^2}\frac{\lambda_{\max}\left(\mtx{X}\mtx{X}^T\right)}{\lambda_{\min}\left(\mtx{X}\mtx{X}^T\right)}\tn{\vct{\theta}_0-\bteta^\star},
\end{align}
%\MS{Please check above and its proof}
so that the total path length is a constant multiple of the distance between initialization and the closest global optima. Furthermore, applying Theorem \ref{GDthm} with a smaller learning rate, the right hand side can be improved to $\frac{\Gamma}{\gamma}\frac{\|\X\|}{\smn{\X}}\tn{\vct{\theta}_0-\bteta^\star}$. Thus, gradient descent takes a near direct route.

\subsection{Low-rank regression}\label{sec over rank}
A variety of modern learning problems spanning recommender engines to controls involve fitting low-rank models to data. In this problem given a data set of size $n$ consisting of input/features $\mtx{X}_i\in\R^{d\times d}$ and labels $\vct{y}_i\in\R$ for $i=1,2,\ldots,n$, we aim to fit nonlinear models of the form
\begin{align*}
\mtx{X}\mapsto f(\mtx{X};\mtx{\Theta})=\langle \mtx{X},\mtx{\Theta}\mtx{\Theta}^T\rangle=\text{trace}\left(\mtx{\Theta}^T\mtx{X}\mtx{\Theta}\right),
\end{align*}
with $\mtx{\Theta}\in\R^{d\times r}$ the parameter of the model. Fitting such models require optimizing losses of the form
\begin{align}
\label{matrixNLS}
\underset{\mtx{\Theta}\in\R^{d\times r}}{\min}\Lc(\bTeta)=\frac{1}{2}\sum_{i=1}^n \left(\vct{y}_i-\li\X_i,\bTeta\bTeta^T\ri\right)^2.
\end{align}
This approach, originally proposed by Burer and Monteiro \cite{burer2003nonlinear}, shifts the search space from a large low-rank positive semidefinite matrix $\bTeta\bTeta^T$ to its factor $\bTeta$. In this section we study the behavior of GD and SGD on this problem in the over-parameterized regime where $n<dr$.
\begin{theorem} \label{low rank reg2}Consider the problem of fitting a low-rank model of the form $\mtx{X}\mapsto f(\mtx{X};\mtx{\Theta})=\text{trace}\left(\mtx{\Theta}^T\mtx{X}\mtx{\Theta}\right)$ with $\mtx{\Theta}\in\R^{d\times r}$ with $r\le d$ to a data set $(y_i,\X_i)\in \R\times \R^{d\times d}$ for $i=1,2,\ldots,n$ via the loss \eqref{matrixNLS}. Assume the input features $\X_i$ are random and distributed i.i.d.~with entries i.i.d.~$\mathcal{N}(0,1)$. Furthermore, assume the labels $\vct{y}_i$ are arbitrary and denote the vector of all labels by $\vct{y}\in\R^n$. Set the initial parameter $\mtx{\Theta}_0\in\R^{d\times r}$ to a matrix with singular values lying in the interval $\big[\frac{\sqrt{\twonorm{\y}}}{\sqrt[4]{rn}},2\frac{\sqrt{\twonorm{\y}}}{\sqrt[4]{rn}}\big]$ Furthermore, let $c,c_1,c_2>0$ be numerical constants and assume\vs\vs
\begin{align*}
n\le cdr.
\end{align*}% \frac{c_1}{dr\tn{\y}\sqrt{r/n}(1+\sqrt{n/d})
We run gradient descent iterations of the form $\mtx{\Theta}_{\tau+1}=\mtx{\Theta}_{\tau}-\eta \nabla\mathcal{L}(\mtx{\Theta}_{\tau})$ starting from $\mtx{\Theta}_0$ with $\eta =\frac{c_1\sqrt{n}}{r^2d\tn{\y}}$. Then, with probability at least $1-4e^{-\frac{n}{2}}$ all GD iterates obey
\[
\sum_{i=1}^n\left(\vct{y}_i-\li\X_i,\bTeta_{\tau}\bTeta_{\tau}^T\ri\right)^2\leq 100\left(1- \frac{c_2}{r^{3/2}}\right)^\tau \twonorm{\vct{y}}^2,
\]
%\tn{\text{trace}\left(\bTeta_{\tau}^T\mtx{X}_i\bTeta_{\tau}\right)-\y_i}^2
%and the total gradient path length is bounded by
%\begin{align*}
%\sum_{\tau=0}^{\infty}\fronorm{\mtx{\Theta}_{\tau+1}-\mtx{\Theta}_{\tau}}\le \frac{1}{2400}\sqrt{\twonorm{\y}\sqrt{\frac{r}{n}}}.%????\fronorm{\mtx{\Theta}^*-\mtx{\Theta}_0}
%\end{align*}
\end{theorem}
%\MS{Need to make sure changes are correct with new init. Seems like convergence rate might change.}
%Also let $\mtx{\Theta}^*$ denote the global optima of \eqref{matrixNLS} which is closest to the initialization $\mtx{\Theta}_0$ in Frobenius norm. Then, with probability at least $1-4e^{-\frac{n}{2}}$ all GD iterates obey
%\begin{align*}
%\fronorm{\mtx{\Theta}_\tau-\mtx{\Theta}_0}\le ????\fronorm{\mtx{\Theta}^*-\mtx{\Theta}_0}
%\end{align*}
%Suppose we are given $n$ samples $(y_i,\X_i)\in \R\times \R^{d\times d}$ where $\X_i\distas\Nn(0,\Iden_{d^2})$. Without losing generality, assume labels satisfy $\tn{\y}^2=n/r$. Let $\bTeta_0$ be a matrix with singular values lying in $[1/2\sqrt{r},2/\sqrt{r}]$. Consider the function $f(\X,\bTeta)=\text{trace}(\bTeta^T\X\bTeta)$. Let $c,C>0$ be constants. Suppose $n\leq {dr}/C$ and $r\leq d/10$. Then, with probability $1-4\exp(-n/2)$, starting from $\bTeta_0$, gradient descent iterations with learning rate $\eta={\frac{c}{dr(1+\sqrt{n/d})}}$ obeys
This theorem shows that with modest over-parametrization $ dr\gtrsim n$, GD linearly converges to a globally optimal model and achieves zero loss. Note that degrees of freedom of $d\times r$ matrices is $dr$ hence as soon as $n>dr$, gradient descent can no longer perfectly fit arbitrary labels highlighting a phase transition from zero loss to non-zero as sample size increases. Furthermore, our result holds despite the nonconvex nature of the Burer-Monteiro approach.
%Furthermore, the trajectory GD follows is near optimal with total path length proportional to the length of the direction path from initialization to the closest global optima.
%Furthermore, the point to which GD converges to is within a \textcolor{red}{$???$} factor of the closest global optima to initialization. 

\subsection{Training shallow neural networks}
In this section we specialize our general approach in the context of training simple shallow neural networks. We shall focus on neural networks with only one hidden layer with $d$ inputs, $k$ hidden neurons and a single output. The overall input-output relationship of the neural network in this case is a function $f(\cdot;\vct{\theta}):\R^d\rightarrow\R$ that maps the input vector $\vct{x}\in\R^d$ into a scalar output via the following equation
\begin{align*}
\vct{x}\mapsto f(\vct{x};\mtx{W})=\sum_{\ell=1}^k\vct{v}_\ell \phi\left(\langle \vct{w}_\ell,\vct{x}\rangle\right).
\end{align*}
In the above the vectors $\vct{w}_\ell\in\R^d$ contains the weights of the edges connecting the input to the $\ell$th hidden node and $\vct{v}_\ell\in\R$ is the weight of the edge connecting the $\ell$th hidden node to the output. Finally, $\phi:\R\rightarrow\R$ denotes the activation function applied to each hidden node. For more compact notation we gather the weights $\vct{w}_\ell/\vct{v}_\ell$ into larger matrices $\mtx{W}\in\R^{k\times d}$ and $\vct{v}\in\R^k$ of the form
\begin{align*}
\mtx{W}=\begin{bmatrix}\vct{w}_1^T\\\vct{w}_2^T\\\vdots\\\vct{w}_k^T\end{bmatrix}\quad\text{and}\quad\vct{v}=\begin{bmatrix}{v}_1\\{v}_2\\\vdots\\ {v}_k\end{bmatrix}.
\end{align*}
We can now rewrite our input-output model in the more succinct form
\begin{align}
\label{model2}
\vct{x}\mapsto f(\vct{x};\mtx{W}):=\vct{v}^T\phi(\mtx{W}\vct{x}).
\end{align}
Here, we have used the convention that when $\phi$ is applied to a vector it corresponds to applying $\phi$ to each entry of that vector. When training a neural network, one typically has access to a data set consisting of $n$ feature/label pairs $(\vct{x}_i,y_i)$ with $\vct{x}_i\in\R^d$ representing the feature and $y_i$ the associated label. We wish to infer the best weights $\vct{v},\mtx{W}$ such that the mapping $f$ best fits the training data. In this paper we assume $\vct{v}\in\R^k$ is fixed and we train for the input-to-hidden weights $\mtx{W}$. Without loss of generality we assume $\vct{v}\in\R^k$ has unit Euclidean norm i.e.~$\twonorm{\vct{v}}=1$. The training optimization problem then takes the form
\begin{align}
\label{neuralopt}
\underset{\mtx{W}\in\R^{k\times d}}{\min}\text{ }\mathcal{L}(\mtx{W}):=\frac{1}{2}\sum_{i=1}^n \left(\vct{v}^T\phi\left(\mtx{W}\vct{x}_i\right)-y_i\right)^2.
\end{align}
The theorem below provides geometric global convergence guarantees for one-hidden layer neural networks in a simple over-parametrized regime.
\begin{theorem} [Overparameterized Neural Nets]\label{thm over glm} Consider a data set of input/label pairs $\vct{x}_i\in\R^d$ and $y_i\in\R$ for $i=1,2,\ldots,n$ aggregated as rows/entries of a matrix $\mtx{X}\in\R^{n\times d}$ and a vector $\vct{y}\in\R^n$ with $n\le d$. Also consider a one-hidden layer neural network with $k$ hidden units and one output of the form $\vct{x}\mapsto \vct{v}^T\phi\left(\mtx{W}\vct{x}\right)$ with $\mtx{W}\in\R^{k\times d}$ and $\vct{v}\in\R^k$ the input-to-hidden and hidden-to-output weights. We assume the activation $\phi$ is strictly increasing with bounded derivatives i.e.~$0<\gamma\le \phi'(z)\le \Gamma$ and $\phi''(z)\le M$ for all $z$. We assume $\vct{v}$ is fixed with unit Euclidean norm ($\twonorm{\vct{v}}=1$) and train only over $\mtx{W}$. Starting from an initial weight matrix $\mtx{W}_0$ we run gradient descent updates of the form $\mtx{W}_{\tau+1}=\mtx{W}_\tau-\eta\nabla \mathcal{L}(\mtx{W}_\tau)$ on the loss \eqref{neuralopt} with $\eta\leq\frac{1}{2\Gamma^2\opnorm{\X}^2}\min\left(1,\frac{\gamma^2}{\Gamma M}\frac{\smn{\X}^2}{\tti{\X}\opnorm{\X}}\frac{1}{\tn{f(\mtx{W}_0)-\y}}\right)$.\footnote{Here, $\|\mtx{X}\|_{2,\inf}$ denotes the maximum Euclidean norm of the rows of $\mtx{X}$.} Then, all GD iterates obey
\begin{align}
&\tn{f(\mtx{W}_\tau)-\y}\leq \left(1-\eta\gamma^2 \sigma_{\min}^2\left(\X\right)\right)^\tau\tn{f(\mtx{W}_0)-\y},\\
&\frac{\gamma \smn{\X}}{4}\fronorm{\mtx{W}_\tau}+\tn{f(\mtx{W}_\tau)-\y}\leq \tn{f(\mtx{W}_0)-\y}.
\end{align}
\end{theorem}
The theorem above demonstrates that the nice properties discussed in this paper also holds for one-hidden-layer networks in the over-parameterized regime where $n\leq d$. This result establishes convergence from arbitrary initialization and the result is independent of number of hidden nodes $k$. The result holds for strictly increasing activations where $\phi'$ is bounded away from zero. While this might seem restrictive, we can obtain such a function by adding a small linear component to any non-decreasing function i.e.~$\tilde{\phi}(x)=(1-\gamma)\phi(x)+\gamma x$. We would however like to emphasize that neural networks seem to work with much more modest amounts of over-parameterization e.g.~for one hidden networks like the above $kd\gtrsim n$ seems to be sufficient. As such there is a huge gap between our result and practical use (as with many other recent results which also require heavy but not directly comparable over-parametrization of the form $k\gtrsim$poly$(n)$). That said, we believe it may be possible to utilize more sophisticated techniques from random matrix theory and stochastic processes to prove variations of Theorems \ref{GDthm} and \ref{SGDthm} that continue to apply for such modestly over-parametrized scenarios.
%The next theorem provides a stronger result for generalized linear models where $k=1$.

\section{Beyond nonlinear least-squares}\label{sec polyak}
In this section we explore generalizations of our results beyond nonlinear least-squares problems. In particular we focus on optimizing a general loss $\mathcal{L}(\vct{\theta})$ over $\vct{\theta}\in\R^p$. For exposition purposes throughout this section we assume that $\Lc$ is differentiable and the global minimum is zero, i.e.~$\underset{\vct{\theta}}{\min}$ $\Lc(\bteta)=0$\footnote{Note that this is without loss of generality as for any loss we can apply the results to the shifted loss $\tilde{\mathcal{L}}(\vct{\theta})=\mathcal{L}(\vct{\theta})-\underset{\vct{\theta}}{\min}\text{ }\Lc(\bteta)$.}. This generalization will be based on a local variant of Polyak-Lojasiewicz (PL) inequality. We begin by discussing this local PL condition formally.
%\begin{align}
%\min_{\bteta\in\R^p}\Lc(\bteta).\label{LP opt}
%\end{align}
%Assume \eqref{LP opt} has a non-empty solution set and optimal value of $\Lc$ is $\min_{\bteta}\Lc(\bteta)=0$.
%$\Lc^\star$ be the optimal value of $\Lc$
%\SO{Do we need: We also assume that the optimization problem has a non-empty solution set?}
\begin{definition}[Local PL condition] \label{local pl} We say that the Local PL inequality holds over a set $\Dc\subseteq\R^p$ with $\mu>0$ if for all $\bteta\in\Dc$ we have% $\Lc$ satisfies
\[
\tn{\grad{\bteta}}^2\geq 2\mu\Lc(\bteta).
\]
\end{definition}
Our first result shows that when the PL inequality holds around a minimally small neighborhood of the initialization, the intriguing properties of gradient descent discussed in Theorem \ref{GDthm} and Corollary \ref{mycor} continue to hold beyond nonlinear least-squares problems.
\begin{theorem} \label{pl thm} Let $\mathcal{L}:\R^p\rightarrow\R$ be a loss function. Let $\vct{\theta}_0\in\R^p$ be an initialization parameter and define the set $\mathcal{D}$ to a local neighborhood around this point as follows
\begin{align*}
\mathcal{D}=\mathcal{B}\left(\vct{\theta}_0,R\right)\quad\text{with}\quad R=\sqrt{\frac{8\mathcal{L}(\vct{\theta}_0)}{\mu}}\quad\text{and}\quad \mu>0.
\end{align*}
Assume the loss $\mathcal{L}$ obeys the local PL condition per Definition \ref{local pl} and is $L$-smooth over $\mathcal{D}$ ($\twonorm{\nabla \mathcal{L}(\vct{\theta}_2)-\nabla \mathcal{L}(\vct{\theta}_1)}\le L\twonorm{\vct{\theta}_2-\vct{\theta}_1}$ for  all $\vct{\theta}_1,\vct{\theta}_2\in\mathcal{D}$). Then, starting from $\vct{\theta}_0$ running gradient descent updates of the form 
\begin{align*}
\vct{\theta}_{\tau+1}=\vct{\theta}_{\tau}-\eta \nabla \mathcal{L}(\vct{\theta}_\tau),
\end{align*}
with $\eta\le 1/L$, all iterates $\vct{\theta}_\tau$ obey the following inequalities
\begin{align}
\Lc(\bteta_{\tau})\leq& (1-\eta\mu)^\tau\Lc(\bteta_{0}),\label{pl conv2}\\
\sqrt{\frac{\mu}{8}}\tn{\bteta_\tau-\bteta_0}+\sqrt{\Lc(\bteta_\tau)}\leq& \sqrt{\Lc(\bteta_0)}.\label{pl conv}%\Lc(\bteta_0)-
\end{align}
Furthermore, the total path length of gradient descent is bounded via
\begin{align}
\sum_{\tau=0}^\infty \tn{\bteta_{\tau+1}-\bteta_\tau}\leq \sqrt{\frac{8\Lc(\bteta_0)}{\mu}}.\label{pl short}
\end{align}
\end{theorem}
Similar to Corollary \ref{mycor} a trivial consequence of the above theorem is the following corollary.
\begin{corollary}\label{mycorPL} Consider the setting and assumptions of Theorem \ref{pl thm} above. Let $\vct{\theta}^*$ denote the global optima of the loss $\mathcal{L}(\vct{\theta})$ with smallest Euclidean distance to the initial parameter $\vct{\theta}_0$. Then, the gradient descent iterates $\vct{\theta}_\tau$ obey
\begin{align}
\twonorm{\vct{\theta}_\tau-\vct{\theta}_0}\le 2\frac{L}{\mu}\twonorm{\vct{\theta}^*-\vct{\theta}_0},\label{approxclosest2}\\
\sum_{\tau=0}^\infty\twonorm{\vct{\theta}_{\tau+1}-\vct{\theta}_\tau}\le 2\frac{L}{\mu}\twonorm{\vct{\theta}^*-\vct{\theta}_0}.\label{approxshortest2}
\end{align}
\end{corollary}
Similar to their nonlinear least-squares counter parts the theorem and corollary above show that if the loss function obeys the local PL condition and is smooth in a ball of radius $R$ around the initial point then gradient descent enjoys three intriguing properties: (i) the iterates converge at a linear rate to a global optima, (ii) Gradient descent iterates remain close to the initialization and never leave a neighborhood of radius $2\frac{L}{\mu}\twonorm{\vct{\theta}^*-\vct{\theta}_0}$, and (iii) gradient descent iterates follow a near direct route to the global optima with the length of the path taken by GD iterates within a factor of the distance between the closest global optima and the initialization.

We end this section by discussing a simple lower bound which demonstrates that the required radius over which the Local PL result must hold per Theorem \ref{pl thm} is optimal up to a factor of two. 
\begin{theorem}\label{pl low bound} Let $\mathcal{L}:\R^p\rightarrow \R$ be an $L$-smooth loss function over a ball of radius $R$ centered around a point $\vct{\theta}_0\in\R^p$ ($\mathcal{B}(\vct{\theta}_0,R)$). Then there is no global minima over $\mathcal{B}(\vct{\theta}_0,R)$ when $R<\sqrt{2\mathcal{L}(\vct{\theta}_0)/L}$. Furthermore, for any $\mu$ and $L$ obeying $L\ge\mu\ge 0$, there exists a loss $\mathcal{L}$ such that there is no global minima over the set $\mathcal{B}(\vct{\theta}_0,R)$ as long as $R<\sqrt{2\Lc(\bteta_0)/\mu}$.
\end{theorem}
The result above shows that any global optima is at least a distance $\tn{\bteta-\bteta_0}\geq \sqrt{2\mathcal{L}(\vct{\theta}_0)/L}$ away from the initialization so that the minimum ball around the initial point needs to have radius at least $R\ge \sqrt{2\mathcal{L}(\vct{\theta}_0)/L}$ for convergence to a global optima to occur. Comparing this lower-bound with that of Theorem \ref{pl thm} suggests that the size of the local neighborhood is optimal up to a factor $2$. Collectively our theorems demonstrate that the path taken by gradient descent is by no means arbitrary. Indeed, under local PL and smoothness assumptions similar to Figure \ref{GDpath_fig}, gradient descent iterates exhibit a sharp tradeoff between distance to the initial point ($\twonorm{\vct{\theta}-\bteta_0}$) and square root of loss value ($\sqrt{\mathcal{L}(\vct{\theta}_0)}$). 

\section{Numerical Experiments}\label{sec numeric}
To verify our theoretical claims, we conducted experiments on MNIST classification and low-rank matrix regression. To illustrate the tradeoffs between the loss function and the distance to the initial point, we define normalized misfit and normalized distance as follows.
\begin{align}
\text{Normalized misfit}=\frac{\tn{\y-f(\bteta)}}{\tn{\y-f(\bteta_0)}}\quad\quad,\quad\quad\text{Normalized distance}=\frac{\tn{\bteta-\bteta_0}}{\tn{\bteta_0}}.
\end{align}
%MNIST experiments are described as follows.
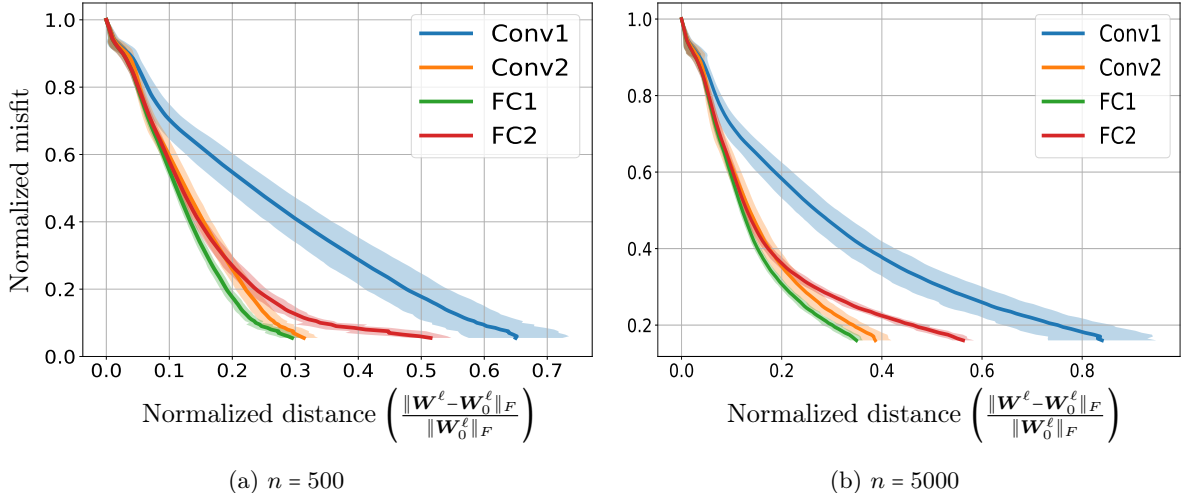
\begin{figure}[t!]
\begin{centering}
\begin{subfigure}[t]{3in}
\begin{tikzpicture}
\node at (0,0) {\includegraphics[height=0.7\linewidth,width=1\linewidth]{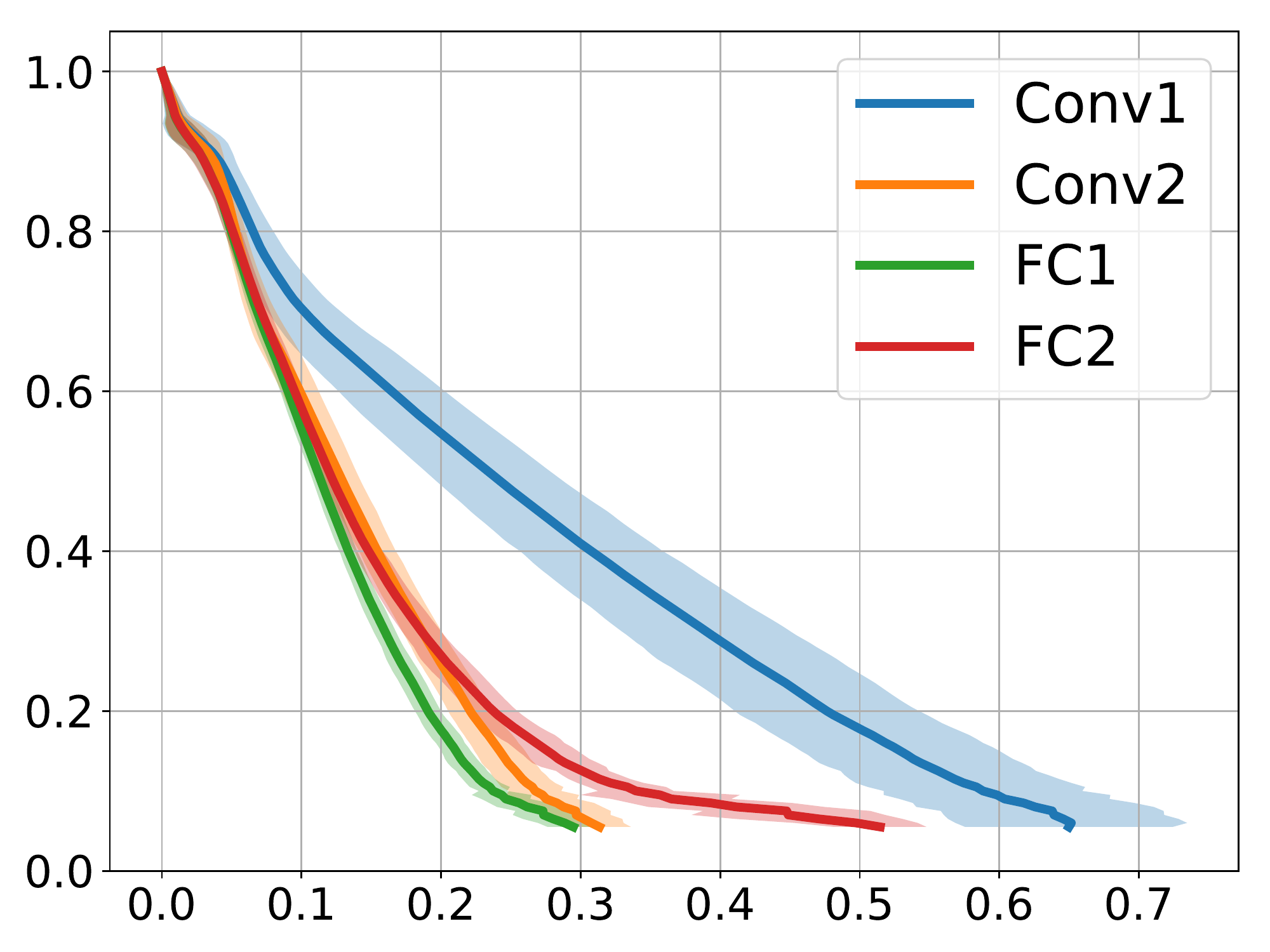}};
\node at (0.25,-3) {Normalized distance $\left(\frac{\tf{\mtx{W}^{\ell}-\mtx{W}^{\ell}_0}}{\tf{\mtx{W}_0^{\ell}}}\right)$};
\node[rotate=90] at (-4,0) {Normalized misfit};
\end{tikzpicture}
%\node at (1,1) {source};
%\includegraphics[height=0.7\linewidth,width=1\linewidth]{figs/neural_net2_500.pdf}\vspace{-5pt}\
\subcaption{$n=500$}
\label{fig1a}
\end{subfigure}
\end{centering}~
\begin{centering}
\begin{subfigure}[t]{3in}
\begin{tikzpicture}
\node at (0,0) {\includegraphics[height=0.7\linewidth,width=1\linewidth]{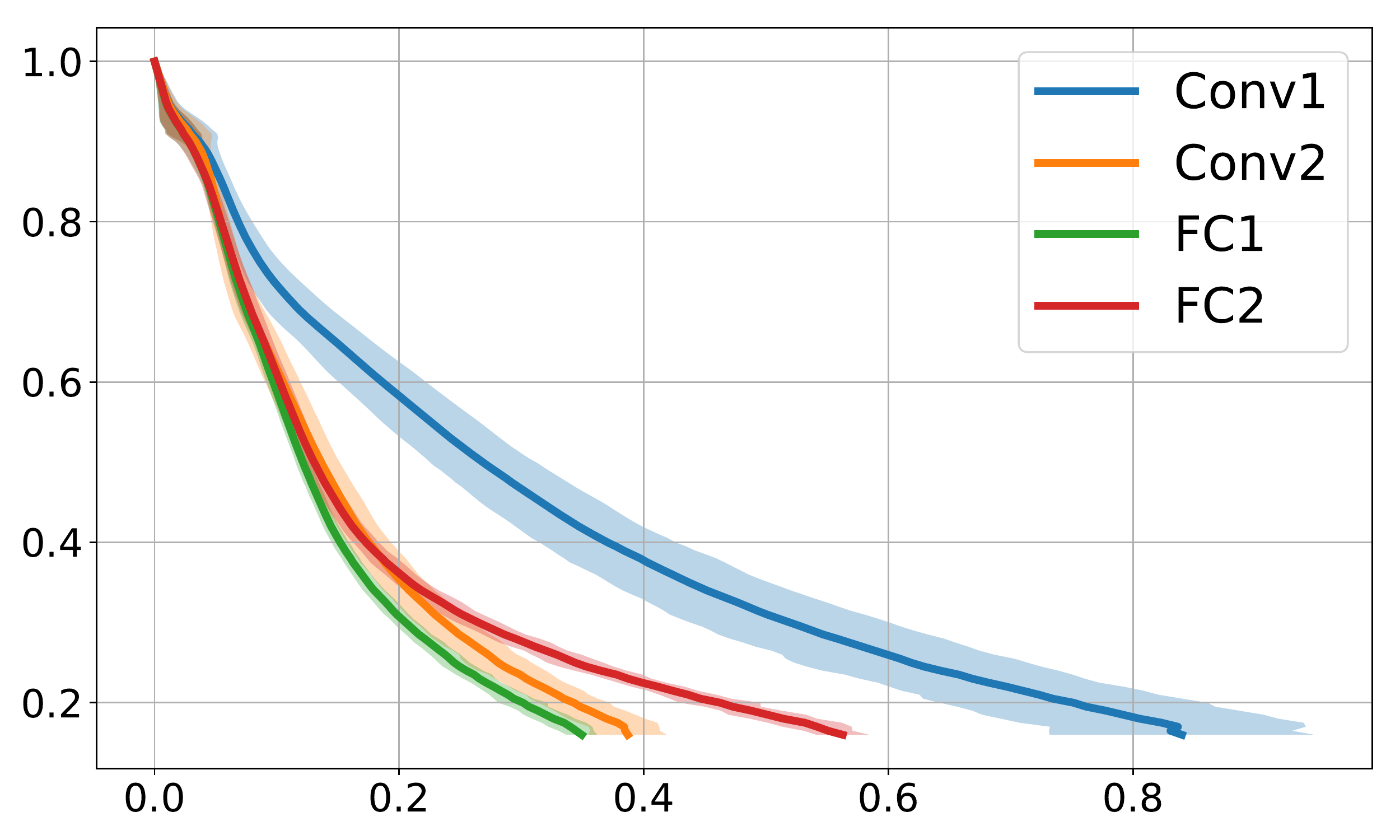}};
\node at (0.25,-3) {Normalized distance $\left(\frac{\tf{\mtx{W}^{\ell}-\mtx{W}^{\ell}_0}}{\tf{\mtx{W}_0^{\ell}}}\right)$};
\end{tikzpicture}
\subcaption{$n=5000$}\label{fig1b}
\end{subfigure}
\end{centering}\vspace{-10pt}\caption{The normalized misfit-distance trajectory for MNIST training for different layers of the network and different sample sizes. The layers from input to output are Conv1, Conv2, FC1, and FC2. \qqq{Each curve represents the average normalized distance (for each layer of the network) corresponding to a fixed normalized misfit value over $20$ independent realizations. The two standard deviation around the average distance is highlighted via the shaded region.}}\label{fig1}
\end{figure}
\subsection{MNIST Experiments}\label{sec mnist}
We consider MNIST digit classification task and use a standard LeNet model \cite{lecun1998gradient} from Tensorflow \cite{abadi2016tensorflow}\footnote{https://github.com/tensorflow/models/blob/master/research/slim/nets/lenet.py}. This model has two convolutional layers followed by two fully-connected layers. Instead of cross-entropy loss, we use least-squares loss, without softmax layer, which falls within our nonlinear least-squares framework. We conducted two set of experiments with $n=500$ and $n=5000$. Both experiments use Adam with learning rate $0.001$ and batch size $100$ for $1000$ iterations. At each iteration, we record the normalized misfit and distance to obtain a misfit-distance trajectory similar to Figure \ref{GDpath}. We repeat the training $20$ times (with independent initialization and dataset selection) to obtain the typical behavior.

%\footnote{To visualize the behavior of distinct layers, we consider the normalized distance }
Since layers have distinct goals (feature extraction vs classification), we kept track of the behavior of individual layers. Specifically, denote the weights of the $\ell$th layer of the neural network by $\mtx{W}^{\ell}$, we consider the per-layer normalized distances $\frac{\tf{\mtx{W}^{\ell}-\mtx{W}^{\ell}_0}}{\tf{\mtx{W}_0^{\ell}}}$ where layer $\ell$ is either convolutional (Conv1, Conv2) or fully-connected (FC1, FC2). In Figure \ref{fig1}, we depict the normalized misfit-distance tradeoff for different layers and sample sizes. Figure \ref{fig1a} illustrates the heavily overparameterized regime which has fewer samples. During the initial phase of the training (i.e.~$\text{misfit}\leq 0.2$) all layers follow a straight loss-distance line which is consistent with our theory (e.g.~Figure \ref{GDpath}). Towards the end of the training, the lines slightly level off which is most visible for the output layer FC2. This is likely due to the degradation of the Jacobian condition number as the model overfits to the data. Figure \ref{fig2a} plots the training and test errors together with normalized misfit to illustrate this. While misfit is around $0.05$ at iteration $1000$, the in-sample (classification) error hits $0$ very quickly at iteration $200$.

In Figure \ref{fig1b} and \ref{fig2b} we increase the sample size to $n=5000$. Similar to the first case, during the initial phase ($\text{misfit}\leq 0.4$) the loss-distance curve is a straight line and levels off later on. Compared to $n=500$, leveling off occurs earlier and is more visible. For instance, at $\text{misfit}=0.2$, output layer FC2 has distance of $0.5$ for $n=5000$ and $0.25$ for $n=500$. This is consistent with Theorem \ref{GDthm} which predicts (i) more samples imply a Jacobian with worse condition number and (ii) the global minimizer lies further away from the initialization and it is less-likely that the Jacobian will be well-behaved over this larger neighborhood.
% radius of convergence is proportional to square-root of the sample sizehowever, we need a likely due to the fact that more samples imply a worse quality Jacobian.

%The confidence intervals highlight the two-standard deviation around the average normalized distance for fixed normalized misfit value.
\begin{figure}[t!]
\begin{centering}
\begin{subfigure}[t]{1.9in}
\begin{tikzpicture}
\node at (0,0) {\includegraphics[height=0.8\linewidth,width=1\linewidth]{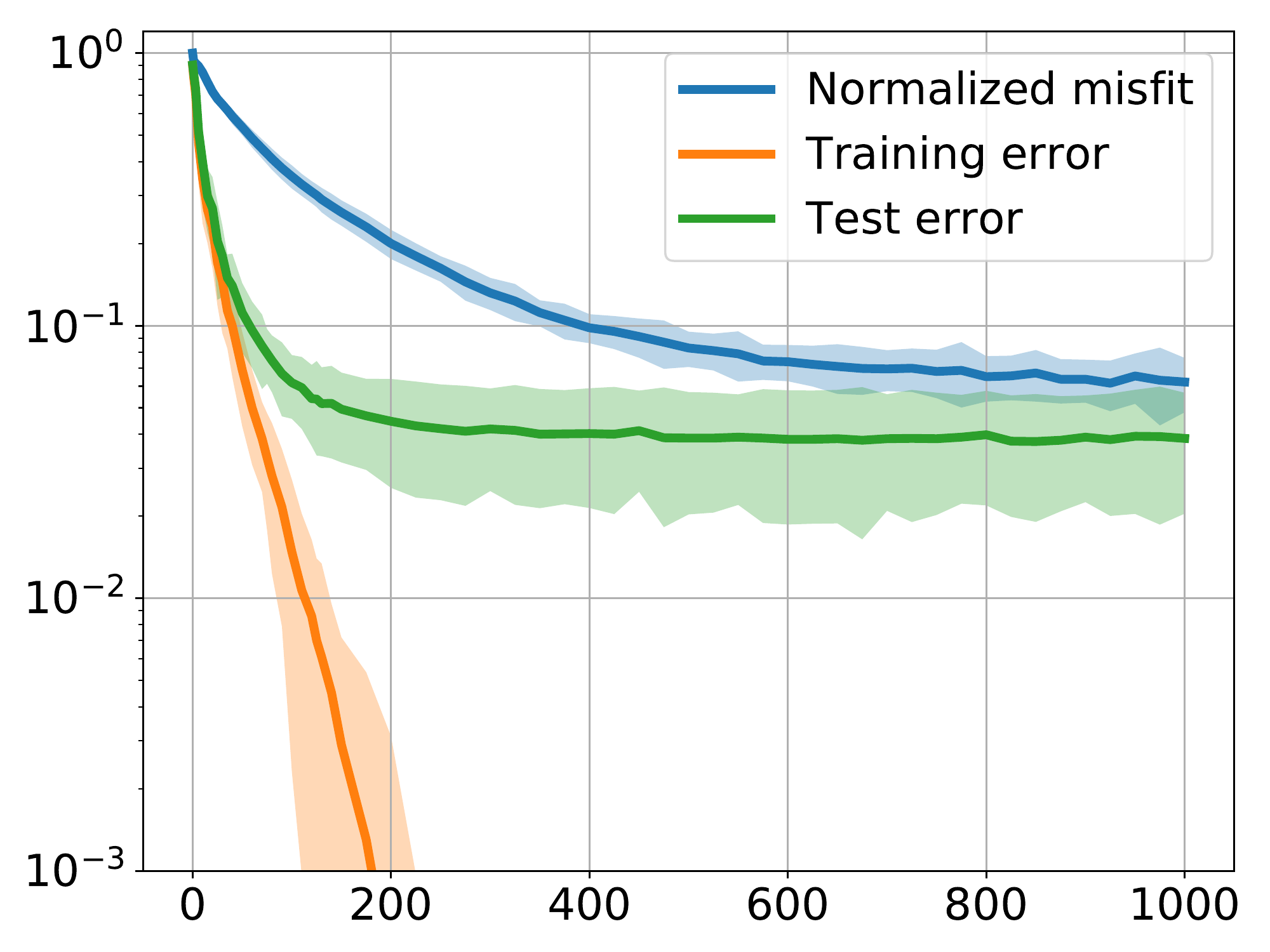}};
\node at (0.25,-2) {Iterations};
\node[rotate=90] at (-2.5,0) {Error};
\end{tikzpicture}
\vspace{-15pt}\subcaption{MNIST, $n=500$}\label{fig2a}
\end{subfigure}
\end{centering}~\hspace{-5pt}
\begin{centering}
\begin{subfigure}[t]{1.9in}
\begin{tikzpicture}
\node at (0,0) {\includegraphics[height=0.8\linewidth,width=1\linewidth]{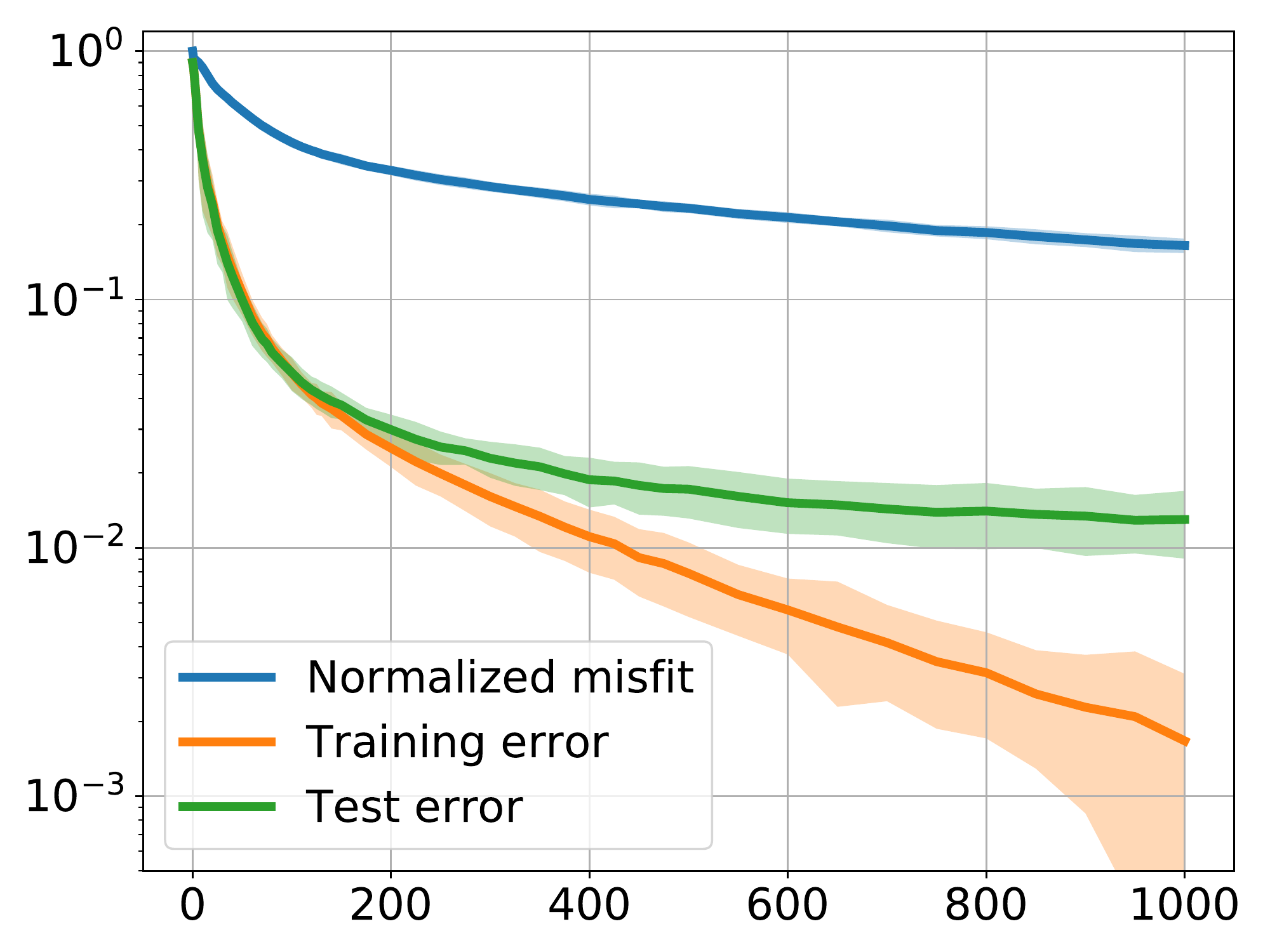}};
\node at (0.25,-2) {Iterations};
%\node[rotate=90] at (-3.4,0) {Normalized misfit};
\end{tikzpicture}
\vspace{-15pt}\subcaption{MNIST, $n=5000$}\label{fig2b}
\end{subfigure}
\end{centering}~
\begin{centering}
\begin{subfigure}[t]{2.5in}
\begin{tikzpicture}
\node at (0,0) {\includegraphics[height=0.77\linewidth,width=1\linewidth]{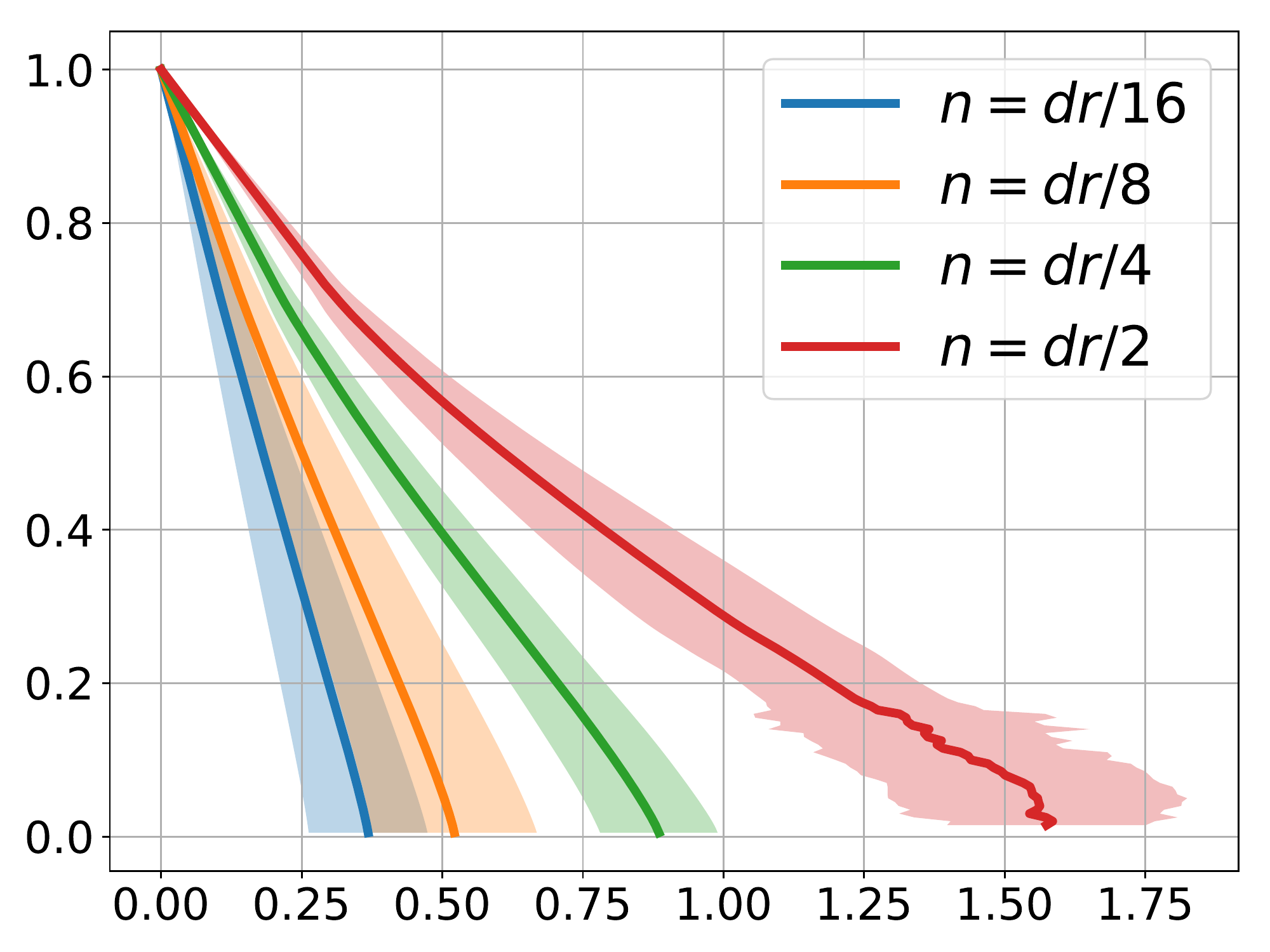}};
\node at (0.25,-2.7) {Normalized distance $\left(\frac{\tf{\bTeta-\bTeta_0}}{\tf{\bTeta_0}}\right)$};
\node[rotate=90] at (-3.3,0) {Normalized misfit};
\end{tikzpicture}
\vspace{-15pt}\subcaption{Low-rank regression}\label{fig2c}
\end{subfigure}
\end{centering}\vspace{-5pt}\caption{Figures \ref{fig2a} and \ref{fig2b} represent the test, training errors and normalized misfit corresponding to Figure \ref{fig1}. The $x$-axis is the number of iterations. Figure \ref{fig2c} highlights the loss-distance trajectory for low-rank matrix regression with $d=100$ and $r=4$.}\label{fig2}% of $20$ independent realizations.
\end{figure}
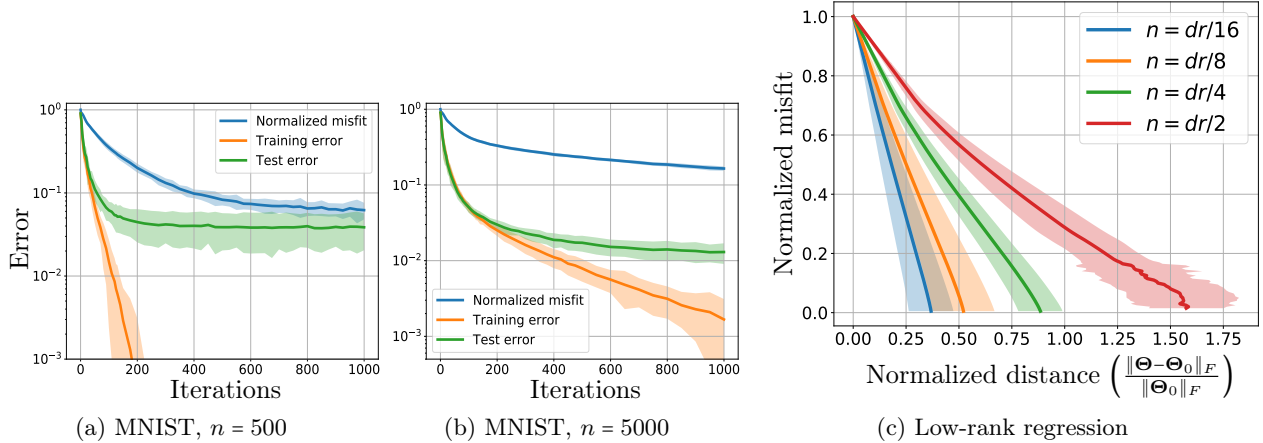

%\begin{figure}[t!]
%\begin{centering}
%\begin{subfigure}[t]{3in}
%\begin{subfigure}[t]{2in}
%\includegraphics[height=0.6\linewidth,width=1\linewidth]{figs/errors_big_5000.pdf}\vspace{-5pt}\subcaption{MNIST, $n=5000$}\label{fig2a}
%\end{subfigure}
%%\includegraphics[height=0.7\linewidth,width=1\linewidth]{figs/errors_big_5000.pdf}\vspace{-5pt}\subcaption{MNIST, $n=5000$}\label{fig2a}
%\begin{subfigure}[t]{2in}
%\includegraphics[height=0.6\linewidth,width=1\linewidth]{figs/errors_big_5000.pdf}\vspace{-5pt}\subcaption{MNIST, $n=5000$}\label{fig2a}
%\end{subfigure}
%\end{subfigure}
%\end{centering}~
%\begin{centering}
%\begin{subfigure}[t]{3in}
%\includegraphics[height=0.77\linewidth,width=1\linewidth]{figs/low_rank_0p02.pdf}\vspace{-5pt}\subcaption{Low-rank regression}\label{fig2b}
%\end{subfigure}
%\end{centering}\vspace{-10pt}\caption{}\label{fig2}% of $20$ independent realizations.
%\end{figure} ince this is a pure regression task and data is synthetically generated, we expect cleaner figures compared to Section \ref{sec mnist}. 
\subsection{Low-rank regression}
We consider a synthetic low-rank regression setup to test the predictions of Theorem \ref{low rank reg2}. We generate input matrices with i.i.d.~standard normal entries and labels with i.i.d.~Rademacher entries. We set $r=4$ and $d=100$ and initialize $\bTeta_0$ according to Theorem \ref{low rank reg2}. We vary the sample size to be $n\in \{25,50,100,200\}=\{dr/16,dr/8,dr/4,dr/2\}$ and run gradient descent for $200$ iterations with a constant learning rate per Theorem \ref{low rank reg2}. We observe a linear tradeoff in terms of misfit-distance to initialization with a narrow confidence interval consistent with our theoretical predictions in Figure \ref{GDpath}. In the large sample size ($n=dr/2$), the problem is less over-parameterized and the confidence intervals become notably wider especially when the misfit is close to zero (i.e.~by the time we reach a global minima). As predicted by our main theorem, the distance to initialization $\bTeta_0$ increases gracefully as the number of labels $n$ increases. 

%distance of the minimizer \MS{what is distance of minimizer} shows more variability as highlighted by the larger confidence interval when the misfit is close to zero.

\section{Prior Art}

\vspace{5pt}

% by Srebro and  coauthors
\noindent {\bf{Implicit regularization:}} There is a growing interest in understanding properties of overparameterized problems. An interesting body of work investigate the implicit regularization capabilities of (stochastic) gradient descent for separable classification problems including \cite{soudry2017implicit,gunasekar2017implicit,neyshabur2017geometry,azizan2018stochastic,nacson2018stochastic,wilson2017marginal,neyshabur2014search}. These results show that gradient descent does not converge to an arbitrary solution, for instance, it has a tendency to converge to the solution with the max margin or minimal norm. Some of this literature apply to regression problems as well (such as low-rank regression). However, for regression problems based on a least-squares formulation the implicit bias/minimal norm property is proven under the assumption that gradient descent converges to a globally optimal solution which is not rigorously proven in these papers.

\noindent {\bf{Overparametrized low-rank regression}.}
As discussed in Section \ref{sec over rank}, there is a rich literature which studies global optimality of nonconvex low-rank factorization formulations such as the Burer-Monteiro factorization in the overparametrized regime\cite{li2018algorithmic,bhojanapalli2016global,boumal2016non,burer2003nonlinear}. These results typically require the factorization rank to be at least $\sqrt{n}$ to guarantee convergence of gradient descent. In contrast, with random data but arbitrary features, our results guarantee global convergence as long as $r\gtrsim n/d$. Specifically, for the problem of nonconvex low-rank regression discussed in this paper if one assumes the labels are created according to a low-rank matrix of rank $r^*$ (i.e.~$y_i=\langle\mtx{X}_i,\mtx{\Theta}_*\mtx{\Theta}_*^T\rangle$ with $\mtx{\Theta}_*\in\R^{d\times r^*}$) and the number of labels is on the order of $dr^*$ (i.e.~$n=cdr^*$) then these classical results require the fitted rank to be $r\ge \sqrt{dr^*}$ where as our results work as soon as $r\gtrsim r^*$. 
%\cite{li2018algorithmic} studies overparameterization for low-rank regression in a similar fashion to Section \ref{sec over rank}. 

\noindent {\bf{Overparameterized neural networks:}} A few recent papers \cite{zhang2016understanding,soltanolkotabi2018theoretical,brutzkus2017sgd,chizat2018global,arora2018optimization, Ji:2018aa, venturi2018spurious, Zhu:2018aa, Soudry:2016aa, Brutzkus:2018aa} study the benefits of overparameterization for training neural networks and related optimization problems. Very recent works \cite{li2018learning,allen2018learning,allen2018convergence,du2018gradient,zou2018stochastic, du2018gradient2} show that overparameterized neural networks can fit the data with random initialization if the number of hidden nodes are polynomially large in the size of the dataset. Similar to us, these works argue that there is a global minima around the random initialization. However these works are specialized towards neural nets and similar to us the bounds on the network size to achieve global optimality appear to be suboptimal.\footnote{We note that while both our results and these papers are suboptimal for one-hidden layer neural networks, they are not directly comparable with each other. We assume $n\le d$ where as these papers assume $poly(n)\lesssim k$. Also the assumptions on the activations are different from each other.} In contrast, we focus on general nonlinearities and also focus on the gradient descent trajectory showing that among all the global optima, gradient descent converges to one with near minimal distance to the initialization. We would also like to note that the importance of the Jacobian for overparameterized neural network analysis has also been noted by other papers including \cite{soltanolkotabi2018theoretical,du2018gradient} and also \cite{keskar2016large,sagun2017empirical,chaudhari2016entropy} which investigate the optimization landscape and properties of SGD for training neural networks. An equally important question to understanding the convergence behavior of optimization algorithms for overparameterized models is understanding their generalization capabilities this is the subject of a few interesting recent papers \cite{Arora:2018aa, Bartlett:2017aa, Golowich:2017aa, song2018mean, oymak2018learning, Brutzkus:2017aa, Belkin:2018aa, Liang:2018aa, Belkin:2018ab}. While our results do not directly address generalization, by characterizing the properties of the global optima that (stochastic) gradient descent converges to it may help demystify the generalization capabilities of overparametrized models trained via first order methods. Rigorous understanding of this relationship is an interesting and important subject for future research. 
%\vspace{5pt}
%\cite{bartlett2017spectrally,neyshabur2017exploring} provides generalization and capacity bounds for neural nets. 

\noindent {{\bf{Stochastic methods:}}} SGD performance guarantees are typically in expectation rather than in probability. Martingale-based methods have been utilized to give probabilistic guarantees \cite{rakhlin2012making,de2015taming}. The main challenge in nonconvex analysis of SGD, is to ensure SGD iterates stay within a region where nonconvex analysis can apply even when using rather large learning rates. While a few papers \cite{allen2018convergence,li2018learning} show that SGD stays in a specific region with high probability in specific instances, these results require using very small learning rates (which translates into very small variance) to ensure standard concentration arguments apply. In contrast,
our approach allows for much larger learning rates by using martingale stopping time arguments. Our approach is in part inspired by \cite{tan2017phase} which studies SGD for nonconvex phase retrieval but involves different assumptions on the loss. 
% addressesThe problem is more manageable if very small learning rates are used as one can apply standard concentration bounds \cite{}.  

%\cite{bassily2018exponential,ma2017power,karimi2016linear,vaswani2018fast}. %around the initialization% to study the non-convex optimization landscape.
\noindent {{\bf{Nonconvex optimization:}}} A key idea for solving nonconvex optimization problems is ensuring that optimization landscape has desirable properties. These properties include Polyak-Lojasiewicz (PL) condition \cite{lojasiewicz1963topological,polyak1963gradient} and the regularity condition (e.g.~local strong convexity) \cite{candes2015phase,li2018rapid, Soltanolkotabi:2017aa}. PL condition is particularly suited for analyzing overparameterized problems and has been utilized by several recent papers \cite{ma2017power,karimi2016linear,bassily2018exponential,vaswani2018fast,lei2017non}. Unlike these works, we show that overparameterized gradient descent trajectory stays in a small neighborhood and we only need properties such as PL to hold over this region. There is also a large body of work that study the applications discussed in this paper in the  over determined regime $p\le n$. For instance, Low-rank regression and generalized linear models have been considered by various works including \cite{sun2018geometric,tu2015low,bhojanapalli2016global,chen2018gradient, Josz:2018aa} in such an overdetermined setting. More recently, provable first order methods for learning neural networks have been investigated by multiple papers including \cite{zhong2017recovery,ge2017learning,soltanolkotabi2017learning,oymak2018stochastic,brutzkus2017globally} in the overdetermined setting. 

\section{Discussion and future directions}
This work provides new insights and theory for overparameterized learning with nonlinear models. We first provided a general convergence result for gradient descent and matching upper and lower bounds showing that if the Jacobian of the nonlinear mapping is well-behaved in a minimally small neighborhood, gradient descent finds a global minimizer which has a nearly minimal distance to the initialization. Second, we extend the results to SGD to show that SGD exhibits the same behavior and converges linearly without ever leaving a minimally small neighborhood of initializtion. Finally, we specialize our general theory to provide new results for overparameterized learning with generalized linear models, low-rank regression and shallow neural network training. A key tool in our results is that we introduce a potential function that captures the tradeoff between the model misfit and the distance to the initial point: the decrease in loss is proportional to the distance from the initialization. Our numerical experiments on real and synthetic data further corroborate this intuition on the loss-distance tradeoff.

In this work we address important challenges surrounding the optimization of nonlinear over-parametrized learning and some of its key features. The fact that gradient descent finds a nearby solution is a desirable property that hints as to why \emph{generalization} to new data instances may be possible. However, we emphasize that this is only suggestive of the generalization capabilities of such algorithms to new data. Indeed, developing a clear understanding of the generalization capabilities of first order methods when solving over-parameterized nonlinear problems is an important future direction.  Making progress towards this generalization puzzle requires merging insights gained from optimization with more intricate tools from statistical learning and is an interesting topic for future research.
\section{Proofs}
%$\Bc^{p-1}$ denotes the unit $\ell_2$ ball in $\R^p$ and 
% often used to analyze the optimization behavior when the problem is nonconvex
\subsection{Notations and definitions}
Before we begin the proof we briefly discuss some notation and definitions that will be used throughout. The spectral norm and the minimum singular value of a matrix $\A$ is denoted by $\|\A\|/\smx{\A}$ and $\smn{\A}$ respectively. $\tti{\A}$ denotes the largest $\ell_2$ norm among the rows of $\A$. $\Bc(\bteta,R)$ denotes the $\ell_2$ ball of radius $R$ around a vector $\bteta$. 

\noindent We introduce the following matrix and vector which play a crucial role in the convergence analysis of our algorithms
\begin{definition} [Average Jacobian] We define the average Jacobian along the path connecting two points $\vct{x},\vct{y}\in\R^p$ as
\begin{align}
&\Jc(\y,\x):=\int_0^1 \mathcal{J}(\x+\alpha(\y-\x))d\alpha.
\end{align}
\end{definition}
\begin{definition}[Residual error]
We also define the residual error at iteration $\tau$, denoted by $\vct{r}_\tau\in\R^n$, as the vector of misfits of the model to the labels. That is,
\begin{align*}
\vct{r}_\tau=f(\vct{\theta}_\tau)-\vct{y}.
\end{align*}
\end{definition}
\subsection{Gradient descent convergence proofs (Theorem \ref{GDthm} and Corollary \ref{mycor})}\label{GDproof}
Theorem \ref{GDthm} and Corollary \ref{mycor} are a special case of a more general result stated below. Theorem \ref{GDthm} and Corollary \ref{mycor} then follows by setting $\lambda=1/2$ and $\rho=1$.
\begin{theorem}\label{mainthm} Consider a nonlinear least-squares optimization problem of the form 
\begin{align*}
\underset{\vct{\theta}\in\R^p}{\min}\text{ }\mathcal{L}(\vct{\theta}):=\frac{1}{2}\twonorm{f(\vct{\theta})-\vct{y}}^2,
\end{align*}
 with $f:\R^p\mapsto \R^n$ and $\vct{y}\in\R^n$. Let $\lambda$ a scalar obeying $0<\lambda\le 1$. Suppose the Jacobian mapping associated with $f$ obeys Assumption \ref{wcond} over a ball of radius $R:=\frac{\twonorm{f(\vct{\theta}_0)-\vct{y}}}{\alpbb}$ around a point $\vct{\theta}_0\in\R^p$, that is $\mathcal{D}=\mathcal{B}\left(\vct{\theta}_0,\frac{\twonorm{f(\vct{\theta}_0)-\vct{y}}}{\alpbb}\right)$. Furthermore, suppose one of the following statements is valid.
\begin{itemize}
\item Assumption \ref{spert} (a) holds over $\mathcal{D}$ and set $\eta\leq \frac{\lambda}{\bp^2}$.
\item Assumption \ref{spert} (b) holds over $\mathcal{D}$ and set $\eta\leq \frac{1}{\bp^2}\cdot\min\left(\lambda,\frac{ 2(1-\la)\bn^2}{\el\twonorm{f(\vct{\theta}_0)-\vct{y}}}\right)$.
\end{itemize}
Then, running gradient descent updates of the form $\vct{\theta}_{\tau+1}=\vct{\theta}_\tau-\eta\nabla\mathcal{L}(\vct{\theta}_\tau)$ starting from $\vct{\theta}_0$, all iterates obey.
\begin{align}%\lambda^2\rho\left(2-\rho\right)\frac{\bn^2}{\bp^2}
\twonorm{f(\vct{\theta}_\tau)-\vct{y}}^2\le&\left(1-\alpha^2\la\eta\right)^\tau\twonorm{f(\vct{\theta}_0)-\vct{y}}^2,\label{mainthm1}\\
\alpbb\twonorm{\vct{\theta}_\tau-\vct{\theta}_0}+\twonorm{f(\vct{\theta}_\tau)-\vct{y}}\le&\twonorm{f(\vct{\theta}_0)-\vct{y}}.\label{mainthm2}
\end{align}
Furthermore, the total gradient path is bounded. That is,
\begin{align}
\sum_{\tau=0}^\infty\twonorm{\vct{\theta}_{\tau+1}-\vct{\theta}_\tau}\le \frac{\twonorm{f(\vct{\theta}_0)-\vct{y}}}{\alpbb}.\label{mainthm3}
\end{align}
Let $\vct{\theta}^*$ denote the global optima of the loss $\mathcal{L}(\vct{\theta})$ with smallest Euclidean distance to the initial parameter $\vct{\theta}_0$. Then, the gradient descent iterates $\vct{\theta}_\tau$ also obey
\begin{align}%\frac{1}{\lambda(1-\rho/2)}\frac{\bp}{\bn}
\twonorm{\vct{\theta}_\tau-\vct{\theta}_0}\le \frac{\bp}{\alpbb}\twonorm{\vct{\theta}^*-\vct{\theta}_0},\label{mainthm4}\\
\sum_{\tau=0}^\infty\twonorm{\vct{\theta}_{\tau+1}-\vct{\theta}_\tau}\le \frac{\bp}{\alpbb}\twonorm{\vct{\theta}^*-\vct{\theta}_0}.\label{mainthm5}
\end{align}
\end{theorem}
\noindent {\bf{Proof Sketch.}} To prove the above theorem we begin by noting that the residual $\rb_\tau$ satisfies the recursion
\begin{align}
\rb_{\tau+1}=&\rb_{\tau}-f(\bteta_{\tau})+f(\bteta_{\tau+1})\nn\\
\overset{(a)}{=}&\rb_{\tau}+\Jc(\bteta_{\tau+1},\bteta_{\tau})(\bteta_{\tau+1}-\bteta_{\tau})\nn\\
\overset{(b)}{=}&\rb_\tau-\eta\Jc(\bteta_{\tau+1},\bteta_\tau)\Jc(\bteta_\tau)^T\rb_{\tau}\nn\\
%&=(\Iden-\eta\Jc(\bteta_{i},\bteta_{i})\Jc(\bteta_{i})^T)\rb_{i}\\
=&~(\Iden-\eta\Cb(\bteta_\tau))\rb_{\tau}.\label{line 4}
\end{align}
where $\Cb(\bteta_\tau):=\Jc(\bteta_{\tau+1},\bteta_\tau)\Jc(\bteta_\tau)^T$. Here, (a) follows from fundamental rule of calculus and (b) from the gradient identity $\nabla\mathcal{L}(\vct{\theta}_\tau)=\mathcal{J}^T(\vct{\theta}_\tau)\vct{r}_\tau$. If $\Iden-\eta\Cb(\bteta_\tau)$ has spectral norm less than $1$, the the residual verctors will converge linearly. We build on this observation and show that one only needs this requirement over a minimally small neighborhood of $\bteta_0$. To this aim, we first introduce a potential set which contains the space of parameters that can be reached by gradient descent.
\begin{definition} [Potential sub-level set] Given a scalar $\alpb>0$, define the radius $R_{\alpb}=\frac{\tn{f(\bteta_0)-\y}}{\alpb}$. The potential sub-level set $\Pc(\bteta_0,R_{\alpb})$ is defined as
\begin{align}
\Pc(\bteta_0,R_{\alpb})=\Bigg\{\bteta\in\R^p\bgl \tn{\bteta-\bteta_0}+\frac{\tn{f(\bteta)-\y}}{\alpb}\leq R_{\alpb}\Bigg\}\label{potent}.
\end{align}
\end{definition}
%satisfies \begin{align}
%\tn{\bteta-\bteta_0}+\frac{\tn{\y-f(\bteta)}}{\alpb}\leq R.\label{achievable eq new}
%\end{align} If
Note that $\mathcal{P}(\bteta_0,R_{\alpb})\subseteq {\Bc}(\bteta_0,R_\alpb)$. Our first lemma shows that, if an iterate $\vct{\theta}_\tau\in\Pc:=\mathcal{P}(\bteta_0,R_{\alpb})$, then the next iterate $\bteta_{\tau+1}$ stays in the set $\Dc:=\mathcal{B}(\vct{\theta}_0,R_{\alpb})$.
\begin{lemma}\label{next inside} Suppose Assumption \ref{wcond} holds over the domain $\Dc=\Bc\left(\bteta_0,\frac{\tn{f(\bteta_0)-\y}}{\alpb}\right)$ for some $\alpb$ obeying $\alpb\le \alpha$. Also assume $\bteta\in\mathcal{P}(\bteta_0,R_{\alpb})$, then gradient iterate $\bteta^{+}=\bteta-\eta\grad{\bteta}$ with $\eta\le \frac{1}{\bp^2}$ satisfies $\bteta^{+}\in\Dc$.
\end{lemma} %and $\eta\leq 1/\bp^2$
\begin{proof} 
 We begin by noting that
 \begin{align}
 \label{lem85temp2}
\tn{\bteta^{+}-\bteta}=\eta \tn{\Jc^T(\bteta)\left(f(\vct{\theta})-\vct{y}\right)}\overset{(a)}{\le} \eta \bp\tn{f(\vct{\theta})-\vct{y}}\overset{(b)}{\le} \frac{\tn{f(\vct{\theta})-\vct{y}}}{\bp}\overset{(c)}{\le} \frac{\tn{f(\vct{\theta})-\vct{y}}}{\bn} 
\overset{(d)}{\le} \frac{\tn{f(\vct{\theta})-\vct{y}}}{\alpb}.%\frac{\tn{\rb_\tau}}{\bn \lambda(1-\alpb/2)}
\end{align}
 In the above, (a) follows from the upper bound on the Jacobian over $\mathcal{D}$ per Assumption \ref{wcond}, (b) from the fact that $\eta\le \frac{1}{\bp^2}$, (c) from $\bn\le\bp$, and (d) from $\alpb \leq \alpha$. The latter combined with the triangular inequality yields
\begin{align*}
\tn{\bteta^{+}-\bteta_0}\leq \tn{\bteta^{+}-\bteta}+\tn{\bteta_{0}-\bteta}\le \tn{\bteta-\bteta_0}+\frac{\tn{f(\bteta)\y}}{\alpb}\leq R_{\alpb},
\end{align*}
concluding the proof of $\vct{\theta}^{+}\in\mathcal{D}$. 
\end{proof}
%The next lemma builds on this observation and shows that we only need this requirement over a minimally small neighborhood of $\bteta_0$.
The next lemma establishes the convergence to a global minima that lies in a minimally small local neighborhood under a Jacobian condition \eqref{ctheta}. The proof of this lemma is deferred to Section \ref{pthm 1}.
\begin{lemma}\label{thm 1} Suppose the Jacobian mapping associated with $f$ obeys Assumption \ref{wcond} over a ball of radius $R_{\alpb}:=\frac{\twonorm{f(\vct{\theta}_0)-\vct{y}}}{\alpb}$ around a point $\vct{\theta}_0\in\R^p$, that is $\mathcal{D}=\mathcal{B}\left(\vct{\theta}_0,\frac{\twonorm{f(\vct{\theta}_0)-\vct{y}}}{\alpb}\right)$. Let $\lambda$ be a scalar obeying $0<\lambda\le 1$ and set $\alpb=\alpbb$. Also assume 
\begin{align}
\label{ctheta}
 \Cb(\vct{\theta})\succeq \lambda\Jc(\vct{\theta})\Jc(\vct{\theta})^T
 \end{align}
holds for all $\vct{\theta}\in \Pc\left(\vct{\theta}_0,\frac{\twonorm{f(\vct{\theta}_0)-\vct{y}}}{\alpb}\right)$. Then, staring from $\vct{\theta}_0$ the GD iterates $\vct{\theta}_{\tau+1}=\vct{\theta}_{\tau}-\eta\nabla\mathcal{L}(\vct{\theta}_{\tau})$ with $\eta\le \frac{\lambda}{\bp^2}$ obey
\begin{align}
\twonorm{f(\vct{\theta}_\tau)-\vct{y}}^2\le&\left(1-\bn^2\la \eta\right)^\tau\twonorm{f(\vct{\theta}_0)-\vct{y}}^2,\label{achievable eq11}\\
\alpb\twonorm{\vct{\theta}_\tau-\vct{\theta}_0}+\twonorm{f(\vct{\theta}_\tau)-\vct{y}}\le&\twonorm{f(\vct{\theta}_0)-\vct{y}}.\label{achievable eq21}
\end{align}
Furthermore, the total gradient path is bounded. That is,
\begin{align}
\sum_{\tau=0}^\infty\twonorm{\vct{\theta}_{\tau+1}-\vct{\theta}_\tau}\le \frac{\twonorm{f(\vct{\theta}_0)-\vct{y}}}{\alpb}.\label{GDpath}
\end{align}
\end{lemma} 
The next lemma shows that \eqref{ctheta} indeed holds. We defer the proof of this lemma to Section \ref{pfsmalldev}.
%=\mathcal{B}\left(\vct{\theta}_0,\frac{\twonorm{f(\vct{\theta}_0)-\vct{y}}}{\alpbb}\right)
\begin{lemma} \label{small dev} Consider a point $\vct{\theta}\in\R^p$ and the result of a gradient update $\vct{\theta}^{+}=\vct{\theta}-\eta\nabla\mathcal{L}(\vct{\theta})$ staring from $\vct{\theta}$. Suppose Assumption \ref{wcond} and one of the following two statements hold over $\Dc={\Bc}\left(\bteta_0,\frac{\twonorm{f(\vct{\theta}_0)-\y}}{\alpb}\right)$ for a $\alpb$ obeying $0\le\alpb\le \bn$
\begin{itemize}
\item Assumption \ref{spert}(a) holds over $\mathcal{D}$ and $\eta\leq \frac{1}{\bp^2}$
\item Assumption \ref{spert}(b) holds over $\mathcal{D}$ and $\eta\le \frac{1}{\bp^2}\min\left(1,\frac{2(1-\lambda)\bn^2}{L\twonorm{f(\bteta_0)-\y}}\right)$. %$\eta\leq \frac{\lambda\rho}{\bp^2}\cdot\min\left(1,\frac{ \bn^2}{\el\twonorm{f(\vct{\theta}_0)-\vct{y}}}\right)$.
\end{itemize}
Then for all $\bteta\in \mathcal{P}\left(\bteta_0,\frac{\twonorm{f(\vct{\theta}_0)-\y}}{\alpb}\right)$, 
\begin{align*}
 \Cb(\vct{\theta}):=\Jc(\bteta^{+},\bteta)\Jc(\bteta)^T\succeq \lambda\Jc(\vct{\theta})\Jc(\vct{\theta})^T.
 \end{align*}
%holds for all $\vct{\theta}\in\mathcal{D}$.
%, GD algorithm with learning rate $\eta$ obeys \eqref{achievable eq1} and \eqref{achievable eq2}.
%\begin{align}
%&\sqrt{2\Lc(\bteta_i)}\leq (1-\eta \bn^2/2)^i\tn{\rb_0},\label{achievable eq1}\\
%&\sqrt{2\Lc(\bteta_i)}+\frac{ \bn^2}{2 \bp}\tn{\bteta-\bteta_i}\leq \tn{\rb_0}.\label{achievable eq2}
%\end{align}
%If Assumption \ref{spert} holds as well, setting $\eta=$, GD satisfies
%\[
%\]
\end{lemma}%\SO{left here} 
With these lemmas in place we are now ready to prove Theorem \ref{mainthm}.\\
{\bf{Proof of Theorem \ref{mainthm}:}} %We suppose \eqref{mainthm2} holds until iteration $\tau$ and show the result for $\tau+1$. 
Set $\alpb=\alpbb$ and observe that
%\item Since $\eta\leq 1/\bp^2$, using Lemma \ref{next inside} with $\alpb=\alpbb$ we find $\bteta_{\tau+1}\in\Dc$.
% $\bteta_{\tau}\in\bar{\Bc}(\bteta_0,R)$, Lemma \ref{small dev} is applicable with $\alpb=\alpbb$ and yields $\Cb(\bteta_{\tau})\succeq \la \Jc(\bteta_\tau)\Jc(\bteta_\tau)^T$. 
%\item With the above two, 
\begin{itemize}
\item Since assumptions of Theorem \ref{mainthm} subsume those of Lemma \ref{small dev}, for all $\bteta\in \mathcal{P}\left(\bteta_0,\frac{\twonorm{f(\vct{\theta}_0)-\y}}{\alpb}\right)$, \eqref{ctheta} holds i.e.~we have that $\Cb(\bteta)\succeq \la \Jc(\bteta)\Jc(\bteta)^T$.
\item Based on the above, the assumptions of Theorem \ref{mainthm} also subsume those of Lemma \ref{thm 1}. Thus \eqref{achievable eq11}, \eqref{achievable eq21}, and \eqref{GDpath} hold for all $\tau$.
\end{itemize}
%Assumptions of Theorem \ref{mainthm} coincide with requirements of Lemma \ref{small dev}, hence \eqref{ctheta} holds for $\bteta_\tau$.
This completes the bounds \eqref{mainthm1}, \eqref{mainthm2}, and \eqref{mainthm3} of Theorem \ref{mainthm}. 
% by combining Lemmas \ref{thm 1} and \ref{small dev} above.
The proofs of \eqref{mainthm4} and \eqref{mainthm5} follow immediately from \eqref{mainthm2} and \eqref{mainthm3} by noting that for any global optima (including the closest global optima to $\vct{\theta}_0$ denoted by $\vct{\theta}^*$) we have
\begin{align*}
\twonorm{\vct{y}-f(\vct{\theta}_0)}=&\twonorm{f(\vct{\theta}^*)-f(\vct{\theta}_0)}\nn\\
=&\twonorm{\int_{0}^1 \mathcal{J}^T\left(\vct{\theta}_0+t(\vct{\theta}^*-\vct{\theta}_0)\right)(\vct{\theta}^*-\vct{\theta}_0) dt}\nn\\
\le&\underset{0\le 1\le t}{\sup}\opnorm{\mathcal{J}\left(\vct{\theta}_0+t(\vct{\theta}^*-\vct{\theta}_0)\right)}\twonorm{\vct{\theta}^*-\vct{\theta}_0}\nn\\
\le&\underset{\vct{\theta}\in\mathcal{D}}{\sup}\opnorm{\mathcal{J}\left(\vct{\theta}\right)}\twonorm{\vct{\theta}^*-\vct{\theta}_0}\\
\le&\bp\twonorm{\vct{\theta}^*-\vct{\theta}_0}.
\end{align*}
This concludes the proof of Theorem \ref{mainthm}. All that remains is to prove Lemmas \ref{thm 1} and \ref{small dev} which are the subject of the two sections below.
%\end{proof}
\subsubsection{Proof of Lemma \ref{thm 1}}\label{pthm 1}
We will prove this lemma by induction. Assume the claim holds until iteration $\tau$. First, since \eqref{achievable eq21} holds, applying Lemma \ref{next inside} and using the facts that $\eta\leq 1/\bp^2$ and $\alpb\leq \bn$, we can conclude that $\bteta_{\tau+1}\in\Dc$.
%, we will show that the $\tau+1$ iterate remains in $\mathcal{D}$. To this aim, we use the triangle inequality
%\[
%\tn{\bteta_{\tau+1}-\bteta_0}\leq \tn{\bteta_{\tau+1}-\bteta_\tau}+\tn{\bteta_{0}-\bteta_\tau}.
%\]
%Since \eqref{achievable eq21} holds for $\bteta_\tau$, scaling both sides by $\alpb$, we find
%\[
%\tn{\bteta_{0}-\bteta_\tau}+\frac{\tn{\y-f(\bteta_\tau)}}{\alpb}\leq R.
%\]
%From the last two inequalities, observe that, if $\tn{\bteta_{\tau+1}-\bteta_\tau}\leq \frac{\tn{\y-f(\bteta_\tau)}}{\alpb}$, we can immediately conclude $\tn{\bteta_{\tau+1}-\bteta_0}\leq R$ i.e. $\bteta_{\tau+1}\in \Dc$. Denoting $\rb_\tau=\y-f(\bteta_\tau)$, this can be shown as follows using the fact that $\eta\leq 1/\bp^2$ and $\alpb\leq \alpha$,
%\begin{align*}\
%\tn{\bteta_{\tau+1}-\bteta_\tau}=\eta \tn{\Jc(\bteta_\tau)\rb_\tau}\overset{(a)}{\le} \eta \bp\tn{\rb_\tau}\overset{(b)}{\le} \frac{\tn{\rb_\tau}}{\bp}\overset{(c)}{\le} \frac{\tn{\rb_\tau}}{\bn} 
%\overset{(d)}{\le} \frac{\tn{\rb_\tau}}{\alpb}.%\frac{\tn{\rb_\tau}}{\bn \lambda(1-\rho/2)}
%\end{align*}
%Hence, $\vct{\theta}_{\tau+1}\in\mathcal{D}$. In the above, (a) follows from the upper bound on the Jacobian over $\mathcal{D}$ per Assumption \ref{wcond}, (b) from the fact that $\eta\le \frac{\lambda}{\bp^2}\le \frac{1}{\bp^2}$, (c) from $\bn\le\bp$, and (d) from the fact that $\alpb \leq \la \alpha \leq \alpha$.% by using \eqref{achievable eq11} guaranteed by the induction hypothesis.

Next, we will simultaneously monitor how the distance to the initial parameter $\vct{\theta}_0$ ($\twonorm{\vct{\theta}_\tau-\vct{\theta}_0}$) and the Euclidean norm of the residual ($\twonorm{\vct{r}_\tau}$) change from iteration $\tau$ to $\tau+1$. For the distance to initialization, using triangular inequality and the formula for the gradient we have
\begin{align}
\label{distance}
\tn{\bteta_{\tau+1}-\bteta_0}&\leq \tn{\bteta_\tau-\bteta_0}+\tn{\bteta_{\tau+1}-\bteta_\tau}= \tn{\bteta_\tau-\bteta_0}+\eta \tn{\Jc(\bteta_\tau)\rb_\tau}.
\end{align}
For the norm of the residual using the fact that $ \Cb(\vct{\theta})\succeq \lambda\Jc(\vct{\theta})\Jc(\vct{\theta})^T$ (per assumption \eqref{ctheta}) we have
\begin{align}
\label{preident}
\twonorm{\rb_{\tau+1}}^2\overset{(a)}{=}& \tn{(\Iden-\eta\Cb(\bteta_\tau))\rb_\tau}^2,\nn\\
=& \tn{\rb_\tau}^2-2\eta\rb_\tau^T\Cb(\bteta_\tau)\rb_\tau+\eta^2\rb_\tau^T\Cb(\bteta_\tau)^T\Cb(\bteta_\tau)\rb_\tau,\nn\\
\overset{(b)}{\le}& \tn{\rb_\tau}^2-2\lambda \eta\rb_\tau^T\Jc(\vct{\theta}_\tau)\Jc(\vct{\theta}_\tau)^T\rb_\tau+\eta^2 \bp^2\rb_\tau^T\Jc(\vct{\theta}_\tau)\Jc(\vct{\theta}_\tau)^T\rb_\tau,\nn\\
\overset{(c)}{\le}& \tn{\rb_\tau}^2-(2\lambda-\eta\bp^2)\eta\tn{\Jc(\vct{\theta}_\tau)^T\rb_\tau}^2.
\end{align}%$\eta\le \frac{1}{\bp^2}$.
Here, (a) follows from \eqref{line 4}, (b) from \eqref{ctheta} and the upper bound on the spectral norm of the Jacobian, (c) and from merging the terms on the right hand side. Combining \eqref{preident} with $\smn{\mathcal{J}(\vct{\theta}_\tau)}\ge \bn$, and using $\eta\leq \la/\bp^2$, we conclude that
\begin{align*}
\twonorm{\rb_{\tau+1}}^2\le \left(1-\bn^2(2\lambda-\eta\bp^2)\eta\right)\twonorm{\rb_{\tau}}^2\le \left(1-\la \bn^2\eta\right)\twonorm{\rb_{\tau}}^2,
\end{align*}
completing the proof of \eqref{achievable eq11}. For the remainder of discussion, denote $\gam=(\lambda-\eta\bp^2/2)\eta$. $\gam$ is nonnegative due to upper bound on $\eta$ and we have
\[
\twonorm{\rb_{\tau+1}}^2\leq \tn{\rb_\tau}^2-2\gam\tn{\Jc(\vct{\theta}_\tau)^T\rb_\tau}^2.
\]

We now turn our attention to proving \eqref{achievable eq21}. To this aim we start from \eqref{preident} and complete the square to conclude that
\begin{align}
\label{dist2}
\twonorm{\rb_{\tau+1}}^2=&\left(\twonorm{\rb_\tau}-\gam\frac{\twonorm{\Jc(\vct{\theta}_\tau)^T\rb_\tau}^2}{\twonorm{\rb_\tau}}\right)^2-\left(\gam\frac{\twonorm{\Jc(\vct{\theta}_\tau)^T\rb_\tau}^2}{\twonorm{\rb_\tau}}\right)^2,\nn\\
\le&\left(\twonorm{\rb_\tau}-\gam\frac{\twonorm{\Jc(\vct{\theta}_\tau)^T\rb_\tau}^2}{\twonorm{\rb_\tau}}\right)^2.
\end{align}
Also note that using the upper bound on spectrum of $\mathcal{J}$ and $\gam\le \la\eta \le\frac{1}{\bp^2}$ we have%\frac{\lambda(1-\rho/2)}{\bp^2}
\begin{align*}
\twonorm{\vct{r}_\tau}^2\ge \frac{1}{\bp^2}\twonorm{\Jc(\vct{\theta}_\tau)^T\rb_\tau}^2\ge \gam\twonorm{\Jc(\vct{\theta}_\tau)^T\rb_\tau}^2\quad\Rightarrow\quad \twonorm{\rb_\tau}-\gam\frac{\twonorm{\Jc(\vct{\theta}_\tau)^T\rb_\tau}^2}{\twonorm{\rb_\tau}}\ge 0.
\end{align*}
Thus, taking square root from both sides of \eqref{dist2} we reach the following identity for changes in the norm of residual
\begin{align}
\label{residual}
\twonorm{\rb_{\tau+1}}\le \twonorm{\rb_\tau}-\gam\frac{\twonorm{\Jc(\vct{\theta}_\tau)^T\rb_\tau}^2}{\twonorm{\rb_\tau}}.
\end{align}
To combine the identities \eqref{distance} and \eqref{residual} in such a way to yield our theorem we proceed by defining the potential/Lyapunov function below with $\alpb=\bn\gam/\eta$.
\begin{align}
\mathcal{V}_\tau:=&\tn{\rb_\tau}+\alpb\sum_{t=0}^{\tau-1}\tn{\bteta_{t+1}-\bteta_t}.
\end{align}
A unique feature of the $\mathcal{V}_\tau$ potential is that it is non-increasing. To see this note that using \eqref{residual} we have
\begin{align}
\label{noninc}
\frac{1}{\eta}\left(\mathcal{V}_{\tau+1}-\mathcal{V}_\tau\right)=&\frac{1}{\eta}\left(\tn{\rb_{\tau+1}}-\tn{\rb_\tau}\right)+\frac{\alpb}{\eta}\twonorm{\vct{\theta}_{\tau+1}-\vct{\theta}_\tau},\nn\\
\overset{(a)}{=}&\frac{1}{\eta}\left(\tn{\rb_{\tau+1}}-\tn{\rb_\tau}\right)+\alpb\twonorm{\Jc(\vct{\theta}_\tau)^T\rb_\tau},\nn\\
\overset{(b)}{\le}&-\frac{\gam}{\eta}\frac{\twonorm{\Jc(\vct{\theta}_\tau)^T\rb_\tau}^2}{\twonorm{\rb_\tau}}+\alpb\twonorm{\Jc(\vct{\theta}_\tau)^T\rb_\tau},\nn\\
=&\twonorm{\Jc(\vct{\theta}_\tau)^T\rb_\tau}\left(\alpb-\frac{\gam}{\eta}\frac{\twonorm{\Jc(\vct{\theta}_\tau)^T\rb_\tau}}{\twonorm{\vct{r}_\tau}}\right),\nn\\
\overset{(c)}{\le}&\twonorm{\Jc(\vct{\theta}_\tau)^T\rb_\tau}\left(\alpb-\bn\frac{\gam}{\eta}\right),\nn\\
=&~0.
%\le \alpb \tn{\Jc(\bteta_i)^T\rb_i}-\alpb\frac{\tn{\Jc(\x)^T\rb_i}^2}{\tn{\rb_i}}\leq \tn{\Jc(\bteta_i)\rb_i}(\alpb-\alpb\frac{\tn{\Jc(\x)^T\rb_i}}{\tn{\rb_i}})\leq 0
\end{align}
Here, (a) follows from the gradient formula, (b) from \eqref{residual}, (c) from $\smn{\mathcal{J}(\vct{\theta}_\tau)}\ge \bn$, and (d) from $\alpb=\bn\gam/\eta$. Using this non-increasing property and triangle inequality over $(\tn{\bteta_{\tau+1}-\bteta_\tau})_{\tau\geq 0}$ we can conclude that
\begin{align*}
\tn{\rb_\tau}+\alpb\tn{\bteta_\tau-\bteta_0}\le\mathcal{V}_\tau\le \mathcal{V}_0=\twonorm{\rb_0},
\end{align*}
proving \eqref{achievable eq21}. 

Finally using the definition of $\mathcal{V}_\tau$ and its non-increasing property \eqref{noninc} we have
\begin{align*}
\sum_{\tau=0}^{\infty}\tn{\bteta_{\tau+1}-\bteta_\tau}\le\frac{\mathcal{V}_{\infty}}{\alpb}\le \frac{\mathcal{V}_{0}}{\alpb}=\frac{\twonorm{\rb_0}}{\alpb},
\end{align*}
concluding the proof of \eqref{GDpath} and Lemma \ref{thm 1} when we substitute $\alpb=(\lambda-\eta\bp^2/2)\bn$.
%\begin{proof} We claim that for all $i\geq 0$, \eqref{line 4} holds and 
%\begin{align}
%\tn{\bteta_i-\bteta_0}\leq R_i=(1-\rho^i)R.\label{induct R}
%\end{align} Suppose these inequalities hold until step $i$; which implies $\tn{\rb_i}\leq \rho^i\tn{\rb_0}$. At step $i+1$, we have that
%\begin{align}
%\tn{\bteta_{i+1}-\bteta_0}&\leq \tn{\bteta_i-\bteta_0}+\tn{\bteta_{i+1}-\bteta_i}\\
%&= R_i+\eta \tn{\grad{\bteta_i}}\\
%&\leq R_i+\eta \|\Jc(\bteta)\|\tn{\rb_i}\\
%&\leq R_i+\eta\rho^i \bp\tn{\rb_0}=R_{i+1}.
%\end{align}
%Since, $\bteta_i,\bteta_{i+1}$ are both inside $\Dc$, \eqref{line 4} is applicable and yields $\tn{\rb_{i+1}}\leq \rho\tn{\rb_i}\leq \rho^{i+1}\tn{\rb_0}$. Finally, note that $\tn{\bteta_i-\bteta_0}/R+\tn{\rb_i}/\tn{\rb_0}\leq (1-\rho^i)+\rho^i\leq 1$.
%%
%%Assuming $\bteta_i,\bteta_{i+1}\in \Dc$, we have the recursion \eqref{line 4}. We next claim that for all $i$
%%\[
%%\]
%\end{proof}
\subsubsection{Proof of Lemma \ref{small dev}}\label{pfsmalldev}
First note that since $\bteta\in\mathcal{P}\left(\bteta_0,\frac{\twonorm{f(\vct{\theta}_0)-\y}}{\alpb}\right)$, we have 
\begin{align}
\label{temppf87}
\tn{\y-f(\bteta)}\leq \tn{\y-f(\bteta_0)}.
\end{align}
Second, applying Lemma \ref{next inside}, we also have $\bteta^{+}=\bteta-\eta \grad{\bteta}\in\Dc:=\Bc\left(\bteta_0,\frac{\tn{f(\bteta_0)-\y}}{\alpb}\right)$. To prove 
\begin{align}
\Cb(\bteta)\succeq \lambda\Jc(\bteta)\Jc(\bteta)^T,\label{cbi bound}
\end{align} 
we consider the two cases related to Assumption \ref{spert} separately.

If Assumption \ref{spert}(a) holds then for any $\vct{\theta}_1,\vct{\theta}_2\in\mathcal{D}$ we have
\begin{align*}
\opnorm{\mathcal{J}\left(\vct{\theta}_2,\vct{\theta}_1\right)-\mathcal{J}(\vct{\theta}_1)}=&\opnorm{\int_0^1\left(\mathcal{J}\left(\vct{\theta}_1+t\left(\vct{\theta}_2-\vct{\theta}_1\right)\right)-\mathcal{J}(\vct{\theta}_1)\right)dt},\\
\le&\int_0^1\opnorm{\mathcal{J}\left(\vct{\theta}_1+t\left(\vct{\theta}_2-\vct{\theta}_1\right)\right)-\mathcal{J}(\vct{\theta}_1)} dt,\\
\le&\int_0^1 \frac{(1-\lambda)\bn^2}{\bp}dt,\\
\le&\frac{(1-\lambda)\bn^2}{\bp}.
\end{align*}
Thus for $\bteta,\bteta^{+}\in\mathcal{D}$ we have
\begin{align*}
\opnorm{\left(\Jc(\bteta^{+},\bteta)-\Jc(\bteta)\right)\Jc(\bteta)^T}&\le \opnorm{\Jc(\bteta^{+},\bteta)-\Jc(\bteta)}\opnorm{\Jc(\bteta)}\\
&\le \frac{(1-\lambda)\bn^2}{\bp}\bp\\
&=(1-\lambda)\bn^2\\
&\leq (1-\lambda)\sigma_{\min}^2\left(\Jc(\bteta)\right).
\end{align*}
Thus we have
\begin{align*}
\mathcal{C}(\bteta)=&\Jc(\bteta^{+},\bteta)\Jc(\bteta)^T,\\
=&\Jc(\bteta^{+},\bteta)\Jc(\bteta)^T-\Jc(\bteta)\Jc(\bteta)^T+\Jc(\bteta)\Jc(\bteta)^T,\\
\succeq& \Jc(\bteta)\Jc(\bteta)^T-\mtx{I}_n\opnorm{\left(\Jc(\bteta^{+},\bteta)-\Jc(\bteta)\right)\Jc(\bteta)^T},\\
\succeq& \lambda\Jc(\bteta)\Jc(\bteta)^T.
\end{align*}
This implies the desired bound \eqref{cbi bound}.

Next, suppose Assumption \ref{spert}(b) holds. Then, for any $\vct{\theta}_1,\vct{\theta}_2\in\mathcal{D}$ we have
\begin{align}
\label{avgsmooth}
\opnorm{\mathcal{J}\left(\vct{\theta}_2,\vct{\theta}_1\right)-\mathcal{J}(\vct{\theta}_1)}=&\opnorm{\int_0^1\left(\mathcal{J}\left(\vct{\theta}_1+t\left(\vct{\theta}_2-\vct{\theta}_1\right)\right)-\mathcal{J}(\vct{\theta}_1)\right)dt},\nn\\
\le&\int_0^1\opnorm{\mathcal{J}\left(\vct{\theta}_1+t\left(\vct{\theta}_2-\vct{\theta}_1\right)\right)-\mathcal{J}(\vct{\theta}_1)} dt,\nn\\
\le&\int_0^1 tL\twonorm{\vct{\theta}_2-\vct{\theta}_1}dt,\nn\\
\le&\frac{L}{2}\twonorm{\vct{\theta}_2-\vct{\theta}_1}.
\end{align}
Thus, for $\eta\le \frac{2(1-\lambda)\bn^2}{L\bp^2\twonorm{\vct{r}_0}}$, 
\[
\opnorm{\Jc(\bteta^{+},\bteta)-\Jc(\bteta)}\leq \frac{\el}{2}\twonorm{\bteta^{+}-\bteta}=\frac{\eta\el}{2}\twonorm{\mathcal{J}^T(\bteta)\left(f(\vct{\theta})-\y\right)}\leq \frac{\eta \bp\el}{2} \twonorm{f(\vct{\theta})-\y}\overset{\eqref{temppf87}}{\leq} \frac{\eta \bp \el}{2}\tn{f(\bteta_0)-\y}\leq \frac{(1-\lambda)\bn^2}{ \bp},
\]
Repeating the previous argument (with Assumption \ref{spert}(a)), we again conclude with \eqref{cbi bound}.
%To proceed with $\Cb(\bteta_i)\succeq \Jc(\bteta_i)\Jc(\bteta_i)^T/2$ we simply need to show that $\bteta_{i+1}\in\Dc$. This can be seen by noticing that %where the last inequality follows from $\|\Cb(\x)-\Jc(\x)\Jc(\x)^T\|=$ where $\y=\x-\eta\grad{\x}$.

\subsection{Lower bounds proofs (Theorem \ref{low bound thm})}\label{lowbnd}
We begin by proving \eqref{total resi}. To show this we first use the upper bound on the Jacobian matrix to prove that the nonlinear mapping is Lipschitz. To this aim note that
\[
f(\bteta)-f(\bteta_0)=\int_0^1 \Jc\left(\bteta_0+t(\bteta-\bteta_0)\right) (\bteta-\bteta_0)dt=\Jc(\bteta,\bteta_0)(\bteta-\bteta_0).
\]
Hence, 
\begin{align*}
\tn{f(\bteta)-f(\bteta_0)}\leq \tn{\Jc(\bteta,\bteta_0)(\bteta-\bteta_0)}\leq \bp\tn{\bteta-\bteta_0},
\end{align*}
completing the proof of the Lipschitz property. This Lipschitz property combined with the triangular inequality allows us to conclude
\begin{align*}
\twonorm{\vct{y}-f(\vct{\theta}_0)}\le \twonorm{f(\vct{\theta})-f(\vct{\theta}_0)}+\twonorm{\vct{y}-f(\vct{\theta})}\le \bp\twonorm{\vct{\theta}-\vct{\theta}_0}+\twonorm{\vct{y}-f(\vct{\theta})},
\end{align*}
completing the proof of \eqref{total resi}.

Next we turn out attention to providing the counter examples. Consider a least squares problem where the loss is equal to $\Lc(\bteta)=\frac{1}{2}\tn{\y-\X\bteta}^2$ and the data matrix $\X$ has orthogonal rows. Suppose the first row $\x_1$ has the smallest $\ell_2$ norm which is $\bn$ and the last row $\x_n$ has the largest $\ell_2$ norm equal to $\bp$. We also set the labels to $\y=\X\bteta^\star$ where $\bteta^\star=\gamma\x_1/\tn{\x_1}$ with $\gamma=\beta/\alpha$. For this linear regression problem, the Jacobian is equal to $\X$ and since the matrix is orthogonal $\bn,\bp$ are the minimum/maximum singular values of the Jacobian.

For any $\alpha,\beta\ge 0$ obeying $\alpha\le\beta$ and any $\vct{\theta}$, we have
\[
\tn{\y-f(\bteta)}=\tn{\X\bteta-\y}\geq  \tn{\X(\bteta-\bteta^\star)}\geq \tn{\x_1^T(\bteta-\bteta_\star)}\geq\tn{\x_1^T\bteta_\star}-\tn{\x_1^T\bteta}\geq \bn(\gamma-\tn{\bteta}).
\]
This yields $\tn{f(\bteta)-\y}+\bn \tn{\bteta}\ge \tn{\y}=\gamma\alpha$ which in turns implies \eqref{total resi2} with $\vct{\theta}_0=\vct{0}$. 

To show \eqref{total resi3}, we set the labels to $\y=\X\bteta^\star$ where $\bteta^\star=\gamma\frac{\x_n}{\tn{\x_n}}$. In this case, gradient iteration starting from $\vct{\theta}_0=\vct{0}$ is simply
\[
\bteta_{\tau+1}=\bteta_\tau +\eta \X^T(\y-\X\bteta_\tau).
\]
If $\bteta_\tau\subset \text{span}(\x_n)$, it is clear that $\bteta_{\tau+1}\subset \text{span}(\x_n)$ as well as $\X^T\y\subset \text{span}(\x_n)$. Since $\bteta_0=0$, this implies that gradient descent recursion is one dimensional over $\x_n$ i.e.~$\bteta_\tau=\frac{\x_n}{\tn{\x_n}}\theta_\tau$ with $\theta_\tau$ a scalar obeying the recursion,
\[
\theta_{\tau+1}=\theta_\tau +\eta \bp^2(\theta^\star-\theta_\tau).
\] 
If $\eta\leq 1/\bp^2$, all iterations satisfy $0\leq \theta_\tau \leq \theta^\star=\gamma$. On the other hand, the misfit in each iteration obeys 
\[
\tn{\y-f(\bteta_{\tau})}=\tn{\X(\bteta^\star-\bteta_{\tau})}=\bp \tn{\bteta^\star-\bteta_{\tau}}=\bp|\theta^\star-\theta_{\tau}|=\bp(\theta^\star-\theta_{\tau}).
\]
The last two identities imply $\tn{\y-f(\bteta_\tau)}+\bp \tn{\bteta_\tau}=\bp\gamma= \tn{\y}$ completing the proof of \eqref{total resi3}.
%For $\eta\leq 1/\bp^2$, we find that

%where the equality is satisfied if $\bteta$ is proportional to $\x_1$. Finally, $\Lc(0)=\alpha^2\mu/2$ and any global minimizer $\bteta$ satisfies $\x_1^T(\bteta-\bteta_\star)=0$ hence $\tn{\bteta-\bteta_0}=\tn{\bteta}\geq \tn{\bteta_\star}=\alpha$. Together, this implies no global minima within $\sqrt{2\Lc(0)/\mu}$.
%Consider the least-squares problem outlined in the proof of Theorem \ref{pl low bound} with $\mu= \bn^2$. Let $\beta=\x_1^T\bteta/\tn{\x_1}$ and $\bbeta=\beta\x_1/\tn{\x_1}$. Since $\X$ has orthonormal rows, we have that
%\[
%\tn{\y-f(\bteta)}=\tn{\X(\bteta_\star-\bteta)}\geq \tn{\X(\bteta_\star-\bbeta)}= \bn|\alpha-\beta|.
%\]
%Similarly, $\tn{\bteta-\bteta_0}\geq \tn{\bbeta}=\beta$. Combining, this yields $\tn{\y-f(\bteta)}+ \bn\tn{\bteta-\bteta_0}\geq \bn\alpha=\tn{\y}$.
% !TEX root = shortest.tex
\subsection{SGD proofs (Proof of Theorem \ref{SGDthm})}\label{app sgd}

\subsubsection{Roadmap of SGD proof}\label{Rmap}
We begin our SGD analysis by writing the SGD iterates in terms of the Jacobian matrix. To this aim define the matrix $\Jc(\bteta_\tau;\rng_\tau)$ which keeps the $\rng_\tau$-th row of $\Jc(\bteta_\tau)$ and sets the remaining rows to zero. We note that
\begin{align}
\label{Gform}
G(\vct{\theta}_\tau;\gamma_\tau)=\Jc(\bteta_\tau;\rng_\tau)^T\left(f(\vct{\theta}_\tau)-\vct{y}\right)\quad\text{and}\quad\E[\Jc(\bteta_\tau;\rng_\tau)]=\frac{1}{n}\Jc(\bteta_\tau).
\end{align}
Also define the matrix $\Cb(\vct{\theta}_\tau;\gamma_\tau)=\Jc(\bteta_{\tau+1},\bteta_\tau)\Jc(\bteta_\tau;\rng_\tau)^T\in\R^{n\times n}$ which can be thought of as a stochastic version of $\Cb(\vct{\theta}_\tau)$ obeying 
\begin{align*}
\E[\Cb(\vct{\theta}_\tau;\gamma_\tau)]=\frac{1}{n}\Cb(\vct{\theta}_\tau).
\end{align*}
Similar to the GD proof we begin by noting that the residual $\rb_\tau$ satisfies the recursion
\begin{align}
\label{SGDresrec}
\rb_{\tau+1}=&\rb_\tau-f(\vct{\theta}_\tau)+f(\vct{\theta}_{\tau+1}),\nn\\
\overset{(a)}{=}&\rb_\tau+\mathcal{J}(\vct{\theta}_{\tau+1},\vct{\theta}_\tau)\left(\vct{\theta}_{\tau+1}-\vct{\theta}_{\tau}\right),\nn\\
\overset{(b)}{=}&\rb_\tau-\eta\mathcal{J}(\vct{\theta}_{\tau+1},\vct{\theta}_\tau)G(\vct{\theta}_\tau;\gamma_\tau),\nn\\
\overset{(c)}{=}&\left(\mtx{I}_n-\eta\Cb(\vct{\theta}_\tau;\gamma_\tau)\right)\rb_\tau.
\end{align}
Here, (a) follows from the fundamental rule of calculus, (b) from the stochastic update rule, and (c) from combining the form of the stochastic gradient in \eqref{Gform} with the definition of $\Cb(\vct{\theta}_\tau;\gamma_\tau)$. 

Given that $\E[\Cb(\vct{\theta}_\tau;\gamma_\tau)]=\Cb(\vct{\theta}_\tau)/n$, similar to the GD proof we can show that under the two assumptions $\E[\Cb(\vct{\theta}_\tau;\gamma_\tau)]$ is positive-definite and thus with a sufficiently small learning rate $\eta$ this implies linear convergence of the expected residual via \eqref{SGDresrec} as long as $\bteta_i\in\Dc$. 

It is completely unclear if SGD stays inside a neighborhood around the initial model to ensure the on average convergence argument discussed above is useful. We will develop a novel martingale-based argument to show that SGD does indeed stay in this local neighborhood. We briefly discuss the intuition behind this approach here. Since SGD is inherently random, ideally, we would like to show that, a variant of \eqref{close} holds. Specifically, define
\begin{align}
\Vc_\tau=c\alpha \tn{\bteta_\tau-\bteta_0}+\tn{f(\bteta_\tau)-\y},\label{sgd v1}
\end{align}
we wish to show that $\Vc_\tau$ is bounded. One approach to do this is to show $\E[\Vc_{\tau}]\leq \Vc_{\tau-1}$ where the expectation is over the $\tau$'th SGD step given first $\tau-1$ steps. If this holds, $\Vc_\tau$ is a {\em{supermartingale}} with respect to the filtration generated by random SGD steps. This allows us to utilize martingale maximal inequality \cite{revuz2013continuous} which bounds the supremum of $\mathcal{V}_\tau$ via a Markov-like inequality
\[
\Pro(\sup_{\tau \geq 0} \Vc_\tau\geq C\E[\Vc_0])\leq \frac{1}{C}.
\]
This immediately establishes that $\Vc_\tau$ is uniformly bounded by $C\E[\Vc_0]$ and thus $\tn{\bteta_\tau-\bteta_0}\leq \frac{C\E[\Vc_0]}{c\alpha}$, hence $\bteta_\tau$ doesn't leave this neighborhood. However, unfortunately such a strategy does not work and a more nuanced argument is required. In particular, we need to overcome two challenges:
\begin{itemize}
\item The first challenge is that \eqref{sgd v1} is not a super martingale for reasonably large values of $c$. However, large values of $c$ are desirable as they yield a small convergence radius (e.g.~$c=1/4$ in \eqref{close}). We overcome this challenge by proposing a new potential function which tracks distances to multiple anchor points around $\bteta_0$ rather than only $\bteta_0$. Denoting these anchor points by $\{\vct{p}_\ell\}_{\ell=1}^K$, we utilize the potential
\[
\Vc_\tau:=\Vc(\vct{\theta}_\tau):=12\tn{f(\bteta_\tau)-\y}+\frac{\alpha }{K}\sum_{\ell=1}^K \tn{\bteta_\tau-\vct{p}_\ell}.\label{sgd v1}
\]
Figure \ref{fig sgd potent} provides a pictorial illustration of this potential function. 
\item The second challenge is that the optimization landscape is assumed to have nice properties only over a small neighborhood $\Dc$ around the initial point. Hence, the super martingale inequality $\E[\Vc_{\tau}]\leq \Vc_{\tau-1}$ applies only if the current and next iterate is over $\Dc$ and optimization essentially fails if we step outside. We overcome this by showing that the chance that SGD iterates exit this neighborhood is small using martingale stopping time arguments. The latter argument is inspired by/adapted from the work of Tan and Vershynin \cite{tan2017phase} in the context of phase retrieval. 

\end{itemize}

The outline of this Section is as follows. We show in Section \ref{avgconv} that from one SGD iterate to the next the misfit decreases in expectation. Then in Section \ref{avgdist} show that from one SGD iterate to the next the average distance to the chosen points $\{\vct{p}_\ell\}_{\ell=1}^K$ do not increase by a significant amount. We then combine the latter two results in Section \ref{supermartingale} to formally show that the potential $\mathcal{V}(\vct{\theta}_\tau)$ is indeed a supermartingale. Next, in section we deploy a martingale stopping time argument to show that with high probability the SGD iterates stay inside a neighborhood around the initial model. Finally, we put together all of these different arguments to complete the proof of Theorem \ref{SGDthm} in Section \ref{together}.

%\MS{add more description on supermartingale stuff}
\begin{figure}
\centering
\includegraphics[scale=1.1]{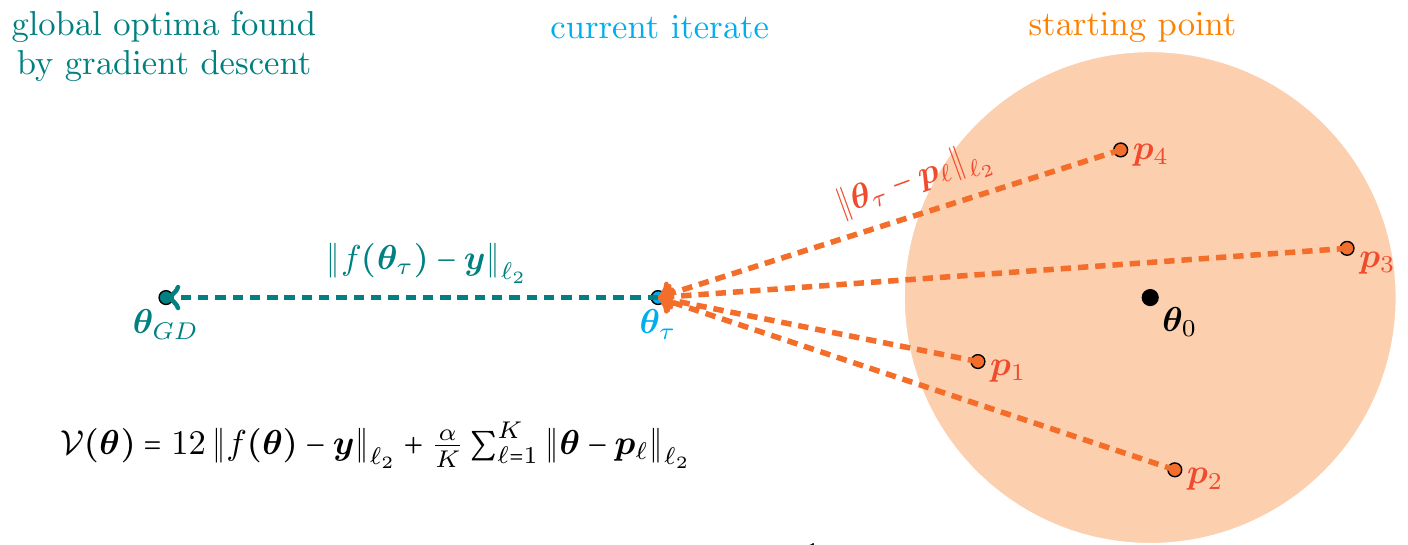}
\caption{SGD potential function is similar to the gradient descent potential \eqref{close}. It provides a balance between misfit error and distance to the initial point. However, to show that this potential is non-increasing, unlike gradient descent, we keep track of distances to multiple points around the initial point $\bteta_0$. This smooths out the potential function and guarantees the desired non-increasing property. Intuitively, the misfit ($\twonorm{f(\vct{\theta}_\tau)-\y}$) can be viewed as a proxy for distance to the global minima ($\twonorm{\vct{\theta}_\tau-\vct{\theta}_{GD}}$) as illustrated.}
\label{fig sgd potent}
\end{figure}

\subsubsection{Decrease of the expected misfit}\label{avgconv}
In this section we will show that under the assumption that SGD iterates always remain close to the initialization, the expected value of the norm of the residual will decrease in each iteration. Concretely, in this section we prove the following lemma.
\begin{lemma}\label{avedec}  Consider a point $\vct{\theta}\in\R^p$ and the result of a stochastic gradient update $\vct{\theta}^{+}:=\vct{\theta}-\eta G(\vct{\theta};\gamma)=\vct{\theta}-\eta\left(f(\vct{x}_\gamma;\vct{\theta})-y_{\gamma}\right)\nabla f(\vct{x}_\gamma;\vct{\theta})$ staring from $\vct{\theta}$ with the index $\gamma$ chosen uniformly at random from $\{1,2,\ldots, n\}$. Also consider the set
\begin{align}
\label{ball alpha}
\mathcal{B}(\nu)=\mathcal{B}\left(\vct{\theta}_0,\nu\frac{\twonorm{f(\vct{\theta}_0)-\vct{y}}}{\bn}\right)\bigcap\Bigg\{\vct{\theta}\in\R^p\Big| \twonorm{f(\vct{\theta})-\vct{y}}\le\frac{2\nu}{3}\twonorm{f(\vct{\theta}_0)-\vct{y}}\Bigg\},
\end{align}
Assume $\vct{\theta}\in\mathcal{D}':=\mathcal{B}(\nu/2)$ with $\nu$ a scalar obeying $\nu\ge 3$. Also assume the Jacobian associated with $f$ obeys Assumption \ref{wcond} over the set $\mathcal{D}:=\mathcal{B}(\nu)$ and the rows of the Jacobian have bounded Euclidean norm over this set, that is
\begin{align*}
\underset{i}{\max}\text{ }\twonorm{\mathcal{J}_i(\vct{\theta})}\le B\quad\text{for all}\quad\vct{\theta}\in\mathcal{D}:=\mathcal{B}(\nu).% \text{ such that } \twonorm{\vct{\theta}-\vct{\theta}_0}\le R.
\end{align*} 
Also assume
\begin{itemize}
\item Assumption \ref{spert}(a) holds over $\Dc$ and $\eta\leq \frac{\bn^2}{2\bp^2\bz^2 }$.
\item Assumption \ref{spert}(b) holds over $\Dc$ and $\eta\leq \frac{1}{2 \bp\bz}\cdot\min\left(\frac{\bn^2}{B\bp},\frac{ 3\bn^2}{\nu \el \tn{f(\vct{\theta}_0)-\y}}\right)$.
\end{itemize}
Then, 
\begin{align}
\E[\tn{f(\vct{\theta}^{+})-\y}]\le& \tn{f(\vct{\theta})-\y}-\frac{\eta}{4n}\frac{\tn{\Jc^T(\bteta)\left(f(\vct{\theta})-\y\right)}^2}{\tn{f(\vct{\theta})-\y}},\label{updateineq}\\
\E\big[\tn{f(\vct{\theta}^{+})-\y}^2\big]\le&\left(1-\frac{\eta\bn^2}{2n}\right)^\tau\twonorm{f(\vct{\theta})-\vct{y}}^2\label{updateineq2}.
\end{align}
\end{lemma}
For simplicity of exposition of the proof of this lemma we define $r(\vct{\theta})=f(\vct{\theta})-\y$ and $r(\vct{\theta}^{+})=f(\vct{\theta}^{+})-\y$. We prove the lemma in three steps.
\begin{itemize}
\item \textbf{Step I:} We show that as long as $\eta\le\frac{1}{\bp B}$, then $\vct{\theta}^{+}\in\mathcal{D}$.
\item \textbf{Step II:} We prove that the matrix $\mtx{C}(\vct{\theta}):=\Jc(\vct{\theta}^{+},\vct{\theta})\Jc^T(\vct{\theta})$ obeys
\begin{align}
\label{sCineq}
\mtx{C}(\vct{\theta})\succeq \frac{1}{2}\Jc(\vct{\theta})\Jc(\vct{\theta})^T.
\end{align}
\item \textbf{Step III:} We use Step I and II to show the inequalities \eqref{updateineq} and \eqref{updateineq2} which are equivalent to
\begin{align}
\E[\tn{\rb(\vct{\theta}^{+})}]\le& \tn{\rb(\vct{\theta})}-\frac{\eta}{4n}\frac{\tn{\Jc^T(\bteta)\left(f(\vct{\theta})-\y\right)}^2}{\tn{f(\vct{\theta})-\y}},\label{rSGD1}\\
\E[\tn{\rb(\vct{\theta}^{+})}]\le& \left(1-\frac{\eta\bn^2}{2n}\right)\tn{\rb(\vct{\theta})}^2\label{rSGD2}.
\end{align}
\end{itemize}
\paragraph*{Step I:} 
We begin this step by noting that
\begin{align}
\label{SGDbnd}
\tn{G(\vct{\theta};\gamma)}\leq \max_{1\leq i\leq n}\tn{\nabla f(\x_i;\bteta)}\abs{f(\vct{x}_i;\vct{\theta})-\y_i}\leq \bz \tn{\rb(\vct{\theta})}.
\end{align}
Using this inequality we can conclude that
\begin{align}
\label{diffplus}
\twonorm{\vct{\theta}^{+}-\vct{\theta}}\le\eta \twonorm{G(\vct{\theta};\gamma)}\overset{(a)}{\le} \eta B\twonorm{\rb(\vct{\theta})}\overset{(b)}{\le} \frac{\eta \nu B}{3}\twonorm{f(\vct{\theta}_0)-\y}\overset{(c)}{\le}\frac{1}{3}\frac{\nu}{\bp}\twonorm{f(\vct{\theta}_0)-\y}.
\end{align}
Here, (a) follows from \eqref{SGDbnd}, (b) from the fact that $\vct{\theta}\in\mathcal{D}':=\mathcal{B}(\nu/2)$, and (c) from $\eta\le \frac{1}{\bp B}$. Furthermore, the simple fact that $\bn\le \bp$ implies that
\begin{align}
\label{diffplus2}
\twonorm{\vct{\theta}^{+}-\vct{\theta}}\le\frac{1}{3}\frac{\nu}{\bn}\twonorm{f(\vct{\theta}_0)-\y}.
\end{align}
Next note that
\begin{align}
\label{mytemp}
\twonorm{r(\vct{\theta}^{+})}\le& \twonorm{r(\vct{\theta})}+\twonorm{r(\vct{\theta}^{+})-r(\vct{\theta})},\nonumber\\
=&\twonorm{r(\vct{\theta})}+\twonorm{\Jc(\vct{\theta}^{+},\vct{\theta})\left(\vct{\theta}^{+}-\vct{\theta}\right)},\nonumber\\
\le&\twonorm{r(\vct{\theta})}+\opnorm{\Jc(\vct{\theta}^{+},\vct{\theta})}\twonorm{\vct{\theta}^{+}-\vct{\theta}},\nonumber\\
\overset{(a)}{\le}&\twonorm{r(\vct{\theta})}+\frac{\nu}{3}\twonorm{f(\vct{\theta}_0)-\y},\nonumber\\
\overset{(b)}{\le}&\frac{2}{3}\nu\twonorm{f(\vct{\theta}_0)-\y}.
\end{align}
Here, (a) follows from \eqref{diffplus} and the fact that $\opnorm{\Jc(\vct{\theta}^{+},\vct{\theta})}\le \bp$ and (b) follows from the fact that $\vct{\theta}\in\mathcal{D}':=\mathcal{B}(\nu/2)$. Combining \eqref{diffplus2} and \eqref{mytemp} we conclude that $\vct{\theta}^{+}\in\mathcal{D}:=\mathcal{B}(\nu)$.
\paragraph*{Step II:}
The proof of \eqref{sCineq} is very similar to the proof of Lemma \ref{small dev} with $\lambda=1/2$. In particular, under Assumption \ref{spert}(a) the exact same argument yields \eqref{sCineq}. To show the result under Assumption \ref{spert}(b) we combine \eqref{avgsmooth} from the proof of Lemma \ref{small dev}, \eqref{SGDbnd}, and $\vct{\theta}\in\mathcal{D}'=\mathcal{B}(\nu/2)$ to conclude that
\begin{align*}
\opnorm{\Jc(\bteta^{+},\bteta)-\Jc(\bteta)}\leq \frac{\el}{2}\twonorm{\bteta^{+}-\bteta}\le \frac{\eta \bz L}{2} \tn{\rb(\vct{\theta})}\le \frac{\eta \nu\bz L}{3} \tn{f(\vct{\theta}_0)-\y}\le \frac{\bn^2}{2\bp},
%\bz \el\cs\tn{\rb_0}/3\leq \bn^2/2 \bp,
\end{align*}
where in the last inequality we use the fact that $\eta\leq \frac{3}{2}\frac{ \bn^2}{\cs\bp\bz  \el\tn{\rb_0}}$. The remainder of the proof of \eqref{sCineq} is exactly the same as the proof of Lemma \ref{small dev}.
\paragraph*{Step III:}
From the arguments of Steps I and II we know that 
\begin{itemize}
\item[(i)] $\bteta^{+}\in \Dc$, 
\item[(ii)] $\|\Cb(\vct{\theta})\|\leq \bp^2$ and $\|\Jc(\bteta^{+},\bteta)\|\leq \bp$, 
\item[(iii)] $\Cb(\vct{\theta})\succeq \frac{1}{2}\Jc(\bteta)\Jc(\bteta)^T$.
\end{itemize}
% and $\Jc($ Set $\w_i=\bteta_i-\bteta_0$. We first show that $\rb_i$ is small. Note that
%\begin{align}
%\tn{\rb_0-\rb_i}=\tn{f(\bteta_i)-f(\bteta_0)}=\tn{\Jc(\bteta_{i},\bteta_{0})(\bteta_{i}-\bteta_{0})}\leq \bp\tn{\w_i}.\label{rw bound}
%\end{align}
%Consequently, $\tn{\rb_i}\leq \tn{\rb_0}+ \bp\tn{\w_i}\leq \tn{\rb_0}+\cs \bpR/2\leq \cs \bpR$. With this, for $\eta\leq \eta_a\leq \frac{1}{\bz \bp}$, we find that $\bteta_{i+1}\in\Dc$ as follows via \eqref{sgd eq}
%\begin{align}
%\tn{\bteta_{i+1}-\bteta_i}=\eta |\rb_{\rng_i}| \tn{\nabla f(\x_{\rng_i},\bteta_i)}\leq \eta\tn{\rb_i}\bz\leq \frac{1}{\bz \bp}\cs \bp\bz R=  \cs R.\label{theta ineq}
%\end{align} %Jacobian has a uniformly bounded spectral norm by $ \bp$ (over the domain)
%Additionally, since $\eta\leq \eta_a\leq\frac{ \bn^2}{\cs \bz \el R \bp^2} $, we have
%\[
%\tn{\bteta_{i+1}-\bteta_i}\leq \eta\tn{\rb_i}\bz\leq \eta \cs \bp\bz R\leq \frac{ \bn^2}{2\el \bp}
%\]
%which implies via Lemma \ref{upp bound J} that 
%\[
%\Jc(\bteta_{i+1},\bteta_{i})\Jc(\bteta_{i+1},\bteta_{i})^T\preceq 2\Jc(\bteta_{i})\Jc(\bteta_{i})^T.
%\]
Using (ii) $\Jc(\bteta^{+},\bteta)\Jc(\bteta^{+},\bteta)^T\preceq \bp^2\Iden_n,$ so that
\[
\Cb(\vct{\theta};\gamma)^T\Cb(\vct{\theta};\gamma)\preceq \bp^2\Jc(\bteta;\rng)\Jc(\bteta;\rng)^T.
\]
Furthermore, $\Jc(\bteta;\rng)\Jc(\bteta;\rng)^T$ is a diagonal matrix with a single nonzero entry which is bounded by $\bz^2$. Thus,
\begin{align}
\E\big[\Cb(\vct{\theta};\gamma)^T\Cb(\vct{\theta};\gamma)\big]\preceq \frac{\bp^2\bz^2}{n}.\label{r2 part}
\end{align}
Also the fact that $\E[\Cb(\vct{\theta};\gamma)]=\frac{1}{n}\mtx{C}(\vct{\theta})$ (also noted in Section \ref{Rmap}) together with (iii) allows us to conclude that
\begin{align}
\label{lowCeq}
\E[\Cb(\vct{\theta};\gamma)]&=\frac{1}{n}\Cb(\vct{\theta})\succeq \frac{1}{2n}\Jc(\bteta)\Jc(\bteta)^T.%\Jc(\bteta_{i})\Jc(\bteta_{i},\rng_i)^T+(\Jc(\bteta_{i+1},\bteta_{i})-\Jc(\bteta_{i}))\Jc(\bteta_{i},\rng_i)^T:=\M_1+\M_2.
\end{align}
%Observe that, $\E[\M_1]=\Jc(\bteta_{i})\Jc(\bteta_{i})^T/n\succeq \bn^2/n$. Similarly, utilizing first inequality of \eqref{theta ineq}
%\begin{align}
%|\rb_i^T\M_2\rb_i|&\leq |\rb_{\rng_i}|\tn{\Jc(\bteta_{i},\rng_i)}\tn{\rb_i}\|\Jc(\bteta_{i+1},\bteta_{i})-\Jc(\bteta_{i})\|\\
%&\leq |\rb_{\rng_i}|\bz\tn{\rb_i}\el\tn{\bteta_{i+1}-\bteta_i}\\
%&\leq\eta \rb_{\rng_i}^2\bz^2 \el\tn{\rb_i}.
%\end{align}
%This implies 
%\[
%\E[|\rb_i^T\M_2\rb_i|]\leq \E[\rb_{\rng_i}^2 \bp^2\eta \el\tn{\rb_i}]\leq \eta\bz^2 \el\tn{\rb_i}^3/n.
%\]
%and using $\tn{\rb_i}\leq \cs \bpR$
Using the latter two inequalities allows us to conclude
\begin{align}
\label{secthird}
\eta\rb(\vct{\theta})^T\E[\Cb(\vct{\theta};\gamma)^T\Cb(\vct{\theta};\gamma)]\rbt\overset{(a)}{\le}& \frac{\eta\bp^2\bz^2}{n}\tn{\rbt}^2,\nn\\
\overset{(b)}{\le}& \frac{\bn^2}{2n}\tn{\rbt}^2,\nn\\
\overset{(c)}{\le}& \frac{1}{2n}\rbt^T \Jc(\bteta)\Jc(\bteta)^T\rbt,\nn\\
\overset{(d)}{\le}& \rbt^T\E[\Cb(\vct{\theta};\gamma)]\rbt.
\end{align}
Here, (a) follows from \eqref{r2 part}, (b) from the fact that the step size obeys $\eta\le \frac{ \bn^2}{2\bz^2 \bp^2}$, (c) from $\smn{\mathcal{J}(\vct{\theta})}\ge \bn$, and (d) from \eqref{lowCeq}.
%where we used $\tn{\bteta_{i+1}-\bteta_i}=\eta\Jc(\bteta_{i},\rng_i)^T\rb_i$.
These inequalities allow us to conclude
\begin{align}
\E[\tn{\rb(\vct{\theta}^{+})}^2]&\overset{(a)}{\le} \rbt^T\left(\Iden_n-2\eta\E[\mtx{C}(\vct{\theta};\gamma)]+\eta^2\E[\mtx{C}(\vct{\theta};\gamma)^T\mtx{C}(\vct{\theta};\gamma)]\right)\rbt,\nn\\%(2 \Jc(\bteta_{i})\Jc(\bteta_{i})^T+2\M_2-\eta\E[\bar{\Jc}_i^T\bar{\Jc}_i)])\rb_i\\%\tn{\rb_i}^2(1-\frac{\eta}{n}(2 \bn^2-\eta \bz^2(2\el\tn{\rb_i}+ \bp^2)))\\
&\overset{(b)}{\le} \rbt^T\left(\Iden_n-\eta \E[\mtx{C}(\vct{\theta};\gamma)]\right)\rbt,\nn\\
&\overset{(c)}{\le} \rbt^T\left(\Iden_n-\frac{\eta}{2n}\Jc(\bteta)\Jc(\bteta)^T\right)\rbt,\nn\\
&= \tn{\rbt}^2-\frac{\eta}{2n}\tn{\Jc(\bteta)^T\rbt}^2,\nn\\
&\overset{(d)}{\le} \left(\twonorm{\rbt}-\frac{\eta}{4n}\frac{\twonorm{\Jc(\bteta)^T\rbt}^2}{\twonorm{\rbt}}\right)^2.
\label{rsq eq}
%&\leq \tn{\rb_i}^2(1-\frac{\eta}{n} \bn^2).\label{rsq eq}
\end{align}
Here, (a) follows from the calculation in \eqref{SGDresrec} applied to $\rbt$ and $\rb(\vct{\theta}^{+})$, (b) from \eqref{secthird}, (c) from \eqref{lowCeq}, and (d) from completing the square. Finally, note that using the upper bound on the spectrum of the Jacobian and the fact that $\eta\le \frac{\bn^2}{2\bp^2\bz^2 }\le \frac{1}{2\bp^2}$\footnote{Note that $\bn\le B$.} we have
\begin{align*}
\frac{\eta}{4n}\twonorm{\Jc(\bteta)^T\rbt}^2\le\eta\frac{\bp^2}{4n}\twonorm{\rbt}^2\le \twonorm{\rbt}^2,
\end{align*}
so that the term inside the parentheses of right-hand sided of \eqref{rsq eq} is positive. Consequently, combining Jensen's inequality with the square root of both sides of \eqref{rsq eq} yields
\[
\E[\tn{\rb(\vct{\theta}^{+})}]\leq \sqrt{\E[\tn{\rb(\vct{\theta}^{+})}^2]}\le \twonorm{\rbt}-\frac{\eta}{4n}\frac{\twonorm{\Jc(\bteta)^T\rbt}^2}{\twonorm{\rbt}},
\]
%\MS{3.11 to 3.12 is not correct but probably does not matter}
%As a sanity check, observing $\eta \bn^2\leq 1$, the right hand side is always nonnegative. Now, using the fact that $\sqrt{1-x}\leq 1-x/2$, we find
%\[
%\E[\tn{\rb_{i+1}}]\leq \sqrt{\E[\tn{\rb_{i+1}}^2]}\leq \tn{\rb_i}(1-\frac{\eta}{2n} \bn^2)
%\]
concluding the proof \eqref{rSGD1}. To prove \eqref{rSGD2} we use the penultimate inequality from \eqref{rsq eq} together with the fact that $\sigma_{\min}\left(\Jc(\vct{\theta})\right)\ge \bn$ to conclude that
\begin{align*}
\E[\tn{\rb(\vct{\theta}^{+})}^2]\le\tn{\rbt}^2-\frac{\eta}{2n}\tn{\Jc(\bteta)^T\rbt}^2\le \left(1-\frac{\eta\bn^2}{2n}\right)\tn{\rbt}^2,
\end{align*}
completing the proof of \eqref{rSGD2}.

\subsubsection{Bounding the increase of expected average distance to anchor points}\label{avgdist}
In this section we will show that under the assumption that SGD iterates always remain close to the initialization, the expected value of the average distance to the anchor points will not significantly increase in each iteration. Specifically, the anchor points we pick are an  $\epsilon$ cover of the neighborhood of the initialization denoted by $\mathcal{P}=\{\vct{p}_1,\vct{p}_2,\ldots,\vct{p}_K\}$. and we monitor the following average distance 
\begin{align}
\label{pdist}
d_{\mathcal{P}}(\vct{\theta}):=\frac{1}{K}\sum_{\ell=1}^K \twonorm{\vct{\theta}-\vct{p}_\ell}.
\end{align}
Concretely, in this section we prove the following lemma.
\begin{lemma}\label{lem param}  Consider the setting and assumptions of Lemma \ref{avedec}. Also assume $\eta\le \frac{3}{\nu B^2}$. Furthermore, fix $K\ge\sqrt{n}\frac{\bp}{\bn}$ and let $\mathcal{P}=\{\vct{p}_1,\vct{p}_2,\ldots,\vct{p}_K\}$ be an $\epsilon:=\frac{\twonorm{f(\vct{\theta}_0)-\y}}{\bn}$ packing of a ball of radius $\cp:=1.25\left(\frac{\bp}{\bn}\right)^{1/p}\frac{\twonorm{f(\vct{\theta}_0)-\y}}{\bn}$ around $\vct{\theta}_0$ so that pairwise distances in this set are at least $\epsilon$ apart.\footnote{Classical results guarantee that, we can find a $(R_p/\eps)^p$ $\epsilon$-packing of an $R_p$-ball. In our case using the fact that $p\ge n$ this reduces to $\left(1.25\left(\frac{\bp}{\bn}\right)^{1/p}\right)^p\ge \left(\frac{5}{4}\right)^p\frac{\bp}{\bn}\ge\sqrt{n}\frac{\bp}{\bn}\ge K$ so that such a packing is possible.} Then, for $d_{\mathcal{P}}$ given by \eqref{pdist} we have
\begin{align}
\E[d_{\mathcal{P}}(\vct{\theta}^{+})]\leq d_{\mathcal{P}}(\vct{\theta})+\frac{3\eta}{n}\tn{\Jc^T(\bteta)\left(f(\vct{\theta})-\y\right)}.\label{main bound d}
\end{align}
\end{lemma}
%. $\tn{\bteta_i-\bteta_0}$. Setting $\w_i=\bteta_i-\bteta_0$
For simplicity of exposition of the proof of this lemma we define $r(\vct{\theta})=f(\vct{\theta})-\y$ and $r(\vct{\theta}^{+})=f(\vct{\theta}^{+})-\y$. 
We start the proof by monitoring the evolution of the parameter vector with respect to a particular reference point $\vct{p}\in \mathcal{P}$.  In particular define $\w={\bteta-\vct{p}}$ and note that $\w^{+}=\vct{\theta}-\vct{p}=\w-\eta \Jc^T(\bteta;\rng)\rb(\vct{\theta})$ and $\E[\Jc(\bteta;\rng)]=\Jc(\bteta)/n$. Thus, 
\begin{align}
\E[\tn{\w^{+}}^2]&=\E[\tn{\w}^2-2\eta\w^T\Jc^T(\bteta;\rng)\rb(\vct{\theta})+\eta^2\tn{\Jc^T(\bteta;\rng)\rb(\vct{\theta})}^2],\nn\\
&=\tn{\w}^2-2\frac{\eta}{n}\w^T\Jc^T(\bteta)\rb(\vct{\theta})+\eta^2\E[\tn{\Jc^T(\bteta;\rng)\rb(\vct{\theta})}^2],\nn\\
&\leq \tn{\w}^2+\frac{2}{n}\eta \tn{\w}\tn{\Jc^T(\bteta)\rb(\bteta)}+\frac{\eta^2}{n}\bz^2\tn{\rb(\bteta)}^2,\label{leqeq_tmp}\\
&\leq \left(\tn{\w}+\frac{2\eta}{n}\tn{\Jc^T(\bteta)\rb(\vct{\theta})}\right)^2+\frac{\eta}{n}\left( \eta\bz^2\tn{\rb(\vct{\theta})}^2-2\tn{\Jc^T(\bteta)\rb(\vct{\theta})}\tn{\w}\right).\label{leqeq}
\end{align}
Using \eqref{leqeq_tmp} and $\|\Jc^T(\vct{\theta})\|\leq \bp$, we also have
\begin{align}
\E[\tn{\w^{+}}]\le& \sqrt{\E[\tn{\w^{+}}^2]}\nn,\\
\leq& \left(\tn{\w}^2+\frac{2}{n}\eta \tn{\w}\tn{\Jc^T(\bteta)\rb(\bteta)}+\frac{\eta^2}{n}\bz^2\tn{\rb(\bteta)}^2\right)^{1/2}\nn\\
\le&\left(\tn{\w}^2+\frac{2\bp }{n}\eta \tn{\w}\tn{\rb(\bteta)}+\frac{\eta^2}{n}\bp^2\tn{\rb(\bteta)}^2\right)^{1/2}\nn\\
\le&\tn{\w}+\frac{\eta}{\sqrt{n}}\bp\twonorm{\rb(\vct{\theta})}.
\label{bnd triv}
\end{align}
%We also have
%\begin{align}
%\E[\tn{\w^{+}}]\le& \E[\tn{\w}]+\E[\twonorm{\w^{+}-\w}]\nn,\\
%=&\tn{\w}+\eta\E\big[\twonorm{\mathcal{J}^T(\vct{\theta};\gamma)\rb(\vct{\theta})}\big]\nn,\\
%=&\tn{\w}+\eta\sqrt{\E\big[\twonorm{\mathcal{J}^T(\vct{\theta};\gamma)\rb(\vct{\theta})}^2\big]},\nn\\
%=&\tn{\w}+\frac{\eta}{\sqrt{n}}\twonorm{\Jc^T(\vct{\theta})\rb(\vct{\theta})},\nn\\
%\le&\tn{\w}+\frac{\eta}{\sqrt{n}}\bp\twonorm{\rb(\vct{\theta})}.
%\label{bnd triv}
%\end{align}
We also prove the following simple lemma.
%\begin{align}
%\E[\tn{\w_{i+1}}]\leq \sqrt{\E[\tn{\w_{i+1}}^2]}\leq  \tn{\w_i}+\frac{\eta}{\sqrt{n}} \bp\tn{\rb_i}.\label{bnd triv}
%\end{align}
%\MS{There are some terms missing in above when you compare to 3.17, it does follow from 3.16 so don't get the point of 3.17 seems to be used later though}
\begin{lemma}\label{simpclaim}If $\tn{\w}\geq \epsilon/2$, then $\eta\bz^2\tn{\rb(\vct{\theta})}^2-2\tn{\Jc(\bteta)\rb(\vct{\theta})}\tn{\w}\leq 0$. 
\end{lemma}
\begin{proof}
%To establish this claim, we need $2\tn{\w_i} \bn\geq \eta \bz^2\tn{\rb_i}$ so that
%\[
%\eta\bz^2\tn{\rb_i}^2\leq 2\tn{\rb_i}\smn{\Jc(\bteta_i)}\tn{\w_i}\leq 2\tn{\Jc(\bteta_i)\rb_i}\tn{\w_i}
%\]
%Hence, we wish to show $\tn{\rb_i}\leq \frac{2 \bn}{\eta\bz^2}\tn{\w_i}$. Using $\tn{\vb-\bteta_0}\leq \cp R$, this implies 
%\[
%\tn{\bteta_i-\bteta_0}\leq \tn{\bteta_i-\vb}+\tn{\vb-\bteta_0}\leq (2\cp+1) \tn{\w_i}\leq 3\cp \tn{\w_i}.
%\]
%Using $\tn{\bteta_i-\bteta_0}\leq \|\Jc(\bteta_{i},\bteta_0)\|\tn{\w_i}$, this yields
%\[
%\tn{\rb_0-\rb_i}\leq 3\cp \bp\tn{\w_i}.%\tn{\bteta_i-\bteta_0}\leq \bp(\tn{\bteta_i-\vb}+\tn{\vb-\bteta_0})\leq \bp(R/2+2R)=5 \bp/2.
%\]
%Secondly, note that $\tn{\rb_0}\leq \bpR\leq 2 \bp\tn{\w_i}\leq \cp \bp\tn{\w_i}$. Merging this with the fact that $\eta\leq \frac{ \bn}{2\cp \bz^2 \bp}$, we find
Using the assumption $\bteta\in\mathcal{B}(\cs/2)$, we have
\begin{align}
\tn{\rb(\vct{\theta})}\leq \frac{\cs}{3}\twonorm{f(\vct{\theta}_0)-\y}.
\end{align}
Consequently, using $\eta\leq \frac{3}{\bz^2\cs}$ and $\sigma_{\min}\left(\Jc^T(\bteta)\rb(\bteta)\right)\ge\bn\tn{\rb(\vct{\theta})}$, we have
\[
\eta B^2\tn{\rb(\vct{\theta})}\leq\frac{\eta B^2 \cs}{3}\tn{f(\vct{\theta}_0)-\y} \leq \tn{f(\vct{\theta}_0)-\y}= 2\bn\frac{\epsilon}{2}\leq 2 \bn\tn{\w}.
\]
\end{proof}
Hence, the lemma above combined with \eqref{leqeq} implies that if $\tn{\w}\geq \eps/2$
\begin{align}
\E[\tn{\w^{+}}]\leq \sqrt{\E[\tn{\w^{+}}^2]}\leq \tn{\w}+\frac{2\eta}{n}\tn{\Jc^T(\bteta)\rb(\bteta)}.\label{clclcl}
\end{align}
Combining \eqref{bnd triv} and \eqref{clclcl}, we conclude that
%
%Substituting this into \eqref{leqeq} yields
%\[
%\E[\tn{\w_{i+1}}]\leq \sqrt{\E[\tn{\w_{i+1}}^2]}\leq \tn{\w_i}+\frac{2\eta}{n} \bp\tn{\rb_i}.
%\]
\begin{align}
\E[\tn{\w^{+}}]\leq \begin{cases}\tn{\w}+\frac{2\eta}{n}\tn{\Jc^T(\bteta)\rb(\bteta)}\quad&\text{if}\quad\tn{\w}\geq \frac{\epsilon}{2}\\ \tn{\w}+\frac{\eta}{\sqrt{n}} \bp\tn{\rb(\bteta)}\quad&\text{otherwise}\end{cases}\label{either or}.
\end{align}
Now define $\vct{w}_\ell:=\vct{\theta}-\vct{p}_\ell$ as the difference between the parameter and the $\ell$th anchor point. Now observe that out of all vectors $\{\vct{w}_1,\vct{w}_2,\ldots,\vct{w}_K\}$, at most one of them can satisfy $\tn{\w_\ell}\leq \frac{\epsilon}{2}$ due to the packing property. Specifically, if $\tn{\w_\ell}\leq \frac{\epsilon}{2}$, then for any $\widetilde{\ell}\neq \ell$ we have
\[
\tn{\w_{\widetilde{\ell}}}=\tn{\vct{p}_\ell-\vct{p}_{\widetilde{\ell}}}-\tn{\w_\ell}\ge \frac{\epsilon}{2}.
\]
Hence, at least $K-1$ of $\w_\ell$ satisfies first line and at most $1$ satisfies the second line of \eqref{either or}. Next, note that
\[
\sqrt{n} \bp\tn{\rb(\vct{\theta})}\leq \sqrt{n}\frac{ \bp}{ \bn}\tn{\Jc^T(\bteta)\rb(\vct{\theta})}\leq K\tn{\Jc^T(\bteta)\rb(\vct{\theta})}.
\] 
Using the latter two identities we conclude that
\begin{align*}
\E\Big[\sum_{\ell=1}^K \tn{\w_\ell^{+}}\Big]&\leq \sum_{\ell=1}^K \tn{\w_\ell}+(K-1)\frac{2\eta}{n}\tn{\Jc^T(\bteta)\rb(\bteta)}+\frac{\eta}{\sqrt{n}} \bp\tn{\rb(\bteta)}\\
&\leq \sum_{\ell=1}^K \tn{\w_\ell}+\frac{3K\eta}{n}\tn{\Jc^T(\bteta)\rb(\bteta)}.
\end{align*}
Dividing both sides by $K$ completes the proof of \eqref{main bound d}.
% This implies that if $\eta \bp\tn{\rb_i}\leq 2\tn{\w_i}$, we have the recursion $\E[\tn{\w_{i+1}}]\leq \tn{\w_i}+\frac{2\eta}{n} \bp\tn{\rb_i}$. This relation holds whenever $\tn{\w_i}\geq \tn{\rb_0}/ \bp$. The proof is obtained as follows
%Hence, $\eta \bp\tn{\rb_i}\leq \eta \bp\tn{\rb_0}+\eta \bp^2\tn{\w_i}\leq \tn{\w_i}+\eta \bp\tn{\rb_0}\leq 2\tn{\w_i}$.
\subsubsection{Shortest path potential is a supermartingale}\label{supermartingale}
In this section we show that the shortest path potential
\begin{align}
\Vc_\tau:=\Vc(\vct{\theta}_\tau):=12\tn{f(\bteta_\tau)-\y}+\frac{\alpha }{K}\sum_{\ell=1}^K \tn{\bteta_\tau-\vct{p}_\ell}.\label{spp}
\end{align}
 is a supermartingale. Specifically we prove the following lemma.
\begin{lemma}\label{tot res}
Consider a nonlinear least-squares optimization problem of the form $\underset{\vct{\theta}\in\R^p}{\min}\text{ }\mathcal{L}(\vct{\theta}):=\frac{1}{2}\twonorm{f(\vct{\theta})-\vct{y}}^2$, with $f:\R^p\mapsto \R^n$ and $\vct{y}\in\R^n$. Suppose the Jacobian mapping associated with $f$ obeys Assumption \ref{wcond} over a ball $\mathcal{D}$ of radius $R:=\nu\frac{\twonorm{f(\vct{\theta}_0)-\vct{y}}}{\bn}$ around a point $\vct{\theta}_0\in\R^p$ with $\nu$ a scalar obeying $\nu\ge 3$. Also consider the set
\begin{align}
\label{ball alpha2}
\mathcal{B}(\nu)=\mathcal{B}\left(\vct{\theta}_0,\nu\frac{\twonorm{f(\vct{\theta}_0)-\vct{y}}}{\bn}\right)\bigcap\Bigg\{\vct{\theta}\in\R^p\Big| \twonorm{f(\vct{\theta})-\vct{y}}\le\frac{2\nu}{3}\twonorm{f(\vct{\theta}_0)-\vct{y}}\Bigg\}.
\end{align}
Also assume the rows of the Jacobian have bounded Euclidean norm over this ball, that is
\begin{align*}
\underset{i}{\max}\text{ }\twonorm{\mathcal{J}_i(\vct{\theta})}\le B\quad\text{for all}\quad\vct{\theta}\in\mathcal{D}.% \text{ such that } \twonorm{\vct{\theta}-\vct{\theta}_0}\le R.
\end{align*} 
Furthermore, suppose one of the following statements is valid.
\begin{itemize}
\item Assumption \ref{spert} (a) holds over $\mathcal{D}$ and set $\eta\leq \frac{\bn^2}{\nu\bp^2B^2}$.
\item Assumption \ref{spert} (b) holds over $\mathcal{D}$ and set $\eta\leq \frac{\bn^2}{\nu\bp^2B^2+\nu\bp BL\twonorm{f(\vct{\theta}_0)-\vct{y}}}$.
\end{itemize}
Fix $K\ge\sqrt{n}\frac{\bp}{\bn}$ and let $\{\vct{p}_\ell\}_{\ell=1}^K$ be an $\epsilon:=\frac{\twonorm{f(\vct{\theta}_0)-\y}}{\alpha}$ packing of a ball of radius $\cp:=1.25\left(\frac{\bp}{\bn}\right)^{1/p}\frac{\twonorm{f(\vct{\theta}_0)-\y}}{\bn}$ around $\vct{\theta}_0$  so that pairwise distances in this set are at least $\epsilon$ and define the potential $\mathcal{V}(\vct{\theta})$ associated with this packing per \eqref{spp}. Starting from $\vct{\theta}_0$ we run stochastic gradient updates of the form \eqref{sgd eq}. Then, $\mathcal{V}(\vct{\theta}_0)\le 14\left(\frac{\bp}{\bn}\right)^{1/p}\twonorm{f(\vct{\theta}_0)-\y}$. Furthermore, if $\vct{\theta}_\tau\in\mathcal{B}(\nu/2)$, then $\E[\mathcal{V}(\vct{\theta}_{\tau+1})]\le \mathcal{V}(\vct{\theta}_{\tau})$.

%Also define the stopping time $T=\min \{\tau: \vct{\theta}_\tau\notin\mathcal{B}(\vct{\theta}_0,R/2)\}$. 
\end{lemma}
To bound $\mathcal{V}(\vct{\theta}_0)$ not that each anchor point in the packing obeys $\tn{\vct{p}_\ell-\bteta_0}\leq 1.25\left(\frac{\bp}{\bn}\right)^{1/p} \frac{\twonorm{f(\vct{\theta}_0)-\y}}{\bn}$, we have
\begin{align*}
\mathcal{V}(\bteta_0)\leq& 12\twonorm{f(\vct{\theta}_0)-\y}+\frac{\bn}{K}\sum_{i=1}^K \tn{\bteta_0-\vct{p}_\ell}\leq 12\twonorm{f(\vct{\theta}_0)-\y}+1.25\left(\frac{\bp}{\bn}\right)^{1/p}\twonorm{f(\vct{\theta}_0)-\y}\\
\leq& 14\left(\frac{\bp}{\bn}\right)^{1/p} \twonorm{f(\vct{\theta}_0)-\y}.
\end{align*}
Turning our attention to the supermartingale property, define $\rb_\tau=f(\vct{\theta}_\tau)-\y$ and note that when $\bteta_\tau\in \Bc(\cs/2)$, by Lemmas \ref{lem param} and \ref{simpclaim} we have
\begin{align*}
\E[\tn{\rb_{\tau+1}}]\le& \tn{\rb_\tau}-\frac{\eta}{4n}\frac{\tn{\Jc(\bteta_\tau)\rb_\tau}^2}{\tn{\rb_\tau}},\\
\E[d_{\mathcal{P}}(\vct{\theta}_{\tau+1})]\le& d(\bteta_\tau)+\frac{3\eta}{n}\tn{\Jc(\bteta_\tau)\rb_\tau}.
\end{align*}
Summing these two identities with a scaling of the first inequality by $12$ and the second one by $\bn$, we obtain
\begin{align*}
\E[\mathcal{V}(\bteta_{\tau+1})]-\mathcal{V}(\bteta_\tau)&\le 12\left(\E[\tn{\rb_{\tau+1}}]-\tn{\rb_\tau}\right)+\bn\left(\E[d_{\mathcal{P}}(\bteta_{\tau+1})]-d_{\mathcal{P}}(\bteta_\tau)\right)\\
&\leq \frac{-12\eta}{4n}\frac{\tn{\Jc(\bteta_{i})\rb_\tau}^2}{\tn{\rb_\tau}}+\frac{3\eta\bn}{n}\tn{\Jc(\bteta_\tau)\rb_\tau}\\
&\leq\frac{3\eta}{n}\tn{\Jc(\bteta_\tau)\rb_\tau}\left( \bn-\frac{\tn{\Jc(\bteta_\tau)\rb_\tau}}{\tn{\rb_\tau}}\right)\\
&\leq 0.
\end{align*}
%\begin{lemma}\label{tot res} Consider the setup and assumptions of Lemma \ref{avedec}. Furthermore, fix $K\ge\textcolor{red}{????}$ and let $\{\vct{p}_\ell\}_{\ell=1}^K$ be an $\epsilon:=\frac{\twonorm{f(\vct{\theta}_0)-\y}}{\alpha}$ packing of $\mathcal{D}$ so that pairwise distances in this set are at least $\epsilon$.\footnote{Classical results guarantee that, we can find a $(1/\eps)^p$ $R$-packing of a $R/\eps$-ball. In our case $\cp^p\geq K$ using $p\geq n$.} Define the shortest path potential
%\begin{align*}
%\mathcal{V}\left(\vct{\theta}^{+}\right)=12\E\big[\twonorm{f(\vct{\theta}^{+})-\y}\big]+\frac{\alpha}{K}\sum_{\ell=1}^K \E\big[\twonorm{\vct{\theta}^{+}-\vct{p}_\ell}\big].
%\end{align*}
%\end{lemma}

\subsubsection{SGD remains in the local neighborhood}\label{neigharg}
In this section we show that SGD iterates remain close to the initialization. Specifically we prove the following lemma.

% Following Lemmas \ref{lemma residuala} (pick $\eta\leq\eta_a/4$), \ref{lemma residual} (pick $\eta\leq\eta_b/4$) and \ref{lem param},  Then
\begin{lemma}\label{supmartin1} Consider the setup of Lemma \ref{tot res} and the potential function $\mathcal{V}$ from \eqref{spp}. Also define the stopping time $T=\min\{\tau:\bteta_\tau\not\in \Bc(\cs/2)\}$. Under the stated assumptions, 
\[
\Pro\big\{T=\infty\big\}\geq 1-\frac{4}{\nu}\left(\frac{\bp}{\bn}\right)^{\frac{1}{p}} .
\]
\end{lemma}
\begin{proof} 
%Recall definition of $\alpha_i$ from Lemma \ref{tot res}. First, 
%Define the stopping time $\bar{\tau}=\min\{k:\alpha_k\geq \cs R/4\}$.
%We first claim that $\bteta_k\not\in \Bc(\bteta_0,\cs R/2)$ implies $\alpha_k\geq \cs R/4$, hence $\tau\geq \bar{\tau}$. To see this, 
Assume $\bteta_\tau\not\in \Bc(\cs/2)$. 
%Then $\tn{\bteta_k-\bteta_0}\geq \cs R/2$. This implies that for all $\vb_j$ in the packing, we have
%\[
%\tn{\vb_j-\bteta_k}\geq \tn{\bteta_k-\bteta_0}-\tn{\vb_j-\bteta_0}\geq (\cs /2-\cp)R\geq \cs R/4.
%\]
%This implies that packing-distance satisfies
%\[
%{ \bn}d(\bteta_k)\geq \cs \tn{\rb_0}/4.
%\]
This implies $\tn{\rb_\tau}\geq \frac{\cs}{3}\tn{\rb_0}$, hence the potential $\mathcal{V}(\vct{\theta}_\tau)$ can be lower bounded as
\[
\mathcal{V}_\tau=\mathcal{V}(\vct{\theta}_\tau)\ge \frac{ \bn}{K}\sum_{\ell=1}^K\tn{\bteta_\tau-\vct{p}_\ell}+12\tn{\rb_\tau}\geq 12\tn{\rb_\tau}\geq 4\cs \tn{\rb_0}.
\]
Define the stopping time $\widetilde{T}$ which is the first instance $\mathcal{V}_\tau\geq 4\cs\tn{\rb_0}$. Clearly $\widetilde{T}\leq T$ and $\mathbb{P}\{T=\infty\}\ge \mathbb{P}\{\widetilde{T}=\infty\}$. To show that $\widetilde{T}=\infty$ holds with high probability we utilize an argument similar to \cite{tan2017phase}. Define $a\wedge b=\min(a,b)$ and the stopped process $\mathcal{U}_\tau=\mathcal{V}_{\tau\wedge\widetilde{T}}$. We will show that $\mathcal{U}_\tau$ is a supermartingale. Let $\Fc_\tau$ denote the $\sigma$-algebra generated by the first $\tau$ SGD random variables $\rng_1, \rng_2, \ldots, \rng_\tau$. By construction, $\bteta_\tau,\rb_\tau,\mathcal{V}_\tau$ are measurable with respect to $\Fc_\tau$. We can decompose the expectation based on the event $\widetilde{T}> \tau$ as follows
\begin{align*}
\E[\mathcal{U}_{\tau+1}\bgl \Fc_\tau]&=\E[\mathcal{V}_{(\tau+1)\wedge\widetilde{T}}\mathbb{1}_{\widetilde{T}\leq \tau}\bgl \Fc_\tau]+\E[\mathcal{V}_{(\tau+1)\wedge\widetilde{T}}\mathbb{1}_{\widetilde{T}> \tau}\bgl \Fc_\tau],\\
&=\E[\mathcal{V}_{\tau\wedge\widetilde{T}}\mathbb{1}_{\widetilde{T}\leq \tau}\bgl \Fc_\tau]+\E[\mathcal{V}_{\tau+1}\mathbb{1}_{\widetilde{T}> \tau}\bgl \Fc_\tau].
\end{align*}
The $\mathcal{V}_{\tau\wedge\widetilde{T}}$ term is measurable with respect to filteration $\Fc_\tau$, hence
\[
\E[\mathcal{V}_{\tau\wedge\widetilde{T}}\mathbb{1}_{\widetilde{T}\leq \tau}\bgl \Fc_\tau]=\mathcal{V}_{\tau\wedge\widetilde{T}}\mathbb{1}_{\widetilde{T}\leq \tau}=\mathcal{U}_\tau\mathbb{1}_{\widetilde{T}\leq \tau}.
\]
Therefore we can focus on the $\mathcal{V}_{\tau+1}\mathbb{1}_{\widetilde{T}> \tau}$ term. As previously discussed, $\widetilde{T}>\tau$ implies $\bteta_\tau\in \Bc(\cs/2)$ and Lemma \ref{tot res} is applicable. This yields
\[
\E[\mathcal{V}_{\tau+1}\mathbb{1}_{\widetilde{T}> \tau}\bgl \Fc_\tau]=\E[\mathcal{V}_{\tau+1}\bgl \Fc_\tau]\mathbb{1}_{\widetilde{T}> \tau}\leq \mathcal{V}_\tau\mathbb{1}_{\widetilde{T}> \tau}.
\]
Also note that $\mathcal{V}_\tau\mathbb{1}_{\widetilde{T}> \tau}=\mathcal{V}_{\tau\wedge \widetilde{T}}\mathbb{1}_{\widetilde{T}> \tau}=\mathcal{U}_\tau\mathbb{1}_{\widetilde{T}> \tau}$. Combining the latter two identities we have
\[
\E[\mathcal{U}_{\tau+1}\bgl \Fc_\tau]\leq  \mathcal{U}_\tau\mathbb{1}_{\widetilde{T}> \tau}+\mathcal{U}_\tau\mathbb{1}_{\widetilde{T}\leq \tau}=\mathcal{U}_\tau.
\]
Now that we established $\mathcal{U}_\tau$ is a supermartingale, Martingale maximal inequality \cite{revuz2013continuous} implies that
\[
\Pro\Big\{\sup_{\tau\geq 0}\text{ }\mathcal{U}_\tau\geq 4\cs \tn{\rb_0}\Big\}\leq \frac{\mathcal{U}_0}{4\cs \tn{\rb_0}}=\frac{\mathcal{V}_0}{4\cs \tn{\rb_0}}\leq \frac{14\left(\frac{\bp}{\bn}\right)^{1/p}}{4\cs}\le 4\frac{\left(\frac{\bp}{\bn}\right)^{1/p}}{\cs}.
\]
Hence, $\Pro\{T=\infty\}\ge \Pro\{\widetilde{T}=\infty\}\geq 1-\frac{4}{\nu}\left(\frac{\bp}{\bn}\right)^{\frac{1}{p}}$.
\end{proof}

\subsubsection{Putting everything together (completing the proof of Theorem \ref{SGDthm})}\label{together}
In this Section we combine the results of the previous sections to complete the proof. First, note that Lemma \ref{supmartin1} already establishes the result for $\Pro\big\{T=\infty\big\}$. We set the event $E$ to be equal to $\{T=\infty\}$. To show the result on the convergence rate, we note that if $T=\infty$ then $\bteta_\tau\in \Bc\left(\frac{\cs}{2} \right)$ and hence \eqref{updateineq2} holds. This in turn implies that $\E[\tn{\rb_{\tau+1}}^2]\leq \left(1-\frac{\eta}{2n}\bn^2\right)\tn{\rb_\tau}^2$. Now recall the filteration $\Fc_\tau$ generated from random SGD updates in the proof of Lemma \ref{supmartin1}. We have
\[
\E[\tn{\rb_{\tau+1}}^2\mathbb{1}_{T=\infty}\bgl \Fc_\tau]\leq \E[\tn{\rb_{\tau+1}}^2\mathbb{1}_{T>\tau}\bgl \Fc_\tau].
\]
To continue further note that $\mathbb{1}_{T>\tau}$ is measurable with respect to $\Fc_\tau$. Hence, applying \eqref{updateineq2} over the event $T>\tau$, we conclude that
\begin{align*}
 \E[\tn{\rb_{\tau+1}}^2\mathbb{1}_{T>\tau}\bgl \Fc_\tau]&=\E[\tn{\rb_{\tau+1}}^2\bgl \Fc_\tau]\mathbb{1}_{T>\tau},\\
 &\leq \left(1-\frac{\eta \bn^2}{2n}\right)\tn{\rb_\tau}^2\mathbb{1}_{T>\tau},\\
 &\leq \left(1-\frac{\eta \bn^2}{2n}\right)\tn{\rb_\tau}^2\mathbb{1}_{T>(\tau-1)}.
\end{align*}
With this recursion established, we take conditional expectations to obtain
\[
 \E\big[\tn{\rb_{\tau+1}}^2\mathbb{1}_{T>\tau}\big]\leq \left(1-\frac{\eta \bn^2}{2n}\right)^\tau \E\big[\tn{\rb_{1}}^2\mathbb{1}_{T>0}\big]\leq \left(1-\frac{\eta \bn^2}{2n}\right)^{\tau+1}\tn{\rb_0}^2,
\]
which completes the proof.
\subsection{GLM proofs (Proof of Theorem \ref{simpGLM})}
First we prove that the is a globally optimal solution achieving zero training error. To see this note that any strictly increasing and differentiable activation $\phi$ is invertible on $\R$ by the implicit function theorem. Let $\Pi_{\Rc}$ and $\Pi_{\Nc}$ denote the projections onto the row space and null space of $\X$ respectively. By the assumptions of the theorem $\mtx{X}$ has full row rank and pseudo-inverse solution $\bteta^\dagger$ is given by $\bteta^\dagger=\X^T(\X\X^T)^{-1}\phi^{-1}(\y)\}$. Hence the set of global optimal solutions  is non-empty and all globally optimal solutions are characterized by the null space as follows
\[
\Gc=\{\vct{\theta}\in\R^p: \bteta=\bteta^\dagger+\vb\quad\text{where}\quad\vb\in\text{null}(\X)\}
\]
Let $\vct{\theta}^*=\Pi_\Nc(\bteta_0)+\bteta^\dagger\in\Gc$. By construction $\vct{\theta}^*$ is the closest global minima to $\bteta_0$ as the null space projections match. We will argue that the gradient descent iterations linearly converge to $\bteta^*$. 

%Note that the gradient descent iterations are given by
Towards this goal, note that $\y=\phi\left(\mtx{X}\vct{\theta}^*\right)$ and note that the gradient descent iterations are given by
\begin{align}
\bteta_{\tau+1}=&\bteta_\tau+\eta\X^T\text{diag}\left(\phi'(\X\bteta_\tau)\right)(\phi(\X\bteta^\star)-\phi(\X\bteta_\tau))\\
&\bteta_\tau+\eta\X^T\text{diag}\left(\phi'(\X\bteta_\tau)\right)(\y-\phi(\X\bteta_\tau)).
\end{align}
%and the gradient is subset of the range space similar to $\bteta^\dagger$.  and the gradient is inside row space $\Rc$ and any global minimizer is equal to $\bteta^\dagger$. Let  denote the closest such solution to $\vct{\theta}_0$. We shall prove that the gradient descent iterates converge to this global optima. 
%We first prove that a global minima within advertised radius of $\bteta_0$. First of all, since $\phi'$ is continuous $\tn{\frac{\pa\Jc(\bteta)}{\pa\bteta}}$ is continuous hence it is bounded over a compact domain. This implies $\Jc(\bteta)$ is Lipschitz over $\Dc$ with some $\el>0$. Now, applying Lemma \ref{lip con phi} and Theorem \ref{small dev}, with sufficiently small learning rate, gradient descent converges to a global minimizer solution $\bteta_\star$ with $0$ loss while satisfying
%\[
%\frac{\gamma \alpha}{4}\tn{\bteta_i}+\sqrt{2\Lc(\bteta_i)}\leq 2\sqrt{n}.
%\]
%This follows from substituting $ \bp, \bn$ and using $\tn{\y-\phi(0)}\leq 2\sqrt{n}$ and implies $\tn{\bteta_\star}\leq 8\sqrt{n}/\gamma\alpha$. Now, we just need to argue that, we can actually choose a large learning rate during convergence. To show this, we take advantage of the structure of the problem and treat it as overdetermined. Let $\Sc=\text{Range}(\X^T)$ be the range space of data. Gradient updates lie on $\Sc$ hence $\bteta_\star,\bteta\in \Sc$ for all $i$ and we can restrict our attention to here. Starting from $0$ and using the fact that $\phi(\X\bteta_\star)=\y$, we have
Now, for two vectors $\vct{a}$ and $\vct{b}$ obeying $\vct{a}\neq \vct{b}$ define $\phi'(\vct{a},\vct{b})=\frac{\phi(\vct{a})-\phi(\vct{b})}{\vct{a}-\vct{b}}$ (with the devision interpreted as entry by entry) and note that by the mean value theorem $\phi'(\vct{a},\vct{b})\ge \gamma$. Also note that, we can write $\phi(\X\bteta_\tau)-\phi(\X\bteta^\star)=\text{diag}(\phi'(\X\bteta_\tau,\X\bteta^\star))\X(\bteta_\tau-\bteta^\star)$. Consequently, setting $\vct{h}_\tau=\bteta_\tau-\bteta^\star$ and $\Db_\tau=\text{diag}(\phi'(\X\bteta_\tau))\text{diag}(\phi'(\X\bteta_\tau,\X\bteta^\star))$, we have
\begin{align}
\label{hrecurs}
\vct{h}_{\tau+1}&=\vct{h}_\tau-\eta \X^T\Db_\tau\X\vct{h}_\tau=\left(\Iden-\eta \X^T\Db_\tau\X\right)\vct{h}_\tau.
\end{align}
Since gradient is an element of the row space $\Rc$, $\Pi_{\Nc}(\bteta_\tau)=\Pi_{\Nc}(\bteta_0)=\Pi_{\Nc}(\bteta^*)$ and $\h_\tau\in \Rc$. To proceed, let $\Vb\in\R^{n\times p}$ be an orthonormal basis (i.e.~$\mtx{V}\mtx{V}^T=\mtx{I}_n$) for the row space of $\mtx{X}$ and define $\widetilde{\vct{h}}_\tau=\Vb\vct{h}_\tau$ and $\widetilde{\X}=\X\Vb^T$. \eqref{hrecurs} yields the following update rule for $\widetilde{\vct{h}}_\tau$
\begin{align*}
\widetilde{\vct{h}}_{\tau+1}&=\Vb(\Iden-\eta \X^T\Db_\tau\X)\Vb^T\widetilde{\vct{h}}_\tau,\\
&=\left(\Iden-\eta\widetilde{\X}^T\Db_\tau\widetilde{\X}\right)\widetilde{\vct{h}}_\tau.
\end{align*}
%We have that $\tn{\Vb\w_\tau}=\tn{\w_i}$ and $\bar{\X}$ has same singular values as $\X$. Hence
% hence combined with spectrum of $\X$, we have $\Iden\succeq \bar{\X}^T\Db\bar{\X}\succeq \gamma^2\alpha^2\Iden$. Hence, if $\eta\leq 1$, 
To continue further, we use the fact that $\Db_\tau$ is diagonal with entries between $\gamma^2$ and $\Gamma^2$. This combined with the fact that the matrices $\widetilde{\mtx{X}}$ and $\mtx{X}$ have the same eigenvalues allow us to conclude that $\gamma^2\sigma_{\min}^2\left(\mtx{X}\right)\mtx{I}\preceq\widetilde{\X}^T\Db_\tau\widetilde{\X}\preceq\Gamma^2\opnorm{\mtx{X}}^2\mtx{I}$. Thus, for $\eta\le\frac{1}{\Gamma^2\opnorm{\mtx{X}}^2}$
\begin{align*}
\mtx{0}\preceq\Iden-\eta\widetilde{\X}^T\Db_\tau\widetilde{\X}\preceq \left(1-\eta\gamma^2\sigma_{\min}^2(\mtx{X})\right)\mtx{I}.
\end{align*}
Thus, using the fact that $\twonorm{\widetilde{\vct{h}}_\tau}=\twonorm{\mtx{V}\vct{h}_\tau}=\twonorm{\vct{h}_\tau}$ (as $\vct{h}_\tau\in \Rc$) we conclude that
\begin{align*}
\twonorm{\vct{h}_{\tau+1}}\le \left(1-\eta\gamma^2\sigma_{\min}^2(\mtx{X})\right)\twonorm{\vct{h}_\tau},
\end{align*}
completing the proof of \eqref{GLM}. Furthermore, note that 
\[
\tn{\bteta_{\tau+1}-\bteta_\tau}=\tn{\h_{\tau+1}-\h_{\tau}}\leq \tn{\eta \X^T\Db_\tau\X\h_\tau}\leq \eta \Gamma^2\|\X\|^2\tn{\h_\tau}.
\]
Summing these up from $\tau=0$ to $\infty$ and using $\tn{\h_\tau}\leq \left(1-\eta\gamma^2\sigma_{\min}^2(\mtx{X})\right)^\tau\twonorm{\vct{h}_0}$ we conclude that for $\eta=\frac{1}{\Gamma^2\opnorm{\mtx{X}}^2}$
\begin{align*}
\sum_{\tau=0}^\infty \twonorm{\vct{\theta}_{\tau+1}-\vct{\theta}_\tau}=\sum_{\tau=0}^\infty \twonorm{\vct{h}_\tau}\le \frac{1}{1-\left(1-\eta\gamma^2\sigma_{\min}^2(\mtx{X})\right)}\twonorm{\vct{h}_0}=\frac{\Gamma^2}{\gamma^2}\frac{\lambda_{\max}\left(\mtx{X}\mtx{X}^T\right)}{\lambda_{\min}\left(\mtx{X}\mtx{X}^T\right)}\twonorm{\vct{h}_0},
\end{align*}
establishing \eqref{GLMpath}.
%
%
%$\Iden\succeq \bar{\X}^T\Db\bar{\X}\succeq \gamma^2\alpha^2\Iden$
%
%
%\[
%(1-\eta\alpha^2\gamma^2)\Iden\succeq\Iden-\eta\bar{\X}^T\Db\bar{\X}\succeq 0.
%\]
%Setting $\eta=1$, we found $\tn{\w_{i+1}}\leq (1-\alpha^2\gamma^2)\tn{\w_i}$.
%
%
%\MS{dsfewre}
%
%This theorem establishes a dimension-free convergence rate which only requires $\order{\log(1/\eps)}$ to achieve $\eps$ accuracy. This is consistent with linear regression which would solve the pseudo-inverse $\hat{\bteta}=(\X^T\X)^{-1}\X^T\y$. If $\X$ has independent zero-mean entries with variance $1/p$, then $\tn{\X^T\y}=\order{\sqrt{n}}$ and $\X^T\X$ has $\order{1}$ eigenvalues. We believe similar stronger bounds may hold for the neural net application (with $k>1$); which is left for future work.
\subsection{Low-rank recovery proofs (Proof of Theorem \ref{low rank reg2})}
To specialize Theorem \ref{GDthm} we begin by calculating the Jacobian $\Jc(\bTeta):=\Jc(\vc{\bTeta})$ which is given by an $n\times dr$ matrix of the form
\[
\Jc(\bTeta)=\begin{bmatrix}\vc{\X_1\bTeta}&\vc{\X_2\bTeta}&\ldots&\vc{\X_n\bTeta}\end{bmatrix}^T.
\]
Here, for a matrix $\mtx{M}\in\R^{n_1\times n_2}$ we use vect$(\mtx{M})\in\R^{n_1n_2}$ to denote an $n_1n_2$ dimensional column vectors obtained by concatenating the columns of $\mtx{M}$. Similarly, for a vector $\vct{v}\in\R^{n_1n_2}$ we use $\mat{\vct{v}}\in\R^{n_1\times n_2}$ to denote a matrix obtained by reshaping the vector into an $n_1\times n_2$ matrix. 
%\Bc^{d\times r-1}
\subsubsection{Key Lemmas for low-rank recovery}
In order to verify the assumptions of Theorem \ref{GDthm}, in this section we gather some key lemmas related to the Jacobian matrix that building on top of each other play a crucial role in our proofs. We defer the proofs to Appendix \ref{keyLRpf}.
The first key lemma which will play a crucial role in our proofs is that the nuclear norm $\|\mat{\Jc(\bTeta)^T\vb}\|_\star$ is uniformly bounded for all $\vct{v}$ and $\mtx{\Theta}$ with unit Frobenius/Euclidean norms.
%Let $\Mc$ be the set of $d\times r$ matrices with Frobenius norm equal to $1$. We claim that $\sup_{\bTeta\in\Mc,\tn{\vb}=1} \|\mat{\Jc(\bTeta)^T\vb}\|_\star$ is upper bounded tightly.
\begin{lemma}\label{up bound} For $i=1,2,\ldots,n$, $\mtx{X}_i\in\R^{d\times d}$ be i.i.d.~matrices with i.i.d.~entries distributed as $\mathcal{N}(0,1)$. Furthermore, assume $n\le dr$ and $r\le d$. Then
\begin{align*}
\underset{\vct{v}\in\R^k, \bTeta\in\R^{d\times r}: \twonorm{\vct{v}}=\fronorm{\bTeta}=1}{\sup}  \|\mat{\Jc(\bTeta)^T\vb}\|_\star\leq 12\sqrt{dr},
\end{align*}
holds with probability at least $1-e^{-2dr}$.%\MS{Need to check the probability}
\end{lemma}
The next lemma concerns the average of the nuclear norm of a Gaussian matrix multiplied by a diagonal matrix.
\begin{lemma} \label{nuc nor exp}Let $\Gb\in\R^{d\times r}$ with $d\le r$ be i.i.d.~$\Nn(0,1)$ matrix and $\bSi\in\R^{r\times r}$ be a diagonal matrix with entries obeying
\begin{align*}
\vartheta\le\sigma_{\min}\left(\mtx{\Sigma}\right)\le\sigma_{\max}\left(\mtx{\Sigma}\right)\le2\vartheta
\end{align*}
Then,
\[
\E[\nucnorm{\Gb\bSi}]\geq\frac{1}{32}\vartheta\sqrt{d}r.%2\sqrt{dr}/3.
\]
\end{lemma}
The next key lemma used in our proofs also concerns the nuclear norm $\|\mat{\Jc(\bTeta)^T\vb}\|_\star$, however this time we bound this quantity from both below and above for a fixed matrix $\mtx{\Theta}$ that is well conditioned and for all vectors $\vct{v}\in\R^k$ with unit Euclidean norm.
%\newpage
\begin{lemma} \label{low bound}Let $\bTeta\in\R^{d\times r}$ be a matrix with eigenvalues obeying
\begin{align*}
\vartheta\le\sigma_{\min}\left(\mtx{\Sigma}\right)\le\sigma_{\max}\left(\mtx{\Sigma}\right)\le2\vartheta.
\end{align*}
Furthermore, assume $r\le d$ and $n\le c dr$ with $c$ a fixed numerical constant. Then,
\[
 \frac{1}{40}\vartheta\sqrt{d}r\le \nucnorm{\mat{\Jc(\bTeta)^T\vb}}\le 24\vartheta\sqrt{d}r,
\]
holds for all $\vb\in\Sc$ with probability at least $1-2e^{-\gamma dr}$ with $\gamma$ a fixed numerical constant.
\end{lemma}
Next we bound the spectrum of the Jacobian matrix in a ball around the initialization $\mtx{\Theta}_0$.
\begin{lemma} [Jacobian spectrum bounds]\label{jac bound}Let $\bTeta_0\in\R^{d\times r}$ with $r\le d$ be a matrix with singular values lying in the range $[\vartheta,2\vartheta]$. Consider the Frobenius ball around $\bTeta_0$ given by $\Dc=\Bc\left(\bTeta_0,\frac{1}{2400}\vartheta\sqrt{r}\right)$. Then as long as $n\le Cdr$ with $C$ a fixed numerical constant, then, with probability at least $1-3e^{-\gamma dr}$
\begin{align}
\label{Jspec}
 \frac{1}{50}\vartheta\sqrt{dr}\le \sigma_{\min}\left(\mathcal{J}(\mtx{\Theta})\right)\le\sigma_{\max}\left(\mathcal{J}(\mtx{\Theta})\right)\le 25\vartheta\sqrt{d}r.
\end{align}
Furthermore, the Jacobian matrix is $12\sqrt{dr}$-Lipschitz. That is, for all $\mtx{\Theta}_1,\mtx{\Theta}_2\in\R^{d\times r}$ we have
\begin{align}
\label{lipineq}
\opnorm{\mathcal{J}(\mtx{\Theta}_2)-\mathcal{J}(\mtx{\Theta}_1)}\le 12\sqrt{dr}\fronorm{\mtx{\Theta}_2-\mtx{\Theta}_1}.
\end{align}
\end{lemma}
\subsubsection{Completing the proof of Theorem \ref{low rank reg2}}
We will prove this theorem by a direct application of Theorem \ref{GDthm}. To this aim we need to calculate the various parameters in this theorem.

We begin by calculating the size of the initial misfit. To this aim note that $\li\X_i,\bTeta_0\bTeta_0^T\ri\sim \Nn(0,\tf{\bTeta_0\bTeta_0^T}^2)$ and $\tf{\bTeta_0\bTeta_0^T}\le \sqrt{r}\opnorm{\bTeta_0\bTeta_0^T}\le 4\frac{\twonorm{\y}}{\sqrt{n}}$. Hence, $f(\mtx{\Theta}_0)$ is an i.i.d.~Gaussian random vector with standard deviation at most $4\frac{\twonorm{\y}}{\sqrt{n}}$. Using Lipschitz concentration of Gaussians, this implies that
\begin{align*}
\mathbb{P}\Bigg\{\frac{\twonorm{f(\mtx{\Theta}_0)}}{4\frac{\twonorm{\y}}{\sqrt{n}}}\ge 2\sqrt{n}\Bigg\}\le e^{-\frac{n}{2}}.
\end{align*}
Hence, with probability at least $1-e^{-\frac{n}{2}}$, the following holds
\begin{align}
\label{res0}
\twonorm{f(\mtx{\Theta}_0)-\y}\le \twonorm{f(\mtx{\Theta}_0)}+\twonorm{\y}\le 9\twonorm{\y},
\end{align}
Furthermore, applying Lemma \ref{jac bound} with $\vartheta=\sqrt{\frac{\twonorm{\vct{y}}}{\sqrt{rn}}}$, Jacobian matrix satisfies
\begin{align*}
\bn=\frac{1}{50}\sqrt{d\twonorm{\y}\sqrt{\frac{r}{n}}},\quad\bp=25\sqrt{dr\twonorm{\y}\sqrt{\frac{r}{n}}},\quad\text{and}\quad L=12\sqrt{dr}.
\end{align*}
over the domain $\Dc'=\Bc(\bteta_0,\frac{1}{2400}\vartheta\sqrt{r})$ with probability $1-3e^{-\gamma dr}\geq 1-3e^{-n/2}$ (by picking $c\leq \gamma$). On the other hand, for Theorem \ref{GDthm} to be applicable, we need the domain $\Dc$ radius to be $R:= \frac{4\twonorm{f(\mtx{\Theta}_0)-\y}}{\bn}$. The key idea is choosing $n\le cdr$ for a sufficiently small $c$ to ensure that $\Dc\subset\Dc'$ and Theorem \ref{GDthm} applies. In particular, this follows from
\begin{align*}
R:= \frac{4\twonorm{f(\mtx{\Theta}_0)-\y}}{\bn}\le 1800\sqrt{\frac{\twonorm{\vct{y}}}{d\sqrt{\frac{r}{n}}}}=1800\sqrt{\frac{n}{dr}}\sqrt{\twonorm{\y}\sqrt{\frac{r}{n}}}\le \frac{1}{2400}\sqrt{\twonorm{\y}\sqrt{\frac{r}{n}}}=\frac{1}{2400}\vartheta\sqrt{r},
\end{align*}
%confirming that Lemma \ref{jac bound} was indeed applicable.
Now that Theorem \ref{GDthm} applies, all that remains is to upper bound these quantities in the upper bound on the learning rate. Per Theorem \ref{GDthm} we need to ensure
\begin{align*}
\eta\le \frac{1}{2\bp^2}\cdot\min\left(1,\frac{\bn^2}{L\twonorm{f(\mtx{\Theta}_0)-\y}}\right).
\end{align*}
To do this note that
\begin{align*}
\frac{\bn^2}{L\twonorm{f(\mtx{\Theta}_0)-\y}}\ge \frac{\bn^2}{9L\twonorm{\y}}=\frac{\bn^2}{108\sqrt{dr}\twonorm{\y}}=\frac{\frac{1}{2500}d\tn{\y}\sqrt{r/n}}{108\sqrt{dr}\twonorm{\y}}=\frac{1}{270000}\sqrt{\frac{d}{n}},
\end{align*}
and use $\min(1,\sqrt{\frac{d}{n}})\geq \frac{1}{\sqrt{r}}$. Proceeding, we use this naive bound to simplify the final expressions. This yields the step size requirement of
\[
\eta\leq \frac{c'}{\beta^2\sqrt{r}}=\frac{c_1}{dr\tn{\y}\sqrt{r/n}\sqrt{r}}=\frac{c_1\sqrt{n}}{r^2d\tn{\y}}
\]
Observing $\alpha^2/\beta^2=1/r$ and substituting $\eta$ and convergence rate $1-\eta\alpha^2/2$ concludes the proof.
%and use $\min(1,\sqrt{\frac{d}{n}})\geq \frac{1}{1+\sqrt{\frac{n}{d}}}$. This gives the requirement
%\[
%\eta\leq \frac{c'}{\beta^2(1+\sqrt{\frac{n}{d}})}=\frac{c_1}{dr\tn{\y}\sqrt{r/n}(1+\sqrt{n/d})}=\frac{c_1\sqrt{n}}{r^{3/2}\tn{\y}(d+\sqrt{nd})}
%\]
%Observing $\alpha^2/\beta^2=1/r$ and substituting $\eta$ and convergence rate $1-\alpha^2\eta/2$ concludes the proof.
\subsection{Neural net proofs (Proof of Theorem \ref{thm over glm})}
%First note that under the stated assumptions one can always show that for any $\vct{y}$ there exists a $\mtx{W}\R^{k\times d}$ such that $\y=\phi\left(\mtx{W}\mtx{X}\right)^T\vct{v}$. To see this note that due to the fact that $\twonorm{\vb}=1$ at least one entry of $\vct{v}$ is nonzero. Without loss of generality assume this entry is $\vct{v}_1$. Then similar to the GLM proof there exists a $\vct{w}\in\R^d$ such that $\phi(\mtx{X}\vct{w})=\vct{y}/\vct{v}_1$. Thus, $\mtx{W}$ can simply be chosen as the matrix with all zero entries except for the first row which is equal to $\vct{w}$. Let $\mtx{W}^*$ denote the matrix with the smallest distance (in Frobenius norm) to $\mtx{W}_0$ among all that satisfy $\y=\phi\left(\mtx{W}\mtx{X}\right)^T\vct{v}$.
We begin by noting that the Jacobian matrix in this case is equal to
\begin{align*}
\mathcal{J}(\mtx{W})=
\begin{bmatrix}
\vct{v}_1\mathcal{J}(\vct{w}_1) & \ldots & \vct{v}_k\mathcal{J}(\vct{w}_k)
\end{bmatrix}\in\R^{n\times kd}\quad\text{with}\quad\mathcal{J}(\vct{w}_\ell):=\text{diag}(\phi'(\mtx{X}\vct{w}_\ell))\mtx{X}.
\end{align*}
To prove this theorem we use Theorem \ref{GDthm} with $R=\infty$. We just need to calculate the various parameters and verify that the assumptions hold.

\noindent\textbf{Bounding the spectrum of $\mathcal{J}$}. We begin by calculating $\bn$ and $\bp$. To this aim note
\begin{align*}
\mathcal{J}(\vct{\w_\ell})\mathcal{J}^T(\vct{\w_\ell})=\text{diag}\left((\phi'(\mtx{X}\vct{w}_\ell)\right)\mtx{X}\mtx{X}^T\text{diag}\left((\phi'(\mtx{X}\vct{w}_\ell)\right).
\end{align*}
Thus, using the bounds on $\phi'$
\begin{align*}
\gamma^2\sigma_{\min}^2(\mtx{X})\mtx{I}\preceq\mathcal{J}(\vct{\w_\ell})\mathcal{J}^T(\vct{\w_\ell})\preceq \Gamma^2\opnorm{\mtx{X}}^2\mtx{I}.
\end{align*}
This in turn implies that for $\mathcal{J}(\mtx{W})\mathcal{J}^T(\mtx{W})=\sum_{\ell=1}^k\vct{v}_k^2\mathcal{J}(\vct{\w_\ell})\mathcal{J}^T(\vct{\w_\ell})$ we have
\begin{align*}
\gamma^2\twonorm{\vct{v}}^2\sigma_{\min}^2(\mtx{X})\mtx{I}\preceq\mathcal{J}(\mtx{W})\mathcal{J}^T(\mtx{W})\preceq \Gamma^2\twonorm{\vct{v}}^2\opnorm{\mtx{X}}^2\mtx{I},
\end{align*}
so that we can use
\begin{align*}
\alpha=\gamma\sigma_{\min}(\mtx{X})\quad\text{and}\quad\beta=\Gamma\opnorm{\mtx{X}}.
\end{align*}

\noindent\textbf{Bounding the Lipschitz parameter of $\mathcal{J}$}. To calculate $L$ note that
\begin{align*}
\mathcal{J}(\widetilde{\mtx{W}})-\mathcal{J}(\mtx{W})=
\begin{bmatrix}
\vct{v}_1\left(\mathcal{J}(\widetilde{\vct{w}}_1)-\mathcal{J}(\vct{w}_1)\right) & \ldots & \vct{v}_k\left(\mathcal{J}(\widetilde{\vct{w}}_k)-\mathcal{J}(\vct{w}_k)\right)
\end{bmatrix}.
\end{align*}
Thus 
\begin{align*}
\opnorm{\mathcal{J}(\widetilde{\mtx{W}})-\mathcal{J}(\mtx{W})}^2\overset{(a)}{\le}&\sum_{\ell=1}^k\opnorm{\vct{v}_\ell\left(\mathcal{J}(\widetilde{\vct{w}}_\ell)-\mathcal{J}(\vct{w}_\ell)\right)}^2\\
=&\sum_{\ell=1}^k\vct{v}_\ell^2\opnorm{\text{diag}\left(\phi'(\mtx{X}\widetilde{\vct{w}}_\ell)-\phi'(\mtx{X}\vct{w}_\ell)\right)\mtx{X}}^2\\
=&\sum_{\ell=1}^k\vct{v}_\ell^2 \opnorm{\text{diag}\left(\int_0^1\phi''\left(\mtx{X}\left(t\widetilde{\w}_\ell+(1-t)\w_\ell\right)\right) dt\right)\text{diag}\left(\mtx{X}\left(\widetilde{\vct{w}}-\vct{w}\right)\right)\mtx{X}}^2,\\
\le&\sum_{\ell=1}^k \vct{v}_\ell^2M^2\|\mtx{X}\|_{2,\infty}^2\opnorm{\mtx{X}}^2\twonorm{\widetilde{\w}_\ell-\w_\ell}^2\\
=&\twonorm{\vct{v}}^2M^2\|\mtx{X}\|_{2,\infty}^2\opnorm{\mtx{X}}^2\fronorm{\widetilde{\mtx{W}}-\mtx{W}}^2.
\end{align*}
In the above (a) follows from the fact the square of the spectral norm of concatenation of matrices is bounded by sum of squares of the spectral norms of the individual matrices. Thus we can use
\begin{align*}
L=M\|\mtx{X}\|_{2,\infty}\opnorm{\mtx{X}}.
\end{align*}
The proof is complete by applying Theorem \ref{GDthm}.
\subsection{PL proofs}
\subsubsection{PL convergence proof (Proof of Theorem \ref{pl thm})}
Suppose \eqref{pl conv2} and \eqref{pl conv} hold until step $\tau$. This implies $\bteta_\tau\in\Dc$ and local PL is applicable. If $\Lc(\bteta_\tau)=0$, then $\bteta_\tau$ is global minimizer and {since $\Lc$ is differentiable} $\grad{\bteta_\tau}=0$ which in turn implies that $\bteta_{\tau+1}=\bteta_\tau$ and thus \eqref{pl conv} holds for $\vct{\theta}_{\tau+1}$. Otherwise, $\Lc(\bteta_\tau)>0$ and using the triangular inequality we can conclude that
\begin{align}
\tn{\bteta_{\tau+1}-\bteta_0}\leq \tn{\bteta_\tau-\bteta_0}+\tn{\bteta_{\tau+1}-\bteta_\tau}\leq \tn{\bteta_\tau-\bteta_0}+\eta\tn{\grad{\bteta_\tau}}.\label{teq1}
\end{align}
Since $\grad{\cdot}$ is Lipschitz, we have $\Lc(\bteta_{\tau+1}) \leq \Lc(\bteta_\tau) + (\bteta_{\tau+1}-\bteta_\tau)^T\grad{\bteta_\tau}+\frac{L}{2}\tn{\bteta_{\tau+1}-\bteta_\tau}^2$ for $\eta\leq \eta_{\max}$ where $\eta_{\max}$ is the largest step size ensuring $\bteta_{\tau+1}\in\Dc$. Hence, for any $\eta\leq \widetilde{\eta}_{\max}=\min(1/L,\eta_{\max})$ we have
\begin{align}
\Lc(\bteta_{\tau+1})\leq \Lc(\bteta_\tau)-\frac{\eta}{2}\tn{\grad{\bteta_\tau}}^2.\label{pl conv bound}
\end{align}
Now define
\begin{align*}
\eps_\tau(\eta):=& \left(\sqrt{\Lc(\bteta_\tau)}-\frac{\eta}{4\sqrt{\Lc(\bteta_\tau)}}\tn{\grad{\bteta_\tau}}^2\right)-\sqrt{\Lc(\bteta_\tau)-\frac{\eta}{2}\tn{\grad{\bteta_\tau}}^2},\\
\ge& \left(\sqrt{\Lc(\bteta_\tau)}-\frac{\eta}{4\sqrt{\Lc(\bteta_\tau)}}\tn{\grad{\bteta_\tau}}^2\right)-\sqrt{\left(\sqrt{\Lc(\bteta_\tau)}-\frac{\eta}{4\sqrt{\Lc(\bteta_\tau)}}\tn{\grad{\bteta_\tau}}^2\right)^2},\\
=&0,
\end{align*}
so that $\epsilon_\tau(\eta)> 0$ for $\eta>0$. Using this definition in \eqref{pl conv bound} together with the PL condition for $\bteta_\tau\in\Dc$, we arrive at
\begin{align}%\sqrt{ \Lc(\bteta_{i})-\frac{\eta}{2}\tn{\grad{\bteta_i}}^2}\leq
\sqrt{\Lc(\bteta_{\tau+1})}\leq   \sqrt{\Lc(\bteta_\tau)}-\frac{\eta}{4\sqrt{\Lc(\bteta_\tau)}}\tn{\grad{\bteta_\tau}}^2-\eps_\tau(\eta)\leq \sqrt{\Lc(\bteta_\tau)}-\frac{\eta\sqrt{2\mu}}{4}\tn{\grad{\bteta_\tau}}-\eps_\tau(\eta).\label{res1}
\end{align}
To continue we define the potential/Lyapunov function $\mathcal{V}_\tau=\sqrt{\Lc(\bteta_\tau)}+\sqrt{\mu/8}\tn{\bteta_\tau-\bteta_0}$ to monitor the sum of the square root of the loss and the distance to initialization. Adding inequalities \eqref{teq1} and \eqref{res1}, we find that for all $\eta\leq \widetilde{\eta}_{\max}$
\begin{align}
\frac{1}{\eta}\left(\mathcal{V}_{\tau+1}+\eps_\tau(\eta)-\mathcal{V}_\tau\right)\leq  \sqrt{\frac{\mu}{8}}\tn{\grad{\bteta_\tau}}-\frac{\sqrt{2\mu}}{4}\tn{\grad{\bteta_\tau}}\leq 0\quad\implies\quad \mathcal{V}_{\tau+1}\le \mathcal{V}_\tau-\eps_\tau(\eta).\label{for all eq}
\end{align}
Next, we argue that $\eta_{\max}\geq 1/L$ and thus $\widetilde{\eta}_{\max}=1/L$. Note that $\eta_{\max}>0$ since $\Lc(\bteta_\tau)>0$ which implies $\bteta_\tau$ is strictly inside $\Dc$ via \eqref{pl conv}. To show that $\eta_{\max}\ge 1/L$, we proceed by contradiction and assume that $\eta_{\max}< 1/L$. Now define $\bteta_{\max}:=\bteta_\tau-\eta_{\max}\grad{\bteta_\tau}$ and note that by the definition of $\eta_{\max}$, we have $\tn{\bteta_{\max}-\bteta_0}=R$. On the other hand, since $\eta_{\max}>0$ we have $\eps(\eta_{\max})>0$ so that applying the update inequality \eqref{for all eq} (which holds if $\eta_{\max}< 1/L$) we can conclude that
\[
\sqrt{\mu/8}\tn{\bteta_{\max}-\bteta_0}\le \sqrt{\mathcal{L}(\vct{\theta}_{\max})}+\sqrt{\mu/8}\tn{\bteta_{\max}-\bteta_0} \le\mathcal{V}_\tau -\eps(\eta_{\max})< \mathcal{V}_0\implies \tn{\bteta_{\max}-\bteta_0}<R.
\] 
This is in contradiction with $\tn{\bteta_{\max}-\bteta_0}=R$ and therefore $\eta_{\max}\ge 1/L$ and $\widetilde{\eta}_{\max}=1/L$.

The argument above shows that the recursion \eqref{for all eq} is valid for $\eta\le 1/L$ which proves \eqref{pl conv} and also in turn guarantees that all $\bteta_\tau$'s  stay within the neighborhood $\Dc$ with the learning rate choice of $\eta\leq 1/L$. To show convergence of the loss, we combine \eqref{pl conv bound} with the PL condition $\tn{\grad{\bteta_\tau}}^2\geq2 \mu\Lc(\bteta_\tau)$ to conclude that%The standard PL proof guarantees that
\[
\Lc(\bteta_{\tau+1})\leq (1-\eta\mu)\Lc(\bteta_\tau)\leq (1-\eta\mu)^{\tau+1}\Lc(\bteta_{0}),
\]
completing the proof of \eqref{pl conv2}. To conclude with the result on the shortest path, we add \eqref{res1} from $\tau=0$ to $\infty$ to conclude that
\[
\sum_{\tau=0}^{\infty}\frac{\eta\sqrt{2\mu}}{4}\tn{\grad{\bteta_i}}\le \sqrt{\Lc(\bteta_0)}\quad\implies\quad \sum_{\tau=0}^\infty \twonorm{\vct{\theta}_{\tau+1}-\vct{\theta}_\tau}\le \sqrt{\frac{8\mathcal{L}(\vct{\theta}_0)}{\mu}},
\]
completing the proof of \eqref{pl short}.
%\MS{below is proof the lower bound lemma}
\subsubsection{PL lower bound proof (Proof of Theorem \ref{pl low bound})}
\begin{proof} Suppose there exists $\bteta\in\Dc$ satisfying $\Lc(\bteta)=0$. Since $\Lc$ is differentiable and minimized at $\bteta$ the gradient must vanish, i.e.~$\grad{\bteta}=0$. From smoothness of the loss we conclude that%, we can set $\bteta_\alpha=(1-\alpha)\bteta+\alpha \bteta_0$ and lower bound the objective via
\[
\Lc(\bteta_0)\leq \Lc(\bteta)+(\bteta_0-\bteta)^T\grad{\bteta}+\frac{L}{2}\tn{\bteta-\bteta_0}^2=\frac{L}{2}\tn{\bteta-\bteta_0}^2.
\]
This implies $\tn{\bteta-\bteta_0}\geq \sqrt{2\Lc(\bteta_0)/L}$ and contradicts with the choice of $R$.

The remaining proof is similar to that of Theorem \ref{low bound thm}. Consider the least squares problem where $\X$ is a matrix with orthogonal rows. The first row $\x_1$ of $\X$ has length $\sqrt{\mu}$ and the other rows have arbitrary lengths. Fix an arbitrary scaling $\gamma\geq 0$ and set $\bteta^\star=\gamma\x_1/\tn{\x_1}$ and $\bteta_0=0$. Set labels $\y=\X\bteta^\star$ and loss $\Lc(\bteta)=\frac{1}{2}\tn{\y-\X\bteta}^2$. Gradient is $\|\X^T\X\|$ Lipschitz, which is same as $\ell_2^2$ of the largest row, hence $L$ can be set arbitrarily. For any $\bteta$, we have
\[
\tn{\X^T(\X\bteta-\y)}^2=\tn{\X^T\X(\bteta-\bteta^\star)}^2\geq \mu\tn{\X(\bteta-\bteta^\star)}^2=2\mu\Lc(\bteta)
\]
%where the equality is satisfied if $\bteta$ is proportional to $\x_1$. Finally,
Next, observe that (i) $\Lc(0)=\gamma^2\mu/2$ and (ii) any global minimizer $\bteta$ satisfies $\y=\X\bteta^\star=\X\bteta$ hence we have that 
\[
\tn{\bteta}\geq \frac{\x_1^T\bteta}{\tn{\x_1}}=\frac{\x_1^T\bteta^\star}{\tn{\x_1}}=\gamma.
\]
This implies $\tn{\bteta-\bteta_0}=\tn{\bteta}\geq\gamma$. Thus, there is no global minima within $R<\gamma=\sqrt{2\Lc(0)/\mu}$ neighborhood of $\bteta_0$.
\end{proof}

\section*{Acknowledgements}
M. Soltanolkotabi is supported by the Packard Fellowship in Science and Engineering, an NSF-CAREER under award \#1846369, the Air Force Office of Scientific Research Young
Investigator Program (AFOSR-YIP) under award \#FA9550-18-1-0078, an NSF-CIF award \#1813877, and a Google faculty research award.
\small{
\bibliography{Bibfiles}
\bibliographystyle{unsrt} 
}

% !TEX root = shortest.tex
\appendix
\section{Proof of key lemmas for low-rank recovery}
\label{keyLRpf}
\subsection{Uniform upper bounds on the nuclear norm (Proof of Lemma \ref{up bound})}
Given the random nature of the matrices $\mtx{X}_i$, $\mat{\Jc(\bTeta)^T\vb}$ defines a random process $\Gamma_{\vb,\bTeta}$ indexed by $\bTeta$ and $\vb$ that can be rewritten in the form
\[
\Gamma_{\vb,\bTeta}:=\mat{\Jc(\bTeta)^T\vb}=\sum_{i=1}^n \vb_i\X_i \bTeta.
\]
Define $\mathbb{S}^{dr-1}=\{\bTeta\in\R^{d\times r}: \fronorm{\bTeta}=1\}$ as the space of matrices with unit Frobenius norm and $\Sc$ as the unit sphere in $\R^n$. The statement of the lemma can then be rephrased as bounding the supremum of this stochastic process over $\mathbb{S}^{dr-1}\times \Sc$, that is $\sup_{\vb\in\Sc, \bTeta\in\Mc} \|\Gamma_{\vb,\bTeta}\|_\star$. To establish such a bound, we first determine the behavior of $\Gamma_{\vb,\bTeta}$ for fixed $\bTeta\in\Mc$ and $\vb\in \Sc$. Assume $\bTeta$ has a singular value decomposition $\Ub\bSi\Vb^T$ with $\Ub,\Vb\in\R^{d\times r}$. Define $\Y=\sum_{i=1}^n \vb_i\X_i\Ub$ and note that $\Y\in\R^{d\times r}$ is a matrix with i.i.d.~$\Nn(0,1)$ entries. Hence, using $\|\bSi\|_F= 1$ and $\|\bSi\|_\star\leq\sqrt{r}$, we have
\[
\nucnorm{\Gamma_{\vb,\bTeta}}=\nucnorm{\Y\bSi\Vb^T}=\nucnorm{\Y\bSi}\leq \opnorm{\Y}\nucnorm{\bSi}\leq  \sqrt{r}\opnorm{\Y}.
\]
Note that, expectation of the spectral norm is known to be bounded by $\E[\|\Y\|]\leq \sqrt{d}+\sqrt{r}\leq 2\sqrt{d}$ via Gordon's lemma. This yields
\begin{align}
\label{upexp}
\E[\nucnorm{\Gamma_{\vb,\bTeta}}]\le \E[\sqrt{r}\opnorm{\Y}]\le 2\sqrt{dr}.
\end{align}
Next, we also show that $\nucnorm{\Gamma_{\vb,\bTeta}}$ concentrates well around this expectation. To show this we use the fact stated above that $\nucnorm{\Gamma_{\vb,\bTeta}}=\nucnorm{\Y\bSi}$ is a function of a Gaussian matrix $\mtx{Y}$. Furthermore, $\nucnorm{\Y\bSi}$ is Lipschitz as for any two matrices $\mtx{Y}_1,\mtx{Y}_2$ we have
\begin{align}
\label{Lip}
\abs{\nucnorm{\Y_2\bSi}-\nucnorm{\Y_1\bSi}}\le& \nucnorm{\left(\mtx{Y}_2-\mtx{Y}_1\right)\mtx{\Sigma}}\nonumber\\
=&\langle\mtx{V},\left(\mtx{Y}_2-\mtx{Y}_1\right)\mtx{\Sigma}\rangle\nonumber\\
=&\langle\mtx{Y}_2-\mtx{Y}_1,\mtx{V}\mtx{\Sigma}^T\rangle\nonumber\\
\le&\fronorm{\mtx{Y}_2-\mtx{Y}_1}\fronorm{\mtx{V}\mtx{\Sigma}^T}\nonumber\\
\le&\fronorm{\mtx{Y}_2-\mtx{Y}_1}\opnorm{\mtx{V}}\fronorm{\mtx{\Sigma}}\nonumber\\
\le&\fronorm{\mtx{Y}_2-\mtx{Y}_1}.
\end{align}
Here $\Vb$ follows from dual representation of the nuclear norm and is a matrix with spectral norm bounded by $1$ maximizing $\langle\mtx{V},\left(\mtx{Y}_2-\mtx{Y}_1\right)\mtx{\Sigma}\rangle$. Thus for fixed $\vct{v}$ and $\mtx{\Theta}$, $\nucnorm{\Gamma_{\vb,\bTeta}}$ is a $1$-Lipschitz function of a Gaussian matrix $\mtx{Y}$. Thus utilizing concentration of Gaussian measure combined with \eqref{upexp} implies
\begin{align}
\label{Gconc}
\mathbb{P}\Big\{\nucnorm{\Gamma_{\vb,\bTeta}}\geq 2\sqrt{dr}+t\Big\}\le \mathbb{P}\Big\{\nucnorm{\Gamma_{\vb,\bTeta}}\geq \E\big[\nucnorm{\Gamma_{\vb,\bTeta}}\big]+t\Big\}\le e^{-\frac{t^2}{2}}.
\end{align}
We will combine \eqref{Gconc} above with an application of standard union bound. To this aim let $\mathcal{M}\subset\Mc$ be an $\eps=1/4$ cover of $\Mc$ and $\mathcal{S}\subset\Sc$ be a $\eps=1/4$ cover of $\Sc$ and note that based on standard covering bounds,
\[
\log |\mathcal{S}|\leq 3n\quad\text{and}\quad\log |\mathcal{M}|\leq 3rd.
\]
Using \eqref{Gconc} with $t=4\sqrt{dr}$ combined with the above covering bound we conclude that for $n\le dr$
\begin{align*}
\mathbb{P}\Bigg\{\underset{(\vct{v},\mtx{\Theta})\in\mathcal{S}\times \mathcal{M}}{\sup}\text{ }\nucnorm{\Gamma_{\vb,\bTeta}}\geq 6\sqrt{dr}\Bigg\}\le \abs{\mathcal{S}}\cdot \abs{\mathcal{M}}\cdot\mathbb{P}\Big\{\nucnorm{\Gamma_{\vb,\bTeta}}\geq \E\big[\nucnorm{\Gamma_{\vb,\bTeta}}\big]+4\sqrt{dr}\Big\}\le e^{3n}\cdot e^{3rd}\cdot e^{-8rd}\le e^{-2rd}.
\end{align*}
Thus for all $(\vb,\bTeta)\in\mathcal{S}\times\mathcal{M}$ we have $\|\Gamma_{\vb,\bTeta}\|_\star\leq 6\sqrt{dr}$ with high probability. To extend this over the entire set $\Sc\times \Mc$ define
\[
\left(\vb^\star,\bTeta^\star\right):=\underset{(\vb,\bTeta)\in\Sc\times \Mc}{\arg\sup} \nucnorm{\Gamma_{\vb,\bTeta}}\quad\text{and}\quad \text{OPT}=\nucnorm{\Gamma_{\bTeta^\star,\vb^\star}}.%\sup_{\bTeta\in\Mc,\vb\in\Sc^{n-1}} \|\gamma_{\bTeta,\vb}\|_\star.
\]
Now let $\widetilde{\vb}$ and $\widetilde{\bTeta}$ be the closest points of the covers $\mathcal{S}$ and $\mathcal{M}$ to $\vct{v}^*$ and $\bTeta^*$ and note that $\twonorm{\widetilde{\vct{v}}-\vct{v}^*}\le 1/4$ and $\fronorm{\widetilde{\mtx{\Theta}}-\mtx{\Theta}^*}\le 1/4$. Thus, will probability at least $1-e^{-2rd}$ we have
\begin{align*}
\text{OPT}=&\nucnorm{\mtx{\Gamma}_{\vct{v}^*,\bTeta^*-\widetilde{\bTeta}}+\mtx{\Gamma}_{\vct{v}^*-\widetilde{v},\widetilde{\bTeta}}+\mtx{\Gamma}_{\widetilde{\vct{v}},\widetilde{\bTeta}}},\\
\overset{(a)}{\le}&  \nucnorm{\mtx{\Gamma}_{\vct{v}^*,\bTeta^*-\widetilde{\bTeta}}}+\nucnorm{\mtx{\Gamma}_{\vct{v}^*-\widetilde{v},\widetilde{\bTeta}}}+\nucnorm{\mtx{\Gamma}_{\widetilde{\vct{v}},\widetilde{\bTeta}}},\\
\overset{(b)}{\le}&\text{OPT}\fronorm{\bTeta^*-\widetilde{\mtx{\Theta}}}+\text{OPT}\twonorm{\vct{v}^*-\widetilde{\vct{v}}}+\nucnorm{\mtx{\Gamma}_{\widetilde{\vct{v}},\widetilde{\bTeta}}},\\
\overset{(c)}{\le}&\frac{1}{2}\text{OPT}+6\sqrt{dr},
\end{align*}
which implies that $\text{OPT}=\nucnorm{\Gamma_{\bTeta^\star,\vb^\star}}\le 12\sqrt{dr}$, completing the proof. In the above (a) follows from the triangular inequality, (b) from the linearity of $\mtx{\Gamma}_{\vct{v},\mtx{\Theta}}$ with respect to $\vct{v}$ and $\bTeta$ and the definition of OPT, and (c) from the bound on the cover. 

\subsection{Proof of Lemma \ref{nuc nor exp}}
Note that for a Gaussian random vector $\vct{g}\sim\mathcal{N}(\vct{0},\mtx{I}_d)$ we have
\begin{align*}
\E\big[\twonorm{\vct{g}}^4\big]=\E\Bigg[\left(\sum_{g=1}^d\vct{g}_k^2\right)^2\Bigg]=\sum_{k=1}^d\left(\E[\vct{g}_k^4]-(\E[\vct{g}_k^2])^2\right)+\left(\E[\twonorm{\vct{g}}^2]\right)^2=d^2+2d
\end{align*}
Using the above we can conclude that
\begin{align}
\label{tempA1}
\E\big[\fronorm{\Gb\bSi}^4\big]=&\E\Bigg[\left(\sum_{k=1}^r\mtx{\Sigma}_{kk}^2\twonorm{\mtx{G}_k}^2\right)^2\Bigg]\nonumber\\
=&\sum_{k=1}^r\mtx{\Sigma}_{kk}^4\left(\E\big[\twonorm{\mtx{G}_k}^4\big]-\E\big[\twonorm{\mtx{G}_k}^2\big]^2\right)+\left(\E\Big[\sum_{k=1}^r\mtx{\Sigma}_{kk}^2\twonorm{\mtx{G}_k}^2\Big]\right)^2\nonumber\\
=&2d\sum_{k=1}^r\mtx{\Sigma}_{kk}^4+d^2\fronorm{\mtx{\Sigma}}^4\nonumber\\
\le&(2d+d^2)\fronorm{\mtx{\Sigma}}^4\nonumber\\
\le&3d^2\fronorm{\mtx{\Sigma}}^4\nonumber\\
=&3\left(\E[\fronorm{\mtx{G}\mtx{\Sigma}}^2]\right)^2
\end{align}
Note that, using $\E[\|\Gb\|]\leq \sqrt{d}+\sqrt{r}$,
\begin{align*}
\mathbb{P}\Big\{\opnorm{\Gb\bSi}\ge (\sqrt{d}+\sqrt{r}+t)\|\bSi\|\Big\}\le \mathbb{P}\Big\{\opnorm{\Gb}\ge \sqrt{d}+\sqrt{r}+t\Big\}\le e^{-\frac{t^2}{2}}.
\end{align*}
Define the event $\mathcal{E}=\big\{\mtx{G}\in\R^{d\times r}: \opnorm{\mtx{G}\mtx{\Sigma}}\le 2\vartheta\left(\sqrt{d}+\sqrt{r}\right)\big\}$. Using the above with $t=2\sqrt{r}$ we have
\begin{align}
\label{tempvartheta}
\mathbb{P}\Big\{\mathcal{E}^c\Big\}=\mathbb{P}\Big\{\opnorm{\Gb\bSi}\ge 2\vartheta\left(\sqrt{d}+3\sqrt{r}\right)\Big\}=\mathbb{P}\Big\{\opnorm{\Gb\bSi}\ge 2\vartheta\left(\sqrt{d}+\sqrt{r}+t\right)\Big\}\le e^{-\frac{t^2}{2}}=e^{-2r}.
\end{align}
Using these definitions we conclude that
\begin{align*}
\E\big[\fronorm{\Gb\bSi}^2\big]\overset{(a)}{\le}&\E\big[\opnorm{\Gb\bSi}\nucnorm{\Gb\bSi}\big]\\
=&\E\big[\opnorm{\Gb\bSi}\left(\mathbb{1}_{\mathcal{E}}+\mathbb{1}_{\mathcal{E}^c}\right)\nucnorm{\Gb\bSi}\big]\\
=&\E\big[\opnorm{\Gb\bSi}\mathbb{1}_{\mathcal{E}}\nucnorm{\Gb\bSi}\big]+\E\big[\opnorm{\Gb\bSi}\mathbb{1}_{\mathcal{E}^c}\nucnorm{\Gb\bSi}\big]\\
\overset{(b)}{\le}&2\vartheta\left(\sqrt{d}+3\sqrt{r}\right)\E[\nucnorm{\Gb\bSi}]+\E\big[\opnorm{\Gb\bSi}\mathbb{1}_{\mathcal{E}^c}\nucnorm{\Gb\bSi}\big]\\
\overset{(c)}{\le}&2\vartheta\left(\sqrt{d}+3\sqrt{r}\right)\E[\nucnorm{\Gb\bSi}]+\E\big[
\sqrt{r}\mathbb{1}_{\mathcal{E}^c}\fronorm{\Gb\bSi}^2\big]\\
\overset{(d)}{\le}&2\vartheta\left(\sqrt{d}+3\sqrt{r}\right)\E[\nucnorm{\Gb\bSi}]+\sqrt{r}\sqrt{\E\big[
\mathbb{1}_{\mathcal{E}^c}\big]}\sqrt{\E\big[\fronorm{\Gb\bSi}^4\big]}\\
\overset{(e)}{\le}&2\vartheta\left(\sqrt{d}+3\sqrt{r}\right)\E[\nucnorm{\Gb\bSi}]+\sqrt{3r}\sqrt{\E\big[
\mathbb{1}_{\mathcal{E}^c}\big]}\E\big[\fronorm{\Gb\bSi}^2\big]\\
\overset{(f)}{\le}&2\vartheta\left(\sqrt{d}+3\sqrt{r}\right)\E[\nucnorm{\Gb\bSi}]+\sqrt{3re^{-2r}}\E\big[\fronorm{\Gb\bSi}^2\big]\\
\overset{(g)}{\le}&2\vartheta\left(\sqrt{d}+3\sqrt{r}\right)\E[\nucnorm{\Gb\bSi}]+\frac{3}{4}\E\big[\fronorm{\Gb\bSi}^2\big].
%\overset{(e)}{\le}&2\vartheta\left(\sqrt{d}+\sqrt{r}+t\right)\E[\nucnorm{\Gb\bSi}]+\sqrt{r}\sqrt{\E\big[\mathbb{1}_{\mathcal{E}_t^c}\big]}\sqrt{\E\Bigg[\left(\sum_{k=1}^r\mtx{\Sigma}_{kk}^2\twonorm{\mtx{G}_k}^2\right)^2\Bigg]}\\
\end{align*}%follows from the fact that $\sigma_{\min}\left(\mtx{\Sigma}\right)\ge \vartheta$, (b)
Here, (a) follows from Holder's inequality, (b) from \eqref{tempvartheta}, (c) from the fact that $\opnorm{\Gb\bSi}\le \fronorm{\Gb\bSi}$ and $\nucnorm{\Gb\bSi}\le \sqrt{r}\fronorm{\Gb\bSi}$, (d) from Cauchy Schwarz, and (e) from \eqref{tempA1}, (f) from \eqref{tempvartheta}, and (g) from the fact that $\sqrt{3re^{-2r}}\le \frac{3}{4}$. The above chain of inequalities thus allow us to conclude that
\begin{align*}
\E\big[\fronorm{\Gb\bSi}^2\big]\le 8\vartheta\left(\sqrt{d}+3\sqrt{r}\right).\E[\nucnorm{\Gb\bSi}]\le 32\vartheta \sqrt{d}\E[\nucnorm{\Gb\bSi}].
\end{align*}
Combining the latter with the fact that $\E\big[\fronorm{\Gb\bSi}^2\big]=d\fronorm{\mtx{\Sigma}}^2\ge dr\vartheta^2$, concludes the proof.

\subsection{Proof of Lemma \ref{low bound}}
For the upper bound we use Lemma \ref{up bound} together with the fact that $\fronorm{\mtx{\Theta}}\le 2\vartheta\sqrt{r}$ to conclude that for all $\vct{v}\in\Sc$
\begin{align}
\label{uppbnd}
\nucnorm{\mat{\Jc(\bTeta)^T\vb}}\le 24\vartheta\sqrt{d}r,
\end{align}
holds with probability at least $1-e^{-2dr}$.

We next turn our attention to the lower bound. Given the random nature of the matrices $\mtx{X}_i$, $\mat{\Jc(\bTeta)^T\vb}$ defines a random process $\mtx{\Gamma}_{\vb}$ indexed by $\vb$ which can be rewritten in the form
\[
\mtx{\Gamma}_{\vb}:=\mat{\Jc(\bTeta)^T\vb}=\sum_{i=1}^n \vb_i\X_i \bTeta.
\]
Thus, in this lemma we are interested in lower bounding $\inf_{\vb\in\Sc}\nucnorm{\Gamma_{\vb}}$ for a fixed $\bTeta$. To establish such bounds, we first determine the behavior of $\Gamma_{\vb}$ for a fixed $\vb$. Let $\bTeta$ have singular value decomposition $\Ub\bSi\Vb^T$ with $\Ub\in\R^{d\times r}$ and set $\Y=\sum_{i=1}^n \vb_i\X_i\Ub$ so that $\Gamma_{\vb}=\Y\bSi \Vb^T$. By construction, for a fixed $\vb\in\Sc$, the matrix $\Y\in\R^{d\times r}$ has i.i.d.~$\Nn(0,1)$ entries. Also note that by \eqref{Lip}, $\nucnorm{\mtx{\Gamma}_{\vct{v}}}=\nucnorm{\mtx{Y}\mtx{\Sigma}\mtx{V}^T}$ is a $\fronorm{\mtx{\Sigma}}\le 2\vartheta\sqrt{r}$ Lipschitz function of $\mtx{Y}$. Also by Lemma \ref{nuc nor exp}, $\E[\nucnorm{\mtx{\Gamma}_{\vct{v}}}]\ge\frac{1}{32}\vartheta\sqrt{d}r$. Thus, by concentration of Lipschitz functions of Gaussian we have
\begin{align*}
\mathbb{P}\Big\{\nucnorm{\mtx{\Gamma}_{\vct{v}}}\le \frac{1}{32}\vartheta\sqrt{d}r -t\Big\}\le\mathbb{P}\Big\{\nucnorm{\mtx{\Gamma}_{\vct{v}}}-\E[\nucnorm{\mtx{\Gamma}_{\vct{v}}}]\le -t\Big\}\le e^{-\frac{t^2}{8r\vartheta^2}}.
\end{align*}
Thus using $t=\frac{1}{288}\vartheta\sqrt{d}r$ we conclude that $\nucnorm{\mtx{\Gamma}_{\vct{v}}}\ge \frac{1}{36}\vartheta\sqrt{d}r$ holds with probability at least $1-e^{-2\gamma dr}$ with $\gamma$ a fixed numerical constant. Now pick a $\frac{1}{19000}$ cover $\mathcal{S}$ of $\Sc$. This cover size is at most  $\log\abs{\mathcal{S}}\le \log\left(\frac{3}{\frac{1}{19000}}\right)n\le 11n$. Thus using the union bound we conclude that for $n\le cdr:=\frac{\gamma}{11}dr$ we have
\begin{align*}
\mathbb{P}\Big\{\underset{\vb\in\mathcal{S}}{\inf}\nucnorm{\mtx{\Gamma}_{\vct{v}}}\le \frac{1}{36}\vartheta\sqrt{d}r\Big\}\le e^{11n}e^{-2\gamma dr}\le e^{-\gamma dr}.
\end{align*}
To proceed, given any $\vb\in \Sc$ denote the closest point from the cover $\mathcal{S}$ to this point by $\widetilde{\vb}$. Using the fact that $\tn{\vb-\widetilde{\vb}}\le \frac{1}{19000}$ combined with \eqref{uppbnd} we conclude that
\begin{align*}
\nucnorm{\mtx{\Gamma}_{\vb}}&\ge \nucnorm{\mtx{\Gamma}_{\widetilde{\vb}}}-\nucnorm{\mtx{\Gamma}_{\vb-\widetilde{\vb}}}\\
&\ge \frac{1}{36}\vartheta\sqrt{d}r-\frac{24}{19000}\vartheta\sqrt{d}r\\
&\ge \frac{1}{40}\vartheta\sqrt{d}r,
\end{align*}
holds with probability at least $1-e^{-\gamma dr}-e^{-2dr}\ge 1-2e^{-\gamma dr}$.

\subsection{Proof of Lemma \ref{jac bound}}

For any arbitrary $\bTeta\in \Dc$ and $\vb\in \Sc$ using Lemma \ref{up bound}, 
\[
\|\mat{\Jc(\bTeta)^T\vb}-\mat{\Jc(\bTeta_0)\vb}\|_\star=\|\mat{\Jc(\bTeta-\bTeta_0)\vb}\|_\star\leq 12\sqrt{dr}\fronorm{\mtx{\Theta}-\mtx{\Theta}_0},
\]
holds with probability at least $1-e^{-2dr}$. Using Lemma \ref{low bound}, 
\[
 \frac{1}{40}\vartheta\sqrt{d}r\le \nucnorm{\mat{\Jc(\bTeta_0)^T\vb}}\le 24\vartheta\sqrt{d}r,
\]
holds with probability at least $1-2e^{-\gamma dr}$. Combining the latter two bounds, using $\bTeta\in \Dc$ and definition of $\Dc$, and applying the triangle inequality we conclude that
\[
 \frac{1}{50}\vartheta\sqrt{d}r\le \nucnorm{\mat{\Jc(\bTeta)^T\vb}}\le 25\vartheta\sqrt{d}r,
\]
holds with probability at least $1-3e^{-\gamma dr}$. Using the fact that $\frac{\nucnorm{\mtx{A}}}{\sqrt{r}}\le\fronorm{\mtx{A}}\le \nucnorm{\mtx{A}}$ we thus have
\[
 \frac{1}{50}\vartheta\sqrt{dr}\le \fronorm{\mat{\Jc(\bTeta)^T\vb}}\le 25\vartheta\sqrt{d}r.
\]
Using the fact that $\fronorm{\mat{\Jc(\bTeta)^T\vb}}=\tn{\Jc(\bTeta)^T\vb}$ and the result holds for all $\vb$, completes the proof of \eqref{Jspec}. 

\noindent To prove \eqref{lipineq}, note that on the same event applying Lemma \ref{up bound}, for any $\mtx{\Theta}_1, \mtx{\Theta}_2\in\R^{d\times r}$ we have
\[
\|\Jc(\bTeta_2)-\Jc(\bTeta_1)\|=\sup_{\vb\in\Sc}\tf{\mat{(\Jc(\bTeta_2)-\Jc(\bTeta_1))^T\vb}}\leq 12\sqrt{dr}\fronorm{\bTeta_2-\bTeta_1},
\]
concluding the proof of \eqref{lipineq}.

\end{document}